%% file: main.tex
\documentclass[accepted]{uai_2024/uai2024_camera_ready}

\usepackage[utf8]{inputenc} 
\usepackage[T1]{fontenc}    
\usepackage{url}            
\usepackage{booktabs}       
\usepackage{amsfonts}       
\usepackage{nicefrac}       
\usepackage{microtype}      
\usepackage{xcolor}
\usepackage{hyperref}  
\usepackage{amssymb}
\usepackage{graphicx}
\usepackage{subcaption}
\usepackage{wrapfig}
\usepackage{tikz}
\usepackage{natbib}
\usepackage{include/titletoc}
\usetikzlibrary{bayesnet}

\newenvironment{talign*}
 {\csname align*\endcsname}
 {\endalign}

\definecolor{mydarkblue}{rgb}{0,0.08,0.45}
\hypersetup{ %
    pdftitle={},
    pdfauthor={},
    pdfsubject={},
    pdfkeywords={},
    pdfborder=0 0 0,
    pdfpagemode=UseNone,
    colorlinks=true,
    linkcolor=mydarkblue,
    citecolor=mydarkblue,
    filecolor=mydarkblue,
    urlcolor=mydarkblue,
    pdfview=FitH
}

\include{include/custom}

\include{include/custom_paper}

\usepackage{enumitem}
\setlist{nolistsep,leftmargin=*}

\makeatletter
\newcommand{\customlabel}[2]{%
   \protected@write \@auxout {}{\string \newlabel {#1}{{#2}{\thepage}{#2}{#1}{}} }%
   \hypertarget{#1}{}
}
\makeatother

\newtheorem{theorem}{Theorem}
\newtheorem{lemma}{Lemma}

\title{Conditional Bayesian Quadrature}

\author[1,*]{Zonghao Chen}
\author[1,*]{Masha Naslidnyk}
\author[2]{Arthur Gretton}
\author[3]{Fran\c{c}ois-Xavier Briol}
\affil[1]{%
    Department of Computer Science\\
    University College London, London, UK
}
\affil[2]{%
    Gatsby Computational Neuroscience Unit\\
    University College London, London, UK
    }
\affil[3]{%
    Department of Statistical Science\\
    University College London, London, UK
  }

\begin{document}
\maketitle
\input{uai_2024/0_abstract}

\input{uai_2024/1_introduction.tex}

\input{uai_2024/2_background.tex}

\input{uai_2024/3_methodology.tex}
\input{uai_2024/4_convergence_rate}

\input{uai_2024/5_experiments.tex}

\input{uai_2024/6_conclusion.tex}
\clearpage

\bibliographystyle{plainnat}
\bibliography{main}
\clearpage

\input{uai_2024/appendix}
\end{document}

%% file: include/custom.tex
\usepackage[utf8]{inputenc}
\usepackage{booktabs} 
\usepackage{hyperref}

\usepackage{amsthm}
\usepackage{bbm} 
\usepackage{bm}
\usepackage{verbatim}
\usepackage{float}
\usepackage{color,soul}
\usepackage{enumitem}
\usepackage{mathtools}
\usepackage{hhline}
\usepackage[title]{appendix}
\usepackage[nameinlink]{cleveref} 
\usepackage[font=small,labelfont=bf,
   justification=justified,
   format=plain]{caption}
\usepackage{tikz}
\usepackage{wrapfig,booktabs}
\usepackage{xspace}
\usepackage{cancel}
\usepackage{sidecap}

\graphicspath{ {../figures/} }

\DeclarePairedDelimiterX{\infdivx}[2]{(}{)}{%
  #1\;\delimsize\|\;#2%
}

\usetikzlibrary{shapes.geometric, arrows, bayesnet, calc, positioning}

\def\Id{\mathrm{Id}}

\newcommand{\calH}{\mathcal{H}}
\newcommand{\calC}{\mathcal{C}}

\newcommand{\calN}{\mathcal{N}}
\newcommand{\calO}{\mathcal{O}}
\newcommand{\calL}{\mathcal{L}}

\newcommand{\calX}{\mathcal{X}}
\newcommand{\calY}{\mathcal{Y}}

\newcommand{\calW}{\mathcal{W}}

\newcommand{\x}{\mathbf{x}}

\renewcommand{\l}{{\ell}}

\newcommand{\R}{\mathbb{R}}

\newcommand{\Pb}{\mathbb{P}}
\newcommand{\Qb}{\mathbb{Q}}

\newcommand{\closer}[3]{{\kern-#1ex{#2}\kern-#3ex}}

\DeclareMathOperator{\diag}{diag}

\mathchardef\mhyphen="2D

\DeclareMathOperator{\E}{\mathbb{E}}

\newcommand\reallywidehat[1]{\arraycolsep=0pt\relax%
\begin{array}{c}
\stretchto{
  \scaleto{
    \scalerel*[\widthof{\ensuremath{#1}}]{\kern-.5pt\bigwedge\kern-.5pt}
    {\rule[-\textheight/2]{1ex}{\textheight}} 
  }{\textheight} %
}{0.5ex}\\           
#1\\                 
\rule{-1ex}{0ex}
\end{array}
}

%% file: include/custom_paper.tex
\usepackage{algorithm}
\usepackage{algorithmic}

\definecolor{azure}{rgb}{0.0, 0.5, 1.0}
\definecolor{airforceblue}{rgb}{0.36, 0.54, 0.66}
\definecolor{darkgreen}{rgb}{0.0, 0.2, 0.13}

\usepackage{standalone}
\usepackage{tikz}
\usepackage{pgfplots}
\usepackage{pgf}
\usetikzlibrary{calc}
\usetikzlibrary{positioning}
\usetikzlibrary{angles,quotes}
\usetikzlibrary{backgrounds}
\usetikzlibrary{fit}
\usetikzlibrary{arrows}
\usetikzlibrary{arrows.meta}
\usetikzlibrary{shapes.symbols}
\usetikzlibrary{shadings}
\usetikzlibrary{shapes}
\usetikzlibrary{fadings}
\usetikzlibrary{bayesnet}
\usetikzlibrary{matrix}
\usetikzlibrary{plotmarks}
\usetikzlibrary{intersections}
\usetikzlibrary{pgfplots.fillbetween}
\pgfplotsset{compat=1.14}
\usepackage{siunitx}

\usepackage{colortbl}
\definecolor{mediumgray}{gray}{0.7}
\definecolor{lightgray}{gray}{0.85}
\definecolor{lightlightgray}{gray}{0.9}
\definecolor{C1}{HTML}{1F77B4}
\definecolor{C2}{HTML}{FF7F0E}
\definecolor{C3}{HTML}{2CA02C}
\definecolor{C4}{HTML}{D62728}
\definecolor{C5}{HTML}{9467BD}
\colorlet{C1light}{C1!70!white}
\colorlet{C2light}{C2!70!white}
\colorlet{C3light}{C3!70!white}
\colorlet{C4light}{C4!70!white}
\colorlet{C5light}{C5!70!white}
\colorlet{C1lighter}{C1!50!white}
\colorlet{C2lighter}{C2!50!white}
\colorlet{C3lighter}{C3!50!white}
\colorlet{C4lighter}{C4!50!white}
\colorlet{C5lighter}{C5!50!white}
\colorlet{C1vlight}{C1!20!white}
\colorlet{C2vlight}{C2!20!white}
\colorlet{C3vlight}{C3!20!white}
\colorlet{C4vlight}{C4!20!white}
\colorlet{C5vlight}{C5!20!white}
\colorlet{linkcolor}{violet}

\newtheorem{cor}{Corollary}

\newtheorem{prop}{Proposition}

\crefname{enumi}{}{}
\crefname{enumii}{}{}

\usepackage{xr}

\makeatletter
\newcommand*{\addFileDependency}[1]{
\typeout{(#1)}
\@addtofilelist{#1}
\IfFileExists{#1}{}{\typeout{No file #1.}}
}\makeatother


%% file: uai_2024/0_abstract.tex
\begin{abstract}
We propose a novel approach for estimating conditional or parametric expectations in the setting where obtaining samples or evaluating integrands is costly.
Through the framework of probabilistic numerical methods (such as Bayesian quadrature), our novel approach allows to  incorporates prior information about the integrands especially the prior smoothness knowledge about the integrands and the conditional expectation.
As a result, our approach provides a way of quantifying uncertainty and leads to a fast convergence rate, which is confirmed both theoretically and empirically on challenging tasks in Bayesian sensitivity analysis, computational finance and decision making under uncertainty.

\end{abstract}

%% file: uai_2024/1_introduction.tex
\section{Introduction}\label{sec:introduction}
This paper considers the computational challenge of estimating certain intractable expectations which arise in machine learning, statistics, and beyond. Given a function $f:\calX \times \Theta \rightarrow \R$, we are interested in estimating \emph{conditional expectations} (sometimes also called parametric expectations) $I: \Theta \rightarrow \R$ uniformly over the parameter space $\Theta$, where:
\vspace{-5pt}
\begin{align*}
    I(\theta) = \E_{X \sim \mathbb{P}_\theta}[f(X,\theta)]=\int_\calX  f(x, \theta) \mathbb{P}_\theta(\mathrm{d} x), 
\end{align*}
and $\{\mathbb{P}_\theta\}_{\theta \in \Theta}$ is a family of distributions on the integration domain $\calX$. We will assume that $I(\theta)$ is sufficiently smooth in $\theta$ so that $I(\theta),I(\theta')$ are similar given close enough parameters $\theta,\theta'$, but that $I$ is not available in closed-form and must be approximated through samples and function evaluations. 

The computational challenge of approximating conditional expectations arises in many fields. It must be tackled when calculating tail probabilities in rare-event simulation \citep{Tang2013}, and when computing moment generating, characteristic, or cumulative distribution functions \citep{Giles2015,Krumscheid2018}. It also arises when computing the conditional value at risk or various valuations of options \citep{longstaff2001valuing,alfonsi2022many}, for Bayesian sensitivity analysis \citep{Lopes2011,Kallioinen2021}, or even more broadly for scientific sensitivity analysis; see for example Sobol indices \citep{Sobol2001}. Conditional expectations $I(\theta)$ are also often computed as an intermediate quantity. 
For example, given $\phi:\R \rightarrow \R$ and some probability distribution $\mathbb{Q}$ on $\Theta$, we are often interested in the \emph{nested expectation} given by $\mathbb{E}_{\theta \sim \mathbb{Q}}[\phi(I(\theta))]$ \citep{Hong2009,Rainforth2018}. This problems comes about when computing the expected information gain in Bayesian experimental design \citep{Chaloner1995}, and for computing the expected value of partial perfect information in health economics~\citep{heath2017review}.

Methods for computing $I(\theta)$ generally select $T$ parameter values $\theta_1,\cdots,\theta_T \in \Theta$, then simulate $N$ realisations from each corresponding probability distribution $\mathbb{P}_{\theta_1}, \cdots, \mathbb{P}_{\theta_T}$ at which they evaluate the integrand $f$, leading to a total of $N T$ evaluations. 
The usual approach is to use classical Monte Carlo methods to estimate $I(\theta_1), \cdots, I(\theta_T)$, but in many applications we are also interested in estimating either $I(\theta)$ for a fixed $\theta \notin \{\theta_1,\cdots,\theta_T\}$, or  $I(\theta)$ uniformly over $\theta \in \Theta$. 
As a result, a second step combining the estimates of $I(\theta_1), \cdots, I(\theta_T)$ is often required to complete the task. 

The most straightforward approach to estimating conditional expectation is importance sampling \citep{Glynn1989,Madras1999,Tang2013,Demange-Chryst2022}, where $I(\theta)$ is estimated by weighting function evaluations to account for the fact that the samples were not obtained from $\mathbb{P}_\theta$ but from the importance distributions $\mathbb{P}_{\theta_1}, \cdots, \mathbb{P}_{\theta_T}$. 
Unfortunately, this approach is only applicable when $f$ does not depend on $\theta$ (otherwise new expensive function evaluations are needed), and it is usually difficult to identify  importance distributions that can lead to an accurate estimator for small $N$ and $T$. 
Alternatively, least-squares Monte Carlo  \citep{longstaff2001valuing,alfonsi2022many} first estimates $I(\theta_1),\cdots, I(\theta_T)$ through Monte Carlo, then estimates $I(\theta)$ through linear or polynomial regression based on these estimates. These methods are therefore dependent on the accuracy of the Monte Carlo estimators and the regression method. 

Overall and in addition, there are two main limitations which all of these existing methods suffer from. Firstly, they are very sample-intensive; i.e. they require a relatively large number of function evaluations (i.e. $N$ and $T$) to reach a given level of accuracy, which makes them infeasible if sampling or evaluating the integrand is expensive. Secondly, obtaining a finite-sample quantification of uncertainty for $I(\theta)$ is often infeasible. This is a significant limitation for challenging integration problems, for which we would ideally like to know how accurate our estimator is.

To tackle these limitations, we propose a novel algorithm called \emph{conditional Bayesian quadrature} (CBQ). The name comes from the fact that our approach extends the Bayesian quadrature algorithm~\citep{Diaconis1988,OHagan1991BayesHermiteQ,Rasmussen2003,fx_quadrature} to the computation of conditional expectations. As such, CBQ falls in the line of work on probabilistic numerical methods \citep{hennig2015probabilistic,Cockayne2017BPNM,Oates2019Modern,Hennig2022}.
Our algorithm is based on a hierarchical Bayesian model consisting of two-stages of Gaussian process regression, and leads to a univariate Gaussian posterior distribution on $I(\theta)$ whose mean and variance are parametrised by $\theta$. See \Cref{fig:illustration} for an illustration.

\begin{figure}[t]
    \centering
    \includegraphics[width=260pt]{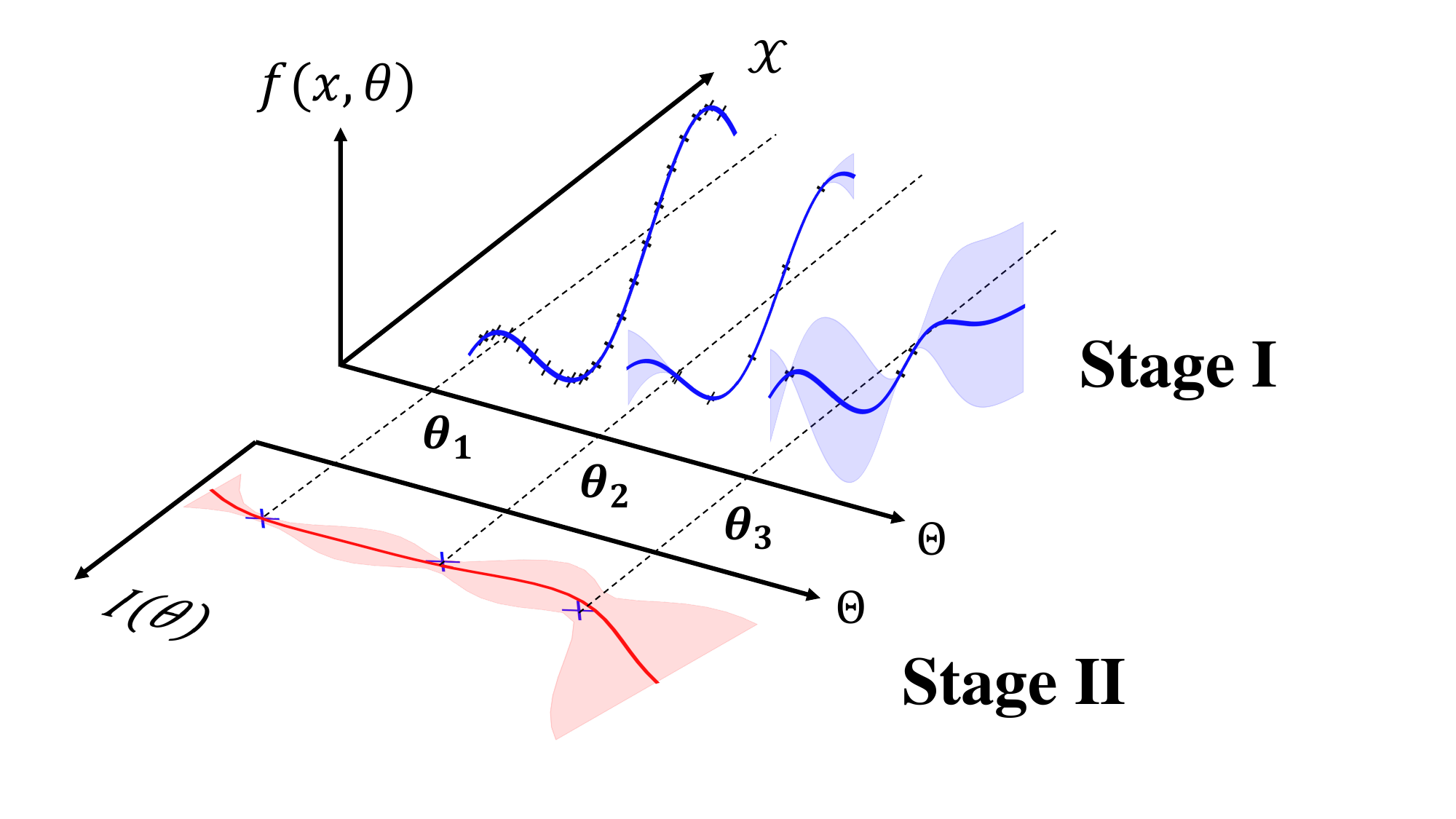}
    \caption{Illustration for \textit{conditional Bayesian quadrature} (CBQ) in \Cref{sec:cbq}. The first stage gives a GP posterior of $f(x,\theta)$ for each $\theta \in \{\theta_1, \cdots, \theta_T\}$, which are then integrated to give \textcolor{blue}{$\hat{I}_{\text{BQ}}(\theta_1), \cdots, \hat{I}_{\text{BQ}}(\theta_T)$}. The second stage then combines all BQ estimates from the first stage to give a GP posterior of $I(\theta)$: \textcolor{red}{$\hat{I}_{\text{CBQ}}(\theta)$}.
    All shared areas represent Bayesian quantification of uncertainty.}
    \label{fig:illustration}
\end{figure}

This approach allows us to mitigate the two main limitations of existing methods. Firstly, we show both theoretically and empirically that our method converges rapidly to the true value and is hence more sample efficient than baselines. This result holds under mild smoothness conditions on $f$ and $I(\theta)$ whenever the dimension of $\calX$ and $\Theta$ is not too large. As a result, a desired accuracy can be reached with smaller $N$ and $T$, and the method will therefore be preferable for expensive problems. Secondly, the fact that we have an entire posterior distribution on $I(\theta)$ allows us to provide finite-sample Bayesian quantification of uncertainty. 

The remainder of the paper is structured as follows: In \Cref{sec:background}, we 
review existing methods for computing conditional expectations and Bayesian quadrature. 
In \Cref{sec:cbq}, we formalise our  novel \textit{conditional Bayesian quadrature} algorithm.  
In \Cref{sec:theory}, we establish the theoretical convergence  of our method.
In \Cref{sec:experiments}, we provide empirical results and compare with baseline methods on challenging tasks in Bayesian sensitivity analysis, computational finance and decision making under uncertainty.

%% file: uai_2024/2_background.tex
\section{Background}\label{sec:background}

 We aim to compute the conditional expectation $I(\theta) = \mathbb{E}_{X \sim \mathbb{P}_\theta}[f(X,\theta)]$, where we assume that  $\calX \subseteq \R^d$, $\Theta \subseteq \R^p$, and $f(\cdot,\theta)$ is in  $\mathcal{L}^2(\mathbb{P}_\theta) :=\{ h:\calX \rightarrow \R : \mathbb{E}_{X \sim \mathbb{P}_\theta}[h(X)^2]<\infty\}$, the space of square-integrable functions with respect to $\mathbb{P}_\theta$ for all $\theta \in \Theta$. The latter is a minimal assumption which ensures that Monte Carlo estimators satisfy the central limit theorem. 
Our observations and corresponding functional evaluations are:
\vspace{-5pt}
\begin{align*}
\begin{aligned}
\theta_{1:T} &:= [\theta_1, \cdots, \theta_T]^\top \in \Theta^T, \\  x^t_{1:N} &:= [x^t_1, \cdots, x^t_N]^\top \in \calX^N, \\
 f(x^t_{1:N}, \theta_t) &:= [f(x^t_1,\theta_t), \cdots, f(x^t_N,\theta_t)]^\top \in \R^N,
\end{aligned}
\end{align*}
for all $t \in \{1,\cdots,T\}$, where we use square brackets to indicate vectors. This could straightforwardly be extended to allow a different number of samples $N_t$ per parameter value $\theta_t$, but we do not consider this case in order to simplify notations throughout. In this section, we will review existing methods for computing conditional expectations and the core ingredient for our method: Bayesian quadrature.

\subsection{Existing Methods for Computing Conditional Expectations}\label{sec:cond_exp}

Existing methods fall into two categories: sampling-based methods and regression-based methods. Throughout, we will assume that $x_{1:N}^t \sim \mathbb{P}_{\theta_t}$ for all $t \in \{1,\cdots,T\}$.

\vspace{-2mm}
\paragraph{Sampling-based Methods} 
We can construct a \emph{Monte Carlo} (MC) estimator \citep{Robert2004} for $I(\theta_t)$ through $\hat{I}_{\text{MC}}(\theta_t) := \frac{1}{N} \sum_{i=1}^N f(x_i^t,\theta_t)$. Unfortunately, we cannot estimate $I(\theta)$ for $\theta \notin \{\theta_{1}, \cdots,\theta_T \}$, and we can only use $N$ rather than $N T$ points to estimate each $I(\theta_t)$, making MC inappropriate for our task. A more suitable alternative is \emph{importance sampling} (IS).
Assume $\mathbb{P}_\theta$ has a Lebesgue density $p_\theta:\calX \rightarrow \R$ which has full support on $\calX$ for all $\theta \in \Theta$, and the integrand does not depend on $\theta$ (i.e. $f(x,\theta) = f(x)$). 
Then the IS estimator is able to make use of all $N T$ samples and can estimate $I(\theta)$ for any parameter $\theta \in \Theta$: $\hat{I}_{\text{IS}}(\theta) := \frac{1}{T} \sum_{t=1}^T \sum_{i=1}^N p_{\theta}(x_i^t)/p_{\theta_t}(x_i^t) f(x_i^t)$. 
The choice of importance distributions  $\mathbb{P}_{\theta_1},\cdots,\mathbb{P}_{\theta_T}$ has been studied in \cite{Glynn1989,Madras1999,Tang2013}, but alternatives beyond this parametric family of distributions could also be used \citep{Demange-Chryst2022}.

\paragraph{Regression-based Methods}
The main regression-based method is least-squares Monte Carlo (LSMC) \citep{longstaff2001valuing}, which is a two-stage approach. Stage 1 consists of computing MC estimators  $\hat{I}_{\text{MC}}(\theta_1), \cdots, \hat{I}_{\text{MC}}(\theta_T)$, whilst stage 2 consists of estimating $I(\theta)$ through linear or polynomial regression based on the estimates from stage 1. 
Other non-parametric regression method could be used though; for kernel ridge regression \citep{Han2009,Hu2020}, we will refer to the algorithm as kernelised least-squares Monte Carlo (KLSMC). Note that KLSMC is identical to standard estimators for conditional kernel mean embeddings based on vector-valued kernel ridge regression and can be recognised as a generalisation of the kernel mean shrinkage estimators of \cite{muandet2016kernelmeanshrinkage,chau2021deconditional}. 

Clearly, both the performance and computational cost of these estimators will depend on the regression method. 
LSMC costs $\calO(TN + p^3)$ with $p$ being the order of polynomial, whereas KLSMC costs $\calO(TN + T^3)$. 
On the other hand, KLSMC is more flexible and will outperform LSMC when $I(\theta)$ cannot be approximated well by a low-order polynomial.

\paragraph{Other Related Work} Alternative approaches for estimating $I(\theta)$ are based on multi-task or meta- learning \citep{xi2018bayesian,gessner2020active,Sun2021,Sun2023}. 
This line of research tends to assume that several related expectations need to be computed, and the relationship between these expectations is encoded through a vector-valued RKHS, or that they are independent draws from a set of tasks. 
Notably, they do not explicitly encode properties of the mapping $\theta \mapsto I(\theta)$, and will therefore be sub-optimal for our setting.
Multilevel Monte Carlo methods are also popular in estimating expensive expectations, by combining samples from multiple levels of resolution~\citep{Giles2015}. However, they are not able to estimate new integrals $I(\theta^\ast)$ or $I(\theta)$ uniformly over $\theta \in \Theta$.

\subsection{Bayesian Quadrature}\label{sec:bayesian_quadrature}
In this section, we present Bayesian quadrature, the foundational component of our approach. Consider the expectation $I = \mathbb{E}_{X \sim \mathbb{P}} [f(X)]$ of some function $f:\calX \rightarrow \mathbb{R}$, where we emphasise that neither $f$ nor $\mathbb{P}$ depend on $\theta$ in this subsection. In Bayesian quadrature (BQ) \citep{Diaconis1988,OHagan1991BayesHermiteQ,Rasmussen2003,fx_quadrature}, we begin by positing a Gaussian process (GP) prior on $f$. We will denote this prior $\mathcal{GP}(m_{\calX},k_{\calX})$, where $m_\calX:\calX \rightarrow \mathbb{R}$ is the mean function and $k_{\calX}:\calX \times \calX \rightarrow \mathbb{R}$ is the covariance (or reproducing kernel) function. These two functions fully characterise the distribution, and can be used to encode prior knowledge about smoothness, periodicity, or sparsity of $f$.
Once a GP prior has been selected, we condition on noiseless function evaluations $f(x_{1:N}) = [f(x_1),\cdots,f(x_N)]^\top$ for $x_{1:N} \in \calX^N$. This leads to a posterior GP on $f$, which induces a univariate Gaussian posterior distribution $\mathcal{N}\big(\hat{I}_\mathrm{BQ},\sigma^2_\mathrm{BQ}\big)$ on $I$, where:
\begin{align*}
\begin{aligned}
    \hat{I}_\mathrm{BQ} & = \mathbb{E}_{X \sim \mathbb{P}}[m_{\calX}(X)] + \mu(x_{1:N})^\top \big(k_{\calX}(x_{1:N}, x_{1:N}) + \\ &\qquad \lambda_{\calX} \Id_N \big)^{-1} \big(f(x_{1:N})-m_{\calX}(x_{1:N}) \big), \\
    \sigma^2_\mathrm{BQ} &= \mathbb{E}_{X,X'\sim \mathbb{P}}\left[k_{\calX}(X,X')\right] - \mu(x_{1:N})^\top \big(k_{\calX}(x_{1:N}, x_{1:N}) \\ &\qquad + \lambda_{\calX} \Id_N \big)^{-1} \mu(x_{1:N}).
\end{aligned}
\end{align*}
Here, $\Id_N$ is an $N$-dimensional identity matrix and $\lambda_{\calX} \geq 0$ is a regularisation parameter, often called ``nugget'' or ``jitter'', which, although not essential from a statistical viewpoint, is often used to ensure the matrix can be numerically inverted \citep{Ababou1994,Andrianakis2012}.

The function  $\mu(x) = \mathbb{E}_{X \sim \mathbb{P}}[k_\calX(X,x)]$ is known as the kernel mean embedding~\citep{muandet2017kernel} of the distribution $\mathbb{P}$ and $\mathbb{E}_{X,X'\sim \mathbb{P}}\left[k_{\calX}(X,X')\right]$ is known as the initial error. 
To implement BQ, we need the kernel mean embedding and the initial error to be available in closed-form, which is a rather strong requirement and does not hold for all pairs of kernel and distribution.
Fortunately, there are multiple solutions for this problem; see  Table 1 in~\citep{fx_quadrature}, ~\citep{Nishiyama2016}, the \texttt{ProbNum} package~\citep{Wenger2021}, or Stein reproducing kernels \citep{anastasiou2023stein}. A discussion is provided in \Cref{appendix:tractable_kernel_means}.

The posterior mean $\hat{I}_\mathrm{BQ}$ provides a point estimate for $I$ whilst the posterior variance $\sigma^2_\mathrm{BQ}$ gives a notion of uncertainty for $I$ which arises due to having only observed $f$ at $N$ points.
For BQ to be well-calibrated and the posterior variance $\sigma^2_\mathrm{BQ}$ to be meaningful, we need to select the GP prior and all associated hyperparameters carefully; this is usually achieved through empirical Bayes~\citep{casella1985emp_bayes}. 
A detailed discussion on hyperparameter selection is provided in \Cref{appendix:hyperparameter_selection}. 
It is noteworthy that BQ does not impose restrictions on how $x_{1:N}$ is selected, and as such does not require independent realisations from $\mathbb{P}$. 
In fact, a number of active learning approaches have proven popular, see \cite{Gunter2014,gessner2020active}. 

The convergence rate of the BQ estimator has been  studied extensively \citep{fx_quadrature,Kanagawa2019,wynne2021convergence} and is particularly fast for low- to mid-dimensional smooth integrands. 
This has to be contrasted with the computational cost, which is inherited from GP regression and is $\calO(N^3)$. 
For this reason, BQ has principally been applied to problems where sampling or evaluating the integrand is very expensive and usually only a small number of samples are available (i.e. small $N$).
Examples range from differential equation solvers \citep{Kersting2016}, neural ensemble search \citep{Hamid2023}, variational inference \citep{Acerbi2018} and simulator-based inference \citep{Bharti2023} to applications in computer graphics \citep{Marques2013,xi2018bayesian}, cardiac modelling \citep{Oates2017heart} and tsunami modelling \citep{li2022multilevel}. 
For cheaper problems, \cite{Jagadeeswaran2018,Karvonen2017symmetric,Karvonen2019} propose BQ methods where the computational cost is much lower, but these are applicable only with specific point sets $x_{1:N}$ and distributions $\mathbb{P}$. \cite{Hayakawa2023} also studies Nystr\"om-type of approximations, whilst \cite{Adachi2022} studies parallelisation techniques. Finally, several alternatives with linear cost in $N$ have also been proposed using tree-based \citep{Zhu2020} or neural-network \citep{Ott2023} models, but these tend to require approximate inference methods such as Laplace approximations or Markov chain Monte Carlo.

%% file: uai_2024/3_methodology.tex
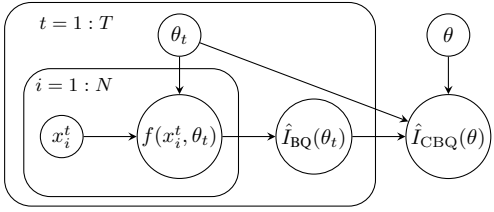
\begin{figure*}[t] 
\centering
\begin{subfigure}[b]{0.38\linewidth}
    \centering
\begin{tikzpicture}[scale=0.3, every node/.style={scale=0.8}]
  \node[latent] (x) at (0, 0) {$x_i^t$};
  \node[latent, right=0.7cm of x] (fx) {$f(x_i^t, \theta_t)$};
  \node[latent, above=0.5cm of fx] (theta) {$\theta_t$};
  \node[latent, right=0.7cm of fx] (IBQ) {$\hat{I}_{\text{BQ}}(\theta_t)$};
\node[latent, right=0.7cm of IBQ] (ICBQ) {$\hat{I}_{\mathrm{CBQ}}(\theta)$};
\node[latent, above=0.5cm of ICBQ] (thetanew) {$\theta$};
  \edge[->,>=stealth] {x} {fx};
  \edge[->,>=stealth] {theta} {fx};
  \edge[->,>=stealth] {fx} {IBQ};
  \edge[->,>=stealth] {IBQ} {ICBQ};
  \edge[->,>=stealth] {thetanew} {ICBQ};
  \edge[->,>=stealth] {theta} {ICBQ};
  
    \plate [inner xsep=0.8cm, inner ysep=0.5cm, xshift=0.0cm, yshift=0.2cm, color=black, rounded corners=10pt, label={[label distance=-0.7cm, yshift=0.0cm, xshift=0.4cm]above left:$i=1:N$}] {N} {(x)(fx)} {};

  \plate [inner xsep=1.4cm, inner ysep=0.7cm, xshift=-0.1cm, yshift=0.0cm, color=black, rounded corners=10pt, label={[label distance=-0.8cm, yshift=0.0cm, xshift=0.0cm]above left:$t=1:T$}] {T} {(x)(fx)(theta)(IBQ)}{};
\end{tikzpicture}
  \end{subfigure}
  \hspace{40pt}
\begin{subfigure}{0.45\textwidth}
\centering
\includegraphics[width=\linewidth]{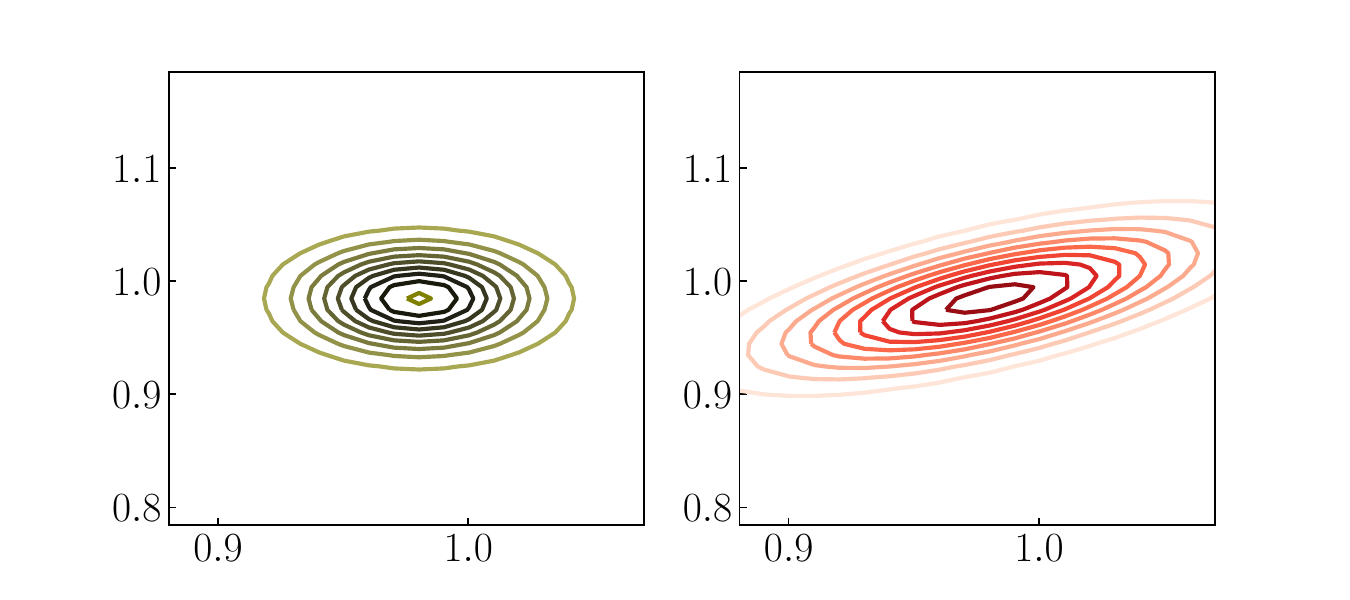}
\vspace{-20pt}
\end{subfigure}
    
\caption{\emph{Illustration of CBQ.} \textbf{Left:} Directed acyclic graph representation. Circle nodes indicate random variables and rectangles correspond to  independent replications over indices. \textbf{Right:} BQ and CBQ posteriors on $I(\theta_{1:2})=[I(\theta_1), I(\theta_2)]^\top$ for $\theta_1 \approx \theta_2$. Unlike BQ, the CBQ posterior accounts for the relation between the two quantities.}
\label{fig:DAG_and_CBQ_bivariate_posterior}
\vspace{-15pt}
\end{figure*}

\section{Methodology}\label{sec:cbq}

\vspace{-2mm}
\emph{Conditional Bayesian quadrature} (CBQ) provides a Bayesian hierarchical model for $I(\theta^*)$ for any $\theta^* \in \Theta$, and the posterior mean of this hierarchical model is called the CBQ estimator. The algorithm falls into the realm of regression-based methods and can therefore be expressed in two stages:
\begin{itemize}[topsep=0pt,leftmargin=*]
    \item \textbf{Stage 1:}  Compute $\hat{I}_\mathrm{BQ}(\theta_{1:T}), \sigma^2_\mathrm{BQ}(\theta_{1:T})$ to obtain the BQ posterior mean and variance on $I(\theta_1),\ldots,I(\theta_T)$. 
    \item \textbf{Stage 2:} Perform GP regression over $I(\theta)$ using the outputs of stage 1. The posterior mean $\hat{I}_\mathrm{CBQ}(\theta)$ is the CBQ estimator for $I(\theta)$, and the variance $k_{\mathrm{CBQ}}(\theta,\theta)$  quantifies uncertainty. 
\end{itemize}
An illustrative figure is provided in \Cref{fig:illustration}.
This two-stage algorithm can also be summarised using the directed acyclic graph in \Cref{fig:DAG_and_CBQ_bivariate_posterior}, where the first stage corresponds to the part of the model inside the largest plate, and the second stage corresponds to the remainder of the graph. 
The CBQ posterior mean and covariance 
are given by
\begin{align*}
\begin{aligned}
    & \hat{I}_{\mathrm{CBQ}}(\theta)  := m_\Theta(\theta)+k_\Theta(\theta, \theta_{1:T}) \big(k_\Theta(\theta_{1:T}, \theta_{1:T}) \\ & \quad + \mathrm{diag}(\lambda_{\Theta}+ \sigma^2_\mathrm{BQ}(\theta_{1:T}))\big)^{-1} (\hat{I}_\mathrm{BQ}\left(\theta_{1:T}) - m_\Theta(\theta_{1:T})\right), \\
    & k_{\mathrm{CBQ}}(\theta,\theta')  := k_{\Theta}(\theta,\theta') - k_\Theta(\theta,\theta_{1:T}) \big( k_{\Theta}(\theta_{1:T}, \theta_{1:T}) \\ &\quad + \mathrm{diag}(\lambda_{\Theta}+ \sigma^2_\mathrm{BQ}(\theta_{1:T})) \big)^{-1} k_\Theta(\theta_{1:T},\theta')
\end{aligned}
\end{align*}
where the observations $\{x_{1:N}^t,f(x_{1:N}^t,\theta_t)\}_{t=1}^T$ enters implicitly through $\hat{I}_\mathrm{BQ}(\theta_{1:T})$.
The terms $\hat{I}_\mathrm{BQ}(\theta_t)$ and $\sigma^2_\mathrm{BQ}(\theta_t)$ are the BQ posterior mean and variance for $I(\theta_t)$, $\mathrm{diag}(\lambda_{\Theta}+ \sigma^2_\mathrm{BQ}(\theta_{1:T})))$ is the diagonal matrix with vector $\lambda_{\Theta}+ \sigma^2_\mathrm{BQ}(\theta_{1:T}))$ on the diagonal and where $\lambda_{\Theta} \geq 0$ acts as a regulariser. We also have    $m_\Theta:\Theta \rightarrow \R$ and $k_{\Theta}:\Theta \times \Theta \rightarrow \R$ which are the prior mean and covariance for the stage 2 GP. 
Similarly to BQ, the ``quadrature" terminology is justified since  $\hat{I}_\mathrm{CBQ}(\theta) := \sum_{t=1}^T \sum_{i=1}^N w_{i,t}^{\mathrm{CBQ}} f(x_i^t,\theta_t)$ for some weights $w_{i,t}^{\mathrm{CBQ}} \in \R$ when $m_\Theta(\theta)=0$.

The first stage corresponds to the BQ procedure highlighted in \Cref{sec:bayesian_quadrature}: we model $f(\cdot,\theta_t)$ with independent $\text{GP}(m^t_{\calX},k^t_{\calX})$ priors, condition on observations $f(x^t_{1:N},\theta_t)$, and consider the posterior distribution on $I(\theta_t)$ for all $t \in \{1,\ldots,T\}$. We therefore require access to closed-form expressions for each of the $T$ kernel mean embeddings and initial errors (see discussion in \Cref{appendix:tractable_kernel_means} on the pairs of kernel and distribution that have a closed form kernel mean embedding). 
Note that at this stage, we do not share any samples across the estimators of $I(\theta_1), \ldots, I(\theta_T)$.

In the second stage, we place a $\text{GP}(m_\Theta,k_\Theta)$ prior on $I:\Theta \rightarrow \R$, and assume $\hat{I}_\mathrm{BQ}(\theta_t)$ are noisy evaluations of $I(\theta_t)$: $\hat{I}_\mathrm{BQ}(\theta_t) = I(\theta_t) +\varepsilon_t$, where the noise terms $\varepsilon_t$ are independent zero-mean Gaussian noise with variance $\sigma^2_\mathrm{BQ}(\theta_t)$ for all $t \in \{1, \dots, T\}$. 
Note that $\hat{I}_\mathrm{BQ}(\theta_t)$ is a deterministic function of independent samples $\theta_t, x^t_1, \cdots, x^t_N$ across $t = 1, \cdots, T$, so $\hat{I}_\mathrm{BQ}(\theta_1), \ldots, \hat{I}_\mathrm{BQ}(\theta_T)$ are also independent. 
As the variance $\epsilon_t$ is input-dependent, this corresponds to heteroscedastic GP regression \citep{Le2005}. 
We now briefly comment on the choice of prior and likelihood in this second stage:
\begin{itemize}[topsep=0pt,leftmargin=*]
    \item The $\text{GP}(m_\Theta,k_\Theta)$ prior can be used to encode prior knowledge about how the expectation $I(\theta)$ varies with the parameter $\theta$. Typically, the stronger this prior information, the faster the CBQ estimator's convergence rate will be; this statement will be made formal in \Cref{sec:theory}.

    \item The likelihood for the heteroscedastic GP is directly inherited from the BQ posteriors in the first stage: the posterior on $I(\theta_t)$ is a univariate normal with mean $\hat{I}_\mathrm{BQ}(\theta_{t})$ and variance $\sigma^2_\mathrm{BQ}(\theta_{t})$. As expected, when the number of samples $N$ grows, the BQ variance $\sigma^2_\mathrm{BQ}(\theta_t)$ will decrease, indicating that we are more certain about $I(\theta_t)$. This is then directly taken into account in stage 2. Note that  heteroscedasticity has previously been shown to be common in practice for LSMC \citep{Fabozzi2017}.
\end{itemize}



CBQ is closely related to  LSMC and KLSMC as it simply corresponds to different choices for the two stages. 
The main difference is in stage 1, where we use BQ rather than MC. This is where we expect the greatest gains for our approach due to the fast convergence rate of BQ estimators (this will be confirmed in \Cref{sec:theory}). For stage 2, we use heteroscedastic GP regression rather than polynomial or kernel ridge regression. As such, the second stage of KLSMC and CBQ is identical up to a minor difference in the way in which the Gram matrix $k_{\Theta}(\theta_{1:T}, \theta_{1:T})$ is regularised before inversion. 
Finally, one significant advantage of CBQ over LSMC and KLSMC is that it is a fully Bayesian model, meaning that we obtain a posterior distribution on $I(\theta)$ for any $\theta \in \Theta$.

The total computational cost of our approach is $\calO(T N^3 + T^3)$ due to the need to compute $T$ BQ estimators in the first stage and heteroscedastic GP regression in the second stage. 
Approximate GP approaches such as~\cite{titsias2009variational} could \emph{not} be used to reduce the cost because they introduce an additional layer of approximation which will slow down the convergence rate of CBQ. 
The cost of CBQ is higher than the cost of $\calO(TN+p^3)$ or $\calO(TN+T^3)$ of LSMC and KLSMC respectively, but as we will see in \Cref{sec:experiments}, the higher computational cost of CBQ will be offset competitive by faster convergence (derived in \Cref{thm:convergence}) and is more competitive compared to baseline methods (see 
\Cref{sec:experiments}).  
Additionally in many applications (such as the SIR model in \Cref{sec:experiments}), the cost of evaluating the integrand will be much larger than the cost of estimation methods, so data-efficient method like CBQ will be more efficient overall. 

Interestingly, CBQ also provides us with a joint Gaussian posterior on the expectation at $\theta^\ast_1, \ldots, \theta^\ast_{T_{\text{Test}}} \in \Theta$ which has mean vector $\hat{I}_{\mathrm{CBQ}}(\theta^\ast_{1:T_{\text{Test}}})$ and covariance matrix $k_{\mathrm{CBQ}}(\theta^\ast_{1:T_{\text{Test}}},\theta^\ast_{1:T_{\text{Test}}})$. This can be computed at an  $\calO(T^2 T_{\text{test}})$ cost, and is illustrated in the right plot of~\Cref{fig:DAG_and_CBQ_bivariate_posterior} on a synthetic example from \Cref{sec:experiments}; as observed, CBQ takes into account of covariances between test points in that the integral value will be similar for similar parameter values, whereas standard BQ treats each integral value independently.

A natural alternative would be to place a GP prior directly on $(x,\theta) \mapsto f(x,\theta)$ and condition on all 
$N \times T$ observations. 
The implied distribution on $I(\theta_1), \ldots, I(\theta_T)$ would also be a multivariate Gaussian distribution. 
This approach coincides with the multi-output Bayesian quadrature (MOBQ) approach of \cite{xi2018bayesian} where multiple integrals are considered simultaneously. 
However, the computational cost of MOBQ is $\calO(N^3 T^3)$, due to fitting a GP on $N T$ observations, and quickly becomes intractable as $N$ or $T$ grow. 
A further comparison of BQ and MOBQ can be found in~\Cref{appendix:cbq_mobq}.
The same holds true if $f$ does not depend on $\theta$, in which case the task reduces to the conditional mean process studied in Proposition 3.2 of \cite{chau2021deconditional}, and when $T=1$, we recover standard Bayesian quadrature. 

\paragraph{Hyperparameters}
The hyperparameter selection for CBQ boils down to the choice of GP interpolation hyperparameters at stage 1 and the choice of GP regression hyperparameters at stage 2. 
To simplify this choice, we renormalise all our function values before performing GP regression and interpolation. 
This is done by first subtracting the empirical mean and then dividing by the empirical standard deviation. 
The choice of covariance functions $k_\calX$ and $k_\Theta$ is made on a case-by-case basis in order to both encode properties we expect the target functions to have, but also to ensure that the corresponding kernel mean is available in closed-form (see \Cref{appendix:tractable_kernel_means}). 
Once this is done, we typically still need to make a choice of hyperparameters for both kernel: lengthscales $l_\calX$, $\l_\Theta$ and amplitudes $A_\calX, A_\Theta$. 
We also need to select the regularizer $\lambda_\calX, \lambda_\Theta$. 
$\lambda_\calX$ is fixed to be $0$ as suggested by \Cref{thm:convergence}, and the rest of the hyperparameters are selected through empirical Bayes, which consists of maximising the log-marginal likelihood.
For more details on hyperparameter selection, please refer to \Cref{appendix:hyperparameter_selection}.

%% file: uai_2024/4_convergence_rate.tex
\section{Theoretical Results}\label{sec:theory}

Our main theoretical result in \Cref{thm:convergence} below guarantees that CBQ is able to recover the true value of $I(\theta)$ when $N$ and $T$ grow. The result of this theorem depends on the smoothness of the problem. We will say a function has smoothness $s$ if it is in the Sobolev space $\calW^{s, 2}$ of functions with at least $s$ (weak) derivatives that are square Lebesgue-integrable \citep{adams2003sobolev}. For a multi-index $\alpha = (\alpha_1, \dots \alpha_p) \in \mathbb{N}^p$, by $D_\theta^\alpha f$ we denote the $|\alpha|=\sum_{i=1}^d \alpha_i$ order weak derivative of a function $f$ on $\Theta$. Similarly, we will say a kernel has smoothness $s$ whenever its corresponding RKHS is a space of functions of smoothness $s$. This is for example the case of the Mat\'ern$-\nu$ kernel in dimension $d$ whenever $s= \nu +d/2$, defined as $k_\nu(x,y) = \frac{\eta}{\Gamma(\nu)2^{\nu - 1}} (\frac{\sqrt{2 \nu}}{l} \| x - y \|_2 )^\nu K_\nu(\frac{\sqrt{2 \nu}}{l} \| x-y \|_2)$ where $K_\nu$ is the modified Bessel function of the second kind and $\eta,l >0$ are hyperparameters.

\begin{theorem}\label{thm:convergence}
Let $x \mapsto f(x, \theta)$ be a function of smoothness $s_f > d/2$, and $\theta \mapsto f(x, \theta)$ be a function of smoothness $s_I > p/2$ such that $\sup_{\theta \in \Theta} \max_{|\alpha|<s_I} \| D_\theta^\alpha f(\cdot, \theta) \|_{\calW^{s_I, 2}(\calX)}<\infty$. Suppose the following assumptions hold:
\vspace{-2pt}
\begin{enumerate}[itemsep=0.1pt,topsep=0pt,leftmargin=*]
\item [A1] The domains $\calX \subset \R^d$ and $\Theta\subset \R^p$ are open, convex, and bounded.
~\label{as:domains}
\item [A2] The parameters and samples satisfy: 
$\theta_{1:T} \sim \mathbb{Q}$, and $x_{1:N}^t \sim \Pb_{\theta_t}$ for all $t \in \{1,\ldots, T\}$. 
~\label{as:pars_and_samples}
\item [A3] $\mathbb{Q}$ has density $q$ such that $\inf_{\theta \in \Theta} q(\theta)>0$ and $\sup_{\theta \in \Theta} q(\theta) < \infty$, and $\Pb_\theta$ has density $p_\theta$ such that $\theta \mapsto p_\theta(x)$ is of smoothness $s_I > p/2$, and for any $\theta \in \Theta$, it holds that $\inf_{\theta \in \Theta, x \in \calX} p_{\theta}(x)>0$ and $\sup_{\theta \in \Theta}\max_{|\alpha|\leq s} \|D_\theta^\alpha p_\theta(x)\|_{\calL^\infty(\calX)}<\infty$.
~\label{as:densities}
\item [A4] The kernels $k_\calX$ and $k_\Theta$ are 
of smoothness $s_\calX \in (d/2, s_f]$ and $s_\Theta \in (p/2, s_I]$ respectively.~\label{as:kernels} 
\item [A5] The regularisers satisfy $\lambda_{\calX}=0$ and $\lambda_{\Theta} = \calO(T^{\frac{1}{2}})$.
\end{enumerate}
\vspace{-2pt}
Then,  we have that for any $\delta \in (0, 1)$ there is an $N_0>0$ such that for any $N \geq N_0$ with probability at least $1-\delta$ it holds that
\begin{align*}
    \left\| \hat I_\mathrm{CBQ} - I \right\|_{\calL^2(\Theta)}
    \leq  C_0(\delta) N^{-\frac{s_\calX}{d} + \varepsilon} + C_1(\delta) T^{-\frac{1}{4}}  ,
\end{align*}
%
for any arbitrarily small $\varepsilon>0$, and the constants $C_0(\delta)=\calO(1/\delta)$ and $C_1(\delta)=\calO(\log(1/\delta))$ are independent of $N, T, \varepsilon$.
\end{theorem}
To prove the result, we represent the CBQ estimator as a \emph{noisy importance-weighted kernel ridge regression} (NIW-KRR) estimator. Then, we extend convergence results for the \emph{noise-free} IW-KRR estimator established in~\citet[Theorem 4]{gogolashvili2023importance} to bound Stage 2 error in terms of the error in Stage 1, which in turn we bound via results on the convergence of GP interpolation from \cite{wynne2021convergence}. See~\Cref{appendix:convergence_rate} for the detailed proof.

We now briefly discuss our assumptions. Many of these were simplified to improve readability, in which case we highlight possible generalisations. A1 is used to guarantee the points eventually cover the domain, and could straightforwardly be generalised to any open and bounded domain with Lipschitz boundary satisfying an interior cone condition; see \cite{kanagawa2020convergence,wynne2021convergence}. A2 ensures  $\theta_{1:T}$ and $x_{1:N}^t$ cover $\calX$ and $\Theta$ sufficiently fast in probability as $N$ and $T$ grow.  The assumption on the point sets could also be straightforwardly generalised to active learning designs or grids following existing work on BQ convergence \citep{Kanagawa2019,kanagawa2020convergence,wynne2021convergence}. A3 ensures that the points will fill $\calX$. A4 guarantees that our first and second stage GPs have the right level of regularity for the problem, although the range of smoothness values could be significantly extended following the approach of \cite{kanagawa2020convergence}. For simplicity, we also implicitly assume that the kernel hyperparameters (such as lengthscales and amplitudes) are known, but this could be extended to estimation in bounded sets; see \citep{Teckentrup2020}. Finally, A5 requires $\lambda_{\calX}=0$, but this could be relaxed at the cost of slowing down convergence (see \Cref{appendix:convergence_rate}). In contrast, growing $\lambda_{\Theta}>0$ in $T$ is natural since we work in a bounded domain and we expect the conditioning of the Gram matrix to become worse as $T \rightarrow \infty$.

We are now ready to discuss the implications of the theorem.
Firstly, the result is expressed in probability to account for randomness in $\theta_{1:T}$ and $x_{1:N}^t$, and provides a rate of $\calO(T^{-1/4}+ N^{- s_\calX/d + \varepsilon})$. We can see that growing $N$ will only help up to some extent (as the second terms approaches zero fast), but that growing $T$ is essential to ensure convergence. This is intuitive since we cannot expect to approximate $I(\theta)$ uniformly simply by increasing $N$ at some fixed points in $\Theta$. 
Despite this, we will see in \Cref{sec:experiments} that increasing $N$ will be essential to improving performance in practice. The rate in $N$ will typically be very fast for smooth targets, but is significantly slowed down for large $d$, demonstrating that our method is mostly suitable for low- to mid-dimensional problems, a common feature shared by Bayesian quadrature based algorithms~\citep{fx_quadrature, frazier2018bayesian}.
There have been some attempts to scale BQ/CBQ to high dimensions; for example in section 5.4 of \cite{fx_quadrature} where the integrand can be decomposed into a sum of low-dimensional functions, however, this is only possible in limited settings when the integrand has certain forms of sparsity. 

\begin{figure*}[t]
\vspace{-10pt}
    \centering
    \begin{minipage}{\textwidth}
    \centering
    \includegraphics[width=320pt]{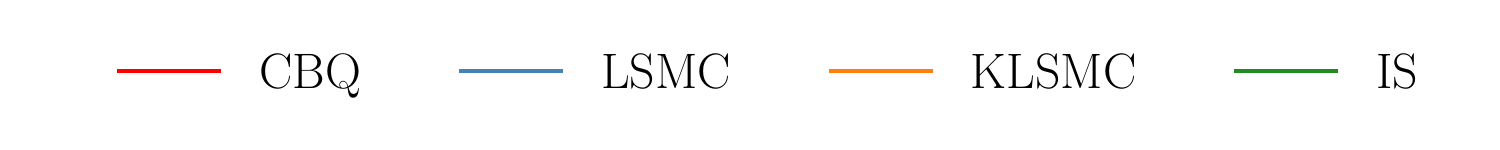}
    \vspace{-7pt}
    \end{minipage}
    
    \centering
    \begin{subfigure}{0.33\textwidth}
        \centering
        \hspace{-10pt}
        \includegraphics[width=1.0\linewidth]{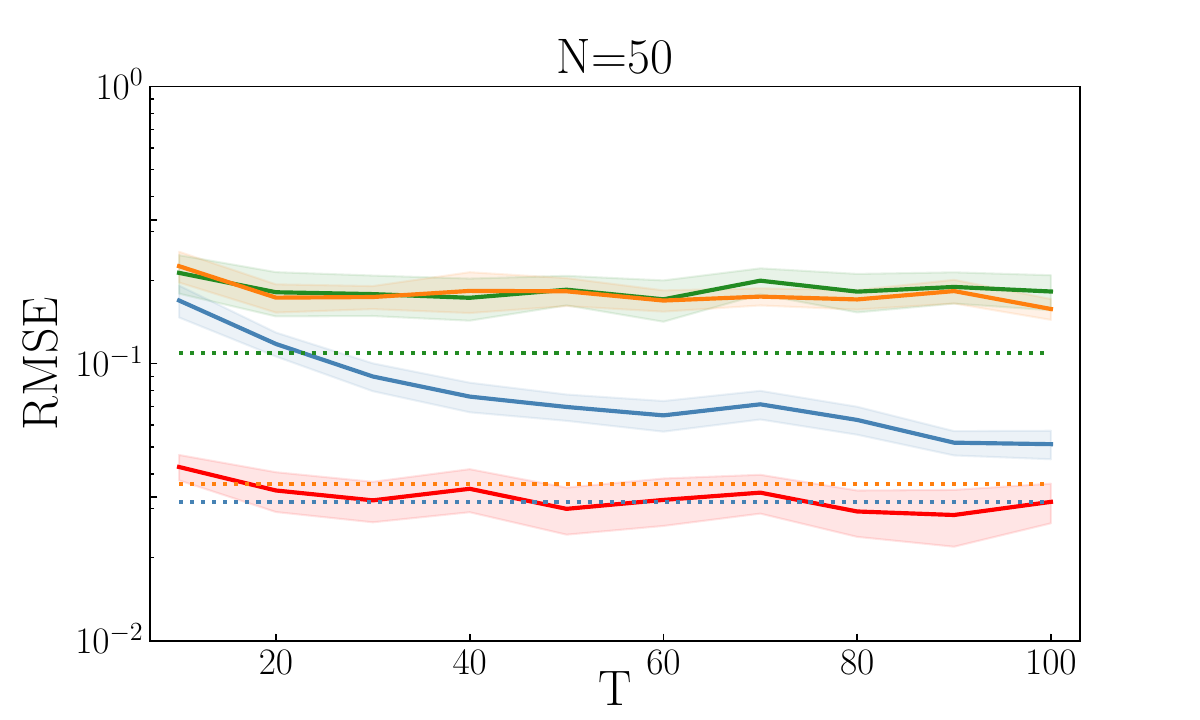}
        \label{fig:bayes_sensitivity_1}
    \end{subfigure}%
    \hfill 
    \begin{subfigure}{0.33\textwidth}
        \centering
        \hspace{-10pt}
        \includegraphics[width=1.0\linewidth]{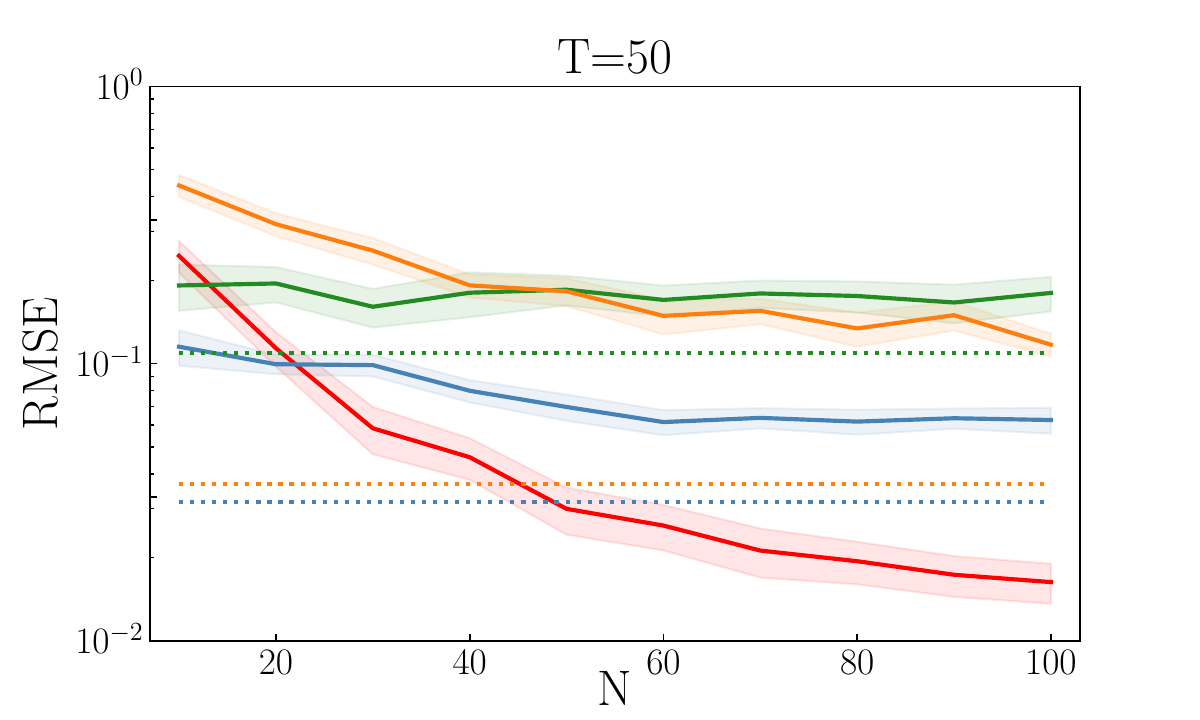}
        \label{fig:bayes_sensitivity_2}
    \end{subfigure}%
    \hfill 
    \begin{subfigure}{0.33\textwidth}
        \centering
        \hspace{-10pt}
        \includegraphics[width=1.0\linewidth]{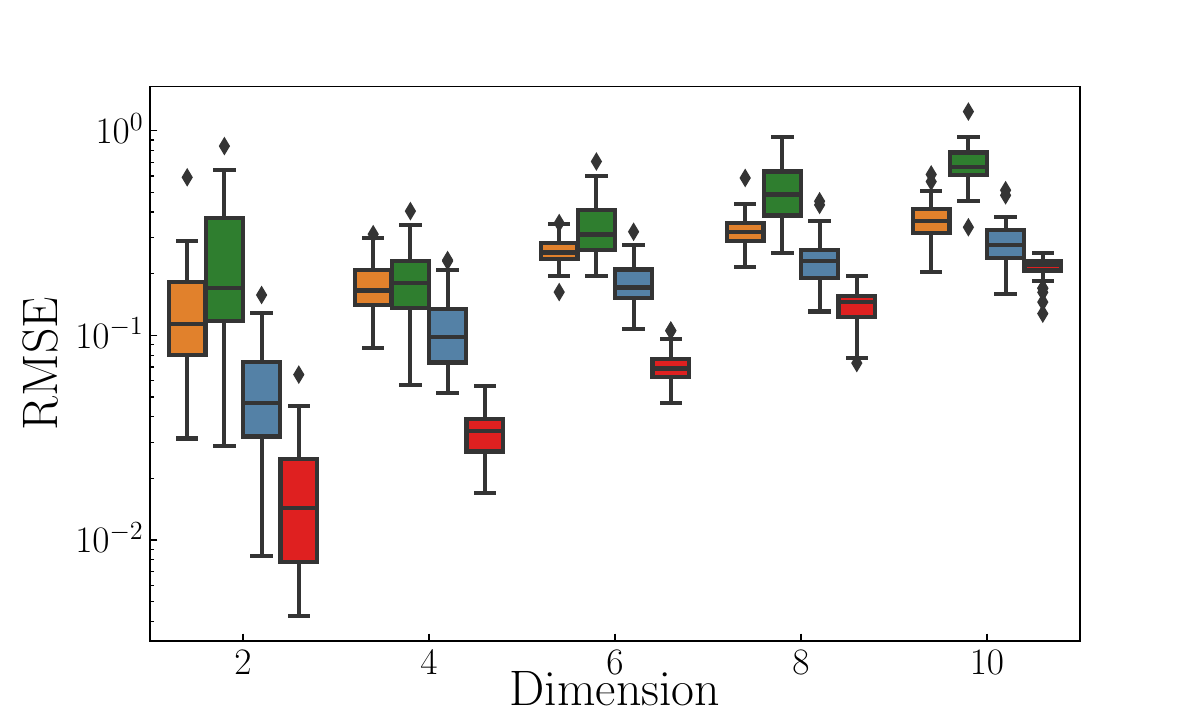}
        \label{fig:bayes_sensitivity_3}
    \end{subfigure}
    \vspace{-3pt}
    \caption{\emph{Bayesian sensitivity analysis for linear models.} \textbf{Left:} RMSE of all methods when $d=2$ and $N=50$. \textbf{Middle:} RMSE of all methods when $d=2$ and $T=50$. \textbf{Right:} RMSE of all methods when $N=T=100$.}
    \label{fig:bayes_sensitivity}
\end{figure*}

Although the bound is dominated by a term $\calO(T^{-1/4})$ in $T$, the proof can be extended to provide a more general result with rate up to $\calO(T^{-1/3})$ under an additional ``source condition'' which requires stronger regularity from $f$; this is further discussed in~\Cref{appendix:convergence_rate}. 
The latter rate is minimax optimal for any nonparametric regression-based method \citep{Stone1982}. Compared to baselines, we note that we cannot expect a similar result for IS since IS does not apply when $f$ depends on $\theta$. 
For LSMC, we also cannot guarantee consistency of the algorithm when $I(\theta)$ is not a polynomial (unless $p \rightarrow \infty$; see \cite{stentoft2004convergence}). 
Although we are not aware of any such result, we expect KLSMC to have the same rate in $T$ as CBQ, and for CBQ to be significantly faster than KLSMC in $N$. This is due to the second stage of KLSMC being essentially the same as that for CBQ, and KLSMC using MC rather than BQ in the first stage: by~\cite{novak1988deterministic}, the convergence rate of BQ, $N^{-s_\mathcal{X}/d}$, is faster than that of MC, $N^{-1/2}$, in the case where the function $x \to f(x, \theta)$ is of smoothness at least $s_\mathcal{X} > d/2$.

%% file: uai_2024/5_experiments.tex
\section{Experiments}\label{sec:experiments}

We will now evaluate the empirical performance of CBQ against baselines including IS, LSMC and KLSMC. 
For the first three experiments, we focus on the case where $f$ does not depend on $\theta$ (i.e. $f(x, \theta) = f(x)$), and for the fourth experiment we focus on the case where $f$ depends on both $x$ and $\theta$. 
All methods use $\theta_{1:T} \sim \mathbb{Q}$ ($\mathbb{Q}$ is specified individually for each experiment) and $x_{1:N}^t \sim \mathbb{P}_{\theta_t}$ to ensure a fair comparison, and we therefore use $\mathbb{P}_{\theta_1}, \ldots, \mathbb{P}_{\theta_T}$ as our importance distributions in IS. 
For experiments on nested expectations, we use standard Monte Carlo for the outer expectation and use CBQ along with all baseline methods to compute conditional expectation for the inner expectation.

Detailed descriptions of hyperparameter selection for CBQ and all baseline methods can be found in \Cref{appendix:practical_considerations}. 
Detailed experimental settings can be found in \Cref{appendix:bayes_sensitivity} to \Cref{appendix:decision} along with detailed checklists on whether the assumptions of \Cref{thm:convergence} can be satisfied in each experiment.
We also provide additional experiments in \Cref{appendix:experiments}. 
\Cref{appendix:cbq_mobq} includes experiments which show MOBQ obtains similar performance to CBQ, but with a computational cost which is between $10$ and $100$ times larger.
\Cref{appendix:QMC} includes experiments with quasi-Monte Carlo points~\citep{ hickernell1998generalized}, which demonstrates that CBQ is not limited to independent samples.  
\Cref{appendix:ablation} includes ablation studies on various kernels $k_\calX$ and $k_\Theta$.
\Cref{appendix:calibration} demonstrates the calibration of CBQ uncertainty.
The code to reproduce all the results in this section is available at the following GitHub repository
\href{https://github.com/hudsonchen/CBQ}{\texttt{https://github.com/hudsonchen/cbq}}. 

\paragraph{Synthetic Experiment: Bayesian Sensitivity Analysis for Linear Models.}
The prior and likelihood in a Bayesian analysis often depend on hyperparameters, and determining the sensitivity of the posterior to these is critical for assessing robustness~\citep{oakley2004probabilistic,Kallioinen2021}. One way to do this is to study how posterior expectations of interest depend on these hyperparameters, a task usually requiring the computation of conditional expectations. We consider this problem in the context of Bayesian linear regression with a zero-mean Gaussian prior with covariance $\theta \Id_d$ where $\Id_d$ is identity matrix and $\theta \in (1, 3)^d$. Using a Gaussian likelihood, we can obtain a conjugate Gaussian posterior $\mathbb{P}_{\theta}$ on the regression weights. We can then analyse sensitivity by computing the conditional expectation $I(\theta)$ of some quantity of interest $f$. For example, if  $f(x)=x^\top x$, then $I(\theta)$ is the second moment of the posterior, whereas if $f(x) = x^\top y^\ast$ for some new observation $y^\ast$, then $I(\theta)$ is the predictive mean. In these simple settings, $I(\theta)$ can be computed analytically, making this a good synthetic example for benchmarking.

\begin{figure}[h]
    \centering
    \includegraphics[width=0.33\textwidth]{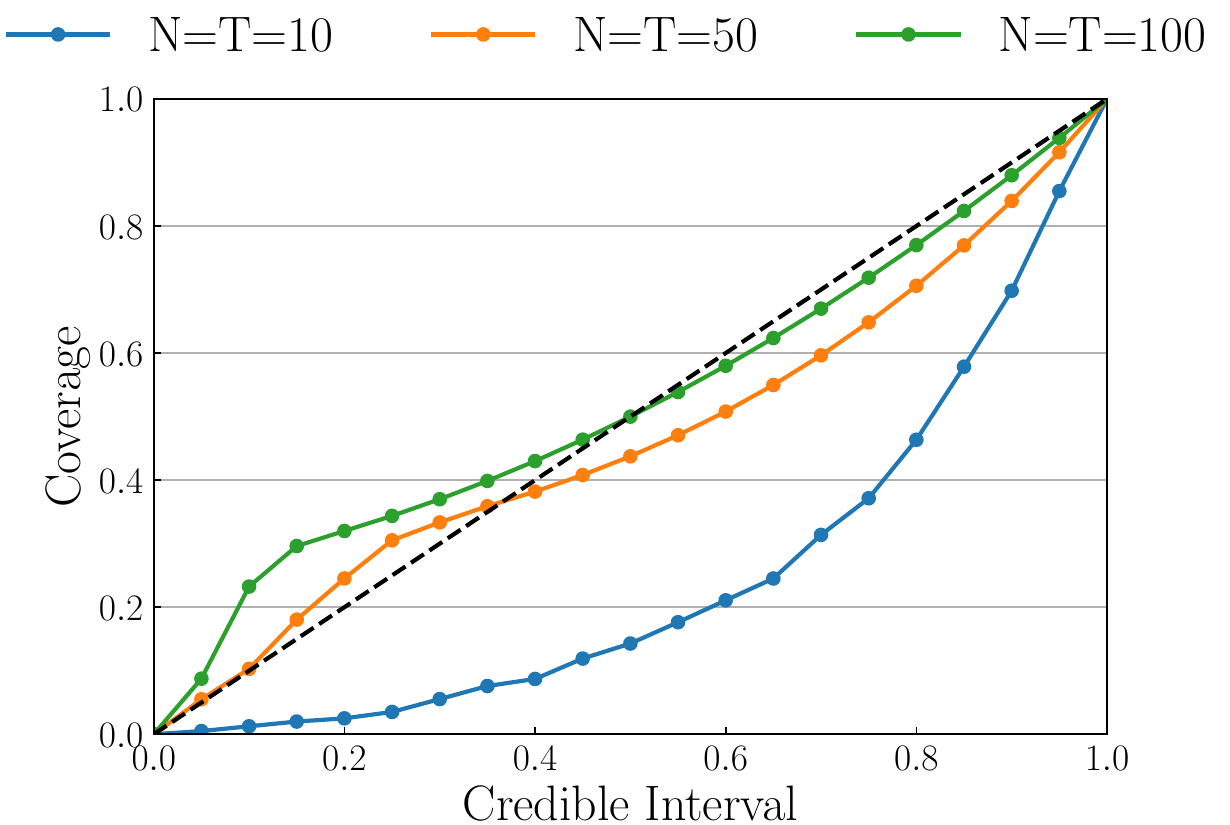}
    \vspace{-5pt}
    \caption{\emph{Bayesian linear model sensitivity analysis in $d=2$}.}
    \label{fig:calib_bayes}
\end{figure}

Our results in \Cref{fig:bayes_sensitivity} are for the second moment, whilst the results for the predictive mean are in \Cref{appendix:bayes_sensitivity}.  
We measure performance in terms of root mean squared error (RMSE) and use $ \mathbb{Q} = \operatorname{Unif}(1, 3)^d$.  For CBQ, $k_\calX$ is chosen to be a Gaussian kernel so that the kernel mean embedding $\mu$ has a closed form, and $k_\Theta$ is a Mat\'ern-3/2 kernel.
\Cref{fig:bayes_sensitivity} shows the performance of CBQ against baselines with varying $N$, $T$ and $d$. LSMC performs well for this problem, and this can be explained by the fact that $I(\theta)$ is a polynomial in $\theta$.
Despite this, the left and middle plots show that CBQ consistently outperforms all competitors. Specifically, its rate of convergence is initially much faster in $N$ than in $T$, which confirms the intuition from \Cref{thm:convergence}. The dotted lines also give the performance of baselines under a very large number of samples $N=T=1000$, and we see that CBQ is either comparable or better than these even when it has access only to much smaller $N$ and $T$. In the right-most panel, we see that the baselines gradually catch up with CBQ as $d$ grows, which is again expected since the rate in \Cref{thm:convergence} is $O(N^{-2 s_{\calX}/d+\varepsilon})$ in $N$. 
Additional experimental results demonstrating these are consistent conclusions for different values of $N, T$ can be found in \Cref{appendix:bayes_sensitivity}.

Our last plot is in \Cref{fig:calib_bayes} and studies the calibration of the CBQ posterior. The coverage is the $\%$ of times a credible interval contains $I(\theta)$ under repetitions of the experiment. The black diagonal line represents perfect calibration, whilst any curve lying above or below the black line indicates underconfidence or overconfidence respectively. 
We observe that when $N$ and $T$ are as small as $10$, CBQ is overconfident. When $N$ and $T$ increase, CBQ becomes underconfident, meaning that our posterior variance is more inflated than needed from a frequentist viewpoint.
Calibration plots for the rest of the experiments can be found in \Cref{appendix:experiments} and demonstrate similar results. 
It is generally preferable to be under-confident than overconfident, and CBQ does a good job most of the time. 
We expect that overconfidence in small $N$ and $T$ can be explained by a poor performance of empirical Bayes, and therefore caution users to not overly rely on the reported uncertainty in this regime.


\begin{figure*}[t]
\vspace{-10pt}
    \begin{minipage}{0.99\textwidth}
    \centering
    \includegraphics[width=480pt]{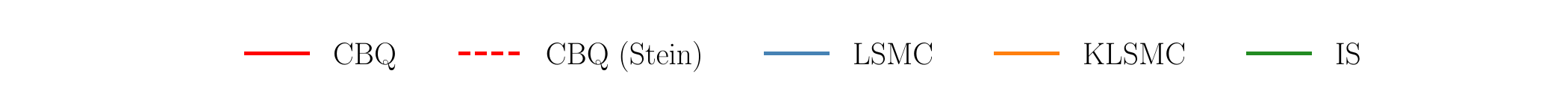}
    \vspace{-10pt}
    \end{minipage}
    
    \centering
    
    \begin{subfigure}{0.33\textwidth}
        \centering
        \includegraphics[width=1.0\linewidth]{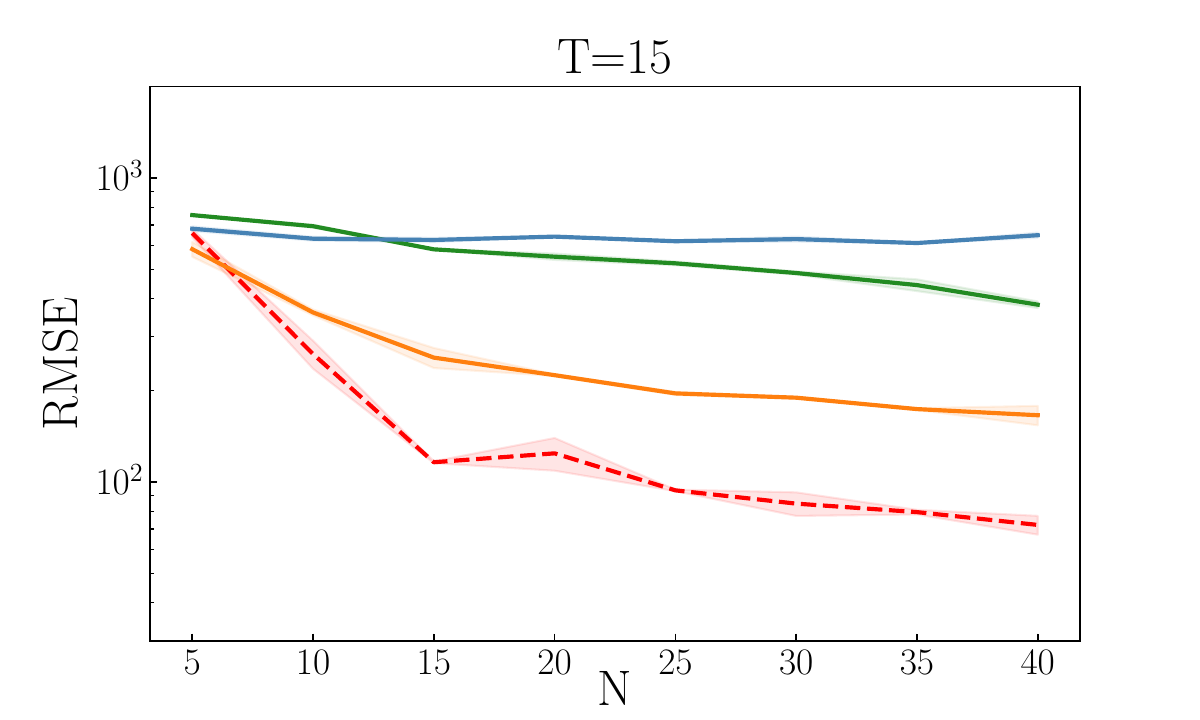}
        \label{fig:sir}
    \end{subfigure}%
    \hfill 
    \begin{subfigure}{0.33\textwidth}
        \centering
        \includegraphics[width=1.0\linewidth]{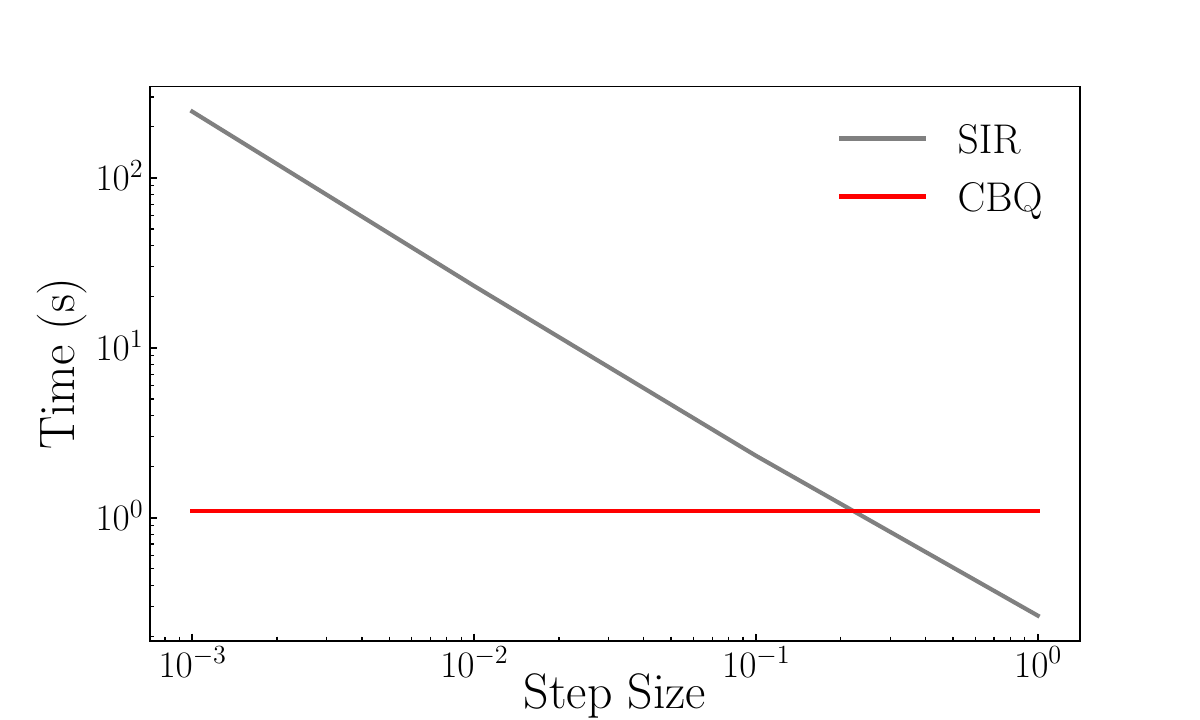}
        \label{fig:sir_time}
    \end{subfigure}
        \hfill
    \begin{subfigure}{0.33\textwidth}
        \centering
        \includegraphics[width=1.0\linewidth]{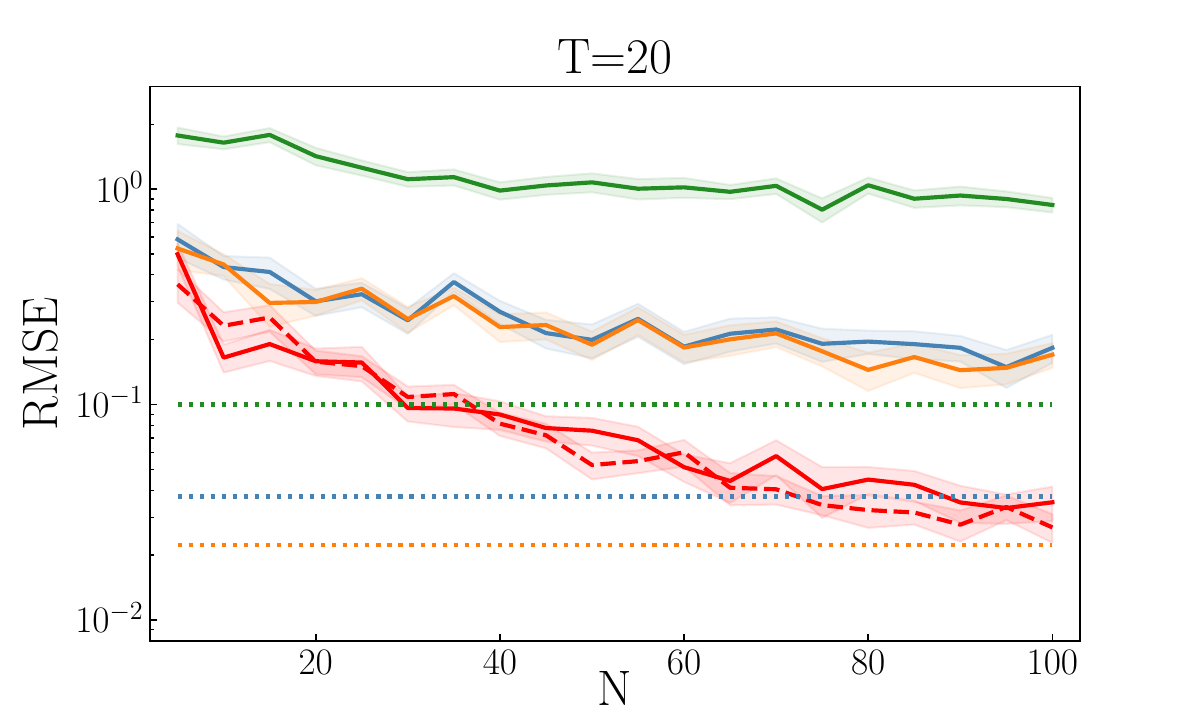}
        \label{fig:finance}
    \end{subfigure}%
    \vspace{-15pt}
    \caption{\emph{Bayesian sensitivity analysis for SIR Model $\&$ Option pricing in mathematical finance.} \textbf{Left:} RMSE of all methods for the SIR example with $T=15$. \textbf{Middle:} The computational cost (in wall clock time) for CBQ ($T=15, N=40$) and for obtaining one single numerical solution from SIR under different discretization step sizes. In practice, the process of obtaining samples from SIR equations is repeated $NT$ times.
    \textbf{Right:} RMSE of all methods for the finance example with $T=20$.
    }
    \label{fig:finance_sir}
\end{figure*}

\paragraph{Bayesian Sensitivity Analysis for the Susceptible-Infectious-Recovered (SIR) Model.}
The SIR model is commonly used to simulate the dynamics of infectious diseases through a population~\citep{kermack1927sir}. In this model, the dynamics are governed by a system of differential equations parametrised by a positive infection rate and a recovery rate (see \Cref{appendix:sir}). 
The accuracy of the numerical solution to this system typically hinges on the step size.
While smaller step sizes yield more accurate solutions, they are also associated with a much higher computational cost. 
For example, using a step size of $0.1$ days for simulating a $150$-day period would require a computation time of $3$ seconds for generating a single sample, which is more costly than running CBQ on $N=40, T=15$ samples.
The cost would become even larger as the step size gets smaller, as depicted in the middle panel of \Cref{fig:finance_sir}. Consequently, when performing Bayesian sensitivity for SIR, there is clear necessity for more data-efficient algorithms such as CBQ.

We perform a sensitivity analysis for the parameter $\theta$ of our $\operatorname{Gamma}(\theta, 10)$ prior on the infection rate $x$. The parameter $\theta$ represents the initial belief of the infection rate deduced from the study of the virus in the laboratory at the beginning of the outbreak.
We are interested in the expected peak number of infected individuals: $f(x)= \max_r N^r_I(x)$, where $N^r_I(x)$ is the solution to the SIR equations and represents the number of infections at day $r$. 
It is important to study the sensitivity of $I(\theta)$ to the shape parameter $\theta$. 
The total population is set to be $10^6$ and 
$\mathbb{Q} = \operatorname{Unif}\left(2,9\right)$ and $\mathbb{P}_{\theta_t} = \operatorname{Gamma}(\theta_t, 10)$. 
We use a Monte Carlo estimator with $5000$ samples as the pseudo ground truth and evaluate the RMSE across all methods. 
For CBQ, we employ a Stein kernel for $k_\calX$, with the Mat\'ern-3/2 as the base kernel, and $k_\Theta$ is selected to be a Mat\'ern-3/2 kernel.

We can see in the left panel of \Cref{fig:finance_sir} that CBQ clearly outperforms baselines including IS, LSMC and KLSMC in terms of RMSE.
Although the CBQ estimator exhibits a higher computational cost compared to baselines, we have demonstrated in the middle panel of \Cref{fig:finance_sir} that, due to the increased computational expense of obtaining samples with smaller step size, using CBQ is ultimately more efficient overall within the same period of time.
Additional experimental results demonstrating these are consistent conclusions for different values of $T$ can be found in \Cref{appendix:sir}.


\paragraph{Option Pricing in Mathematical Finance.}
Financial institutions are often interested in computing the expected loss of their portfolios if a shock were to occur in the economy, which itself requires the computation of conditional expectations (it is in fact in this context that LSMC and KLSMC was first proposed). This is typically a challenging computational problem since simulating from the stock of interest often requires the numerical solution of stochastic differential equations over a long time horizon (see \cite{achdou2005computational}), making data-efficient methods such as CBQ particularly desirable.

Our next experiment is representative of this class of problems, but has been chosen to have a closed-form expected loss and to be amenable to cheap simulation of the stock to enable extensive benchmarking. We consider a butterfly call option whose price $S(\tau)$ at time $\tau \in [0,\infty)$ follows the Black-Scholes formula; see \Cref{appendix:black_scholes} for full details. The payoff at time $\tau$ can be expressed as
$\psi(S({\tau}))=\max (S(\tau)-K_1, 0) + \max (S(\tau)-K_2, 0) - 2\max (S(\tau) - (K_1+K_2)/2, 0)$ for two fixed constants $K_1,K_2\geq 0$.
We follow the set-up in \cite{alfonsi2021multilevel, alfonsi2022many} assuming that a shock occurs at time $\eta$ when the price is $S(\eta)=\theta \in (0,\infty)$, and this shock multiplies the price by $1 + s$ for some $s\geq 0$. As a result, the expected loss of the option is $\calL = \E_{\theta \sim \mathbb{Q}}
[ \max ( I(\theta), 0)]$, where $I(\theta) =  \int_{0}^\infty f(x) \mathbb{P}_\theta(dx)$, $x=S(\zeta)$ is the price at the time $\zeta$ at which the option matures, $f(x) = \psi(x)-\psi((1+s)x)$, and $\mathbb{P}_\theta$ and $\mathbb{Q}$ are two log-normal distributions induced from the Black-Scholes model. 

Results are presented in the right-most panel of \Cref{fig:finance_sir}. We take $K_1 = 50, K_2 = 150, \eta=1, s = 0.2$ and $\zeta=2$. 
For CBQ, $k_\Theta$ is selected to be a Mat\'ern-3/2 kernel and $k_\calX$ is either a Stein kernel with Mat\'ern-3/2 as base kernel or a logarithmic Gaussian kernel (see \Cref{appendix:black_scholes}) in which case $k_{\calX}$ is too smooth to satisfy the assumption of our theorem. 

As expected, CBQ exhibits much faster convergence in $N$ than IS, LSMC or KLSMC, and outperforms these baselines even when they are given a substantial sample size of $N=T=1000$ (see dotted lines). We can also see that CBQ with the log-Gaussian  kernel or with Stein kernel have similar performance, despite the log-Gaussian kernel not satisfying the smoothness assumptions of our theory. Additional experiments in \Cref{appendix:black_scholes} show that these results are consistent for different values of $T$.

\begin{figure}[h]
\vspace{-10pt}
  \centering
  \hspace{-20pt}
  \includegraphics[width=0.8\linewidth]{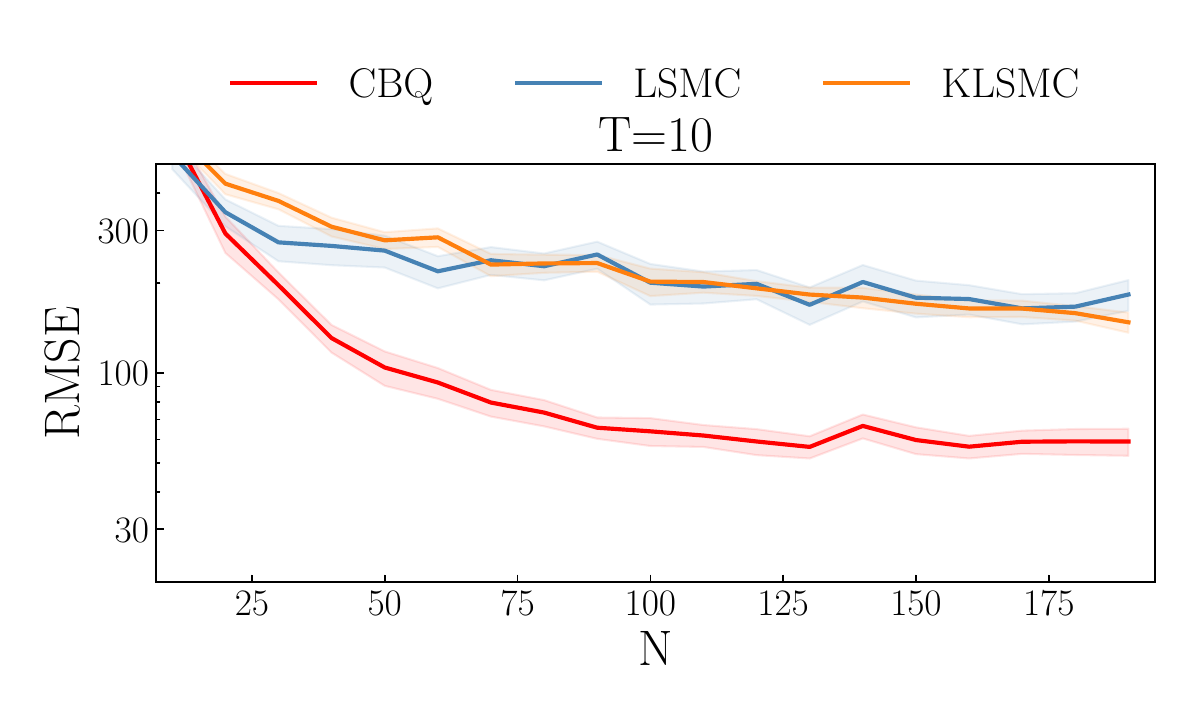}
  \vspace{-10pt}
  \caption{\emph{Uncertainty decision making in health economics.} We study RMSE for different estimators of EVPPI.}
  \label{fig:decision}
  \vspace{-15pt}
\end{figure}

\paragraph{Uncertainty Decision Making in Health Economics.} 
In the medical world, it is important to trade-off the costs and benefits of conducting additional experiments on patients.
One important measure in this context is the expected value of partial perfect information (EVPPI), which quantifies the expected gain from conducting experiments to obtain precise knowledge of some unknown variables \citep{brennan2007calculating}. 
The EVPPI can be expressed as $\E_{\theta \sim \mathbb{Q}}[\max_c I_c(\theta) ] - \max_c \E_{\theta \sim \mathbb{Q}}[I_c(\theta)]$ where $f_c$ represents a measure of patient outcome (such as quality-adjusted life-years) under treatment $c$ among a set of potential treatments $\calC$, $\theta$ denotes the additional variables we could measure, and $I_c(\theta) = \int_{\calX} f_c(x, \theta) \mathbb{P}_\theta(dx)$ denotes the expected patient outcome given our measurement of $\theta$. We highlight that for these applications $N$ and $T$ are often small due to the very high monetary cost and complexity of collecting patient data in real world.

We study the potential use of CBQ for this problem using the synthetic problem of \cite{Giles2019}, where $\mathbb{P}_{\theta}$ and $\mathbb{Q}$ are Gaussians (see \Cref{appendix:decision}). We compute EVPPI with $f_1(x, \theta)=10^4 (\theta_1 x_5 x_6 + x_7 x_8 x_{9})-(x_1 + x_2 x_3 x_4)$ and $f_2(x, \theta) = 10^4 (\theta_2 x_{13} x_{14} + x_{15} x_{16} x_{17})-(x_{10} + x_{11} x_{12} x_4)$. 
The exact practical meanings of $x$ and $\theta$ can be found in \Cref{appendix:decision}.
We draw $10^6$ samples from the joint distribution to generate a pseudo ground truth, and evaluate the RMSE across different method.
Note that IS is no longer applicable here because $f$ depends on both $x$ and $\theta$, so we only compare against KLSMC and LSMC. 
For CBQ, $k_\calX$ is a Mat\'ern-3/2 kernel and $k_\Theta$ is also a Mat\'ern-3/2 kernel. 
In \Cref{fig:decision}, we can see that CBQ consistently outperforms baselines with much smaller RMSE. The results are also consistent with different values of $T$; see \Cref{appendix:decision}.

%% file: uai_2024/6_conclusion.tex
\section{Conclusions}
We propose CBQ, a novel algorithm which is tailored for the computation of conditional expectations in the setting where obtaining samples or evaluating functions is costly. 
We show both theoretically and empirically that CBQ exhibits a fast convergence rate, and provides the additional benefit of Bayesian quantification of uncertainty. 
Looking forward, we believe further gains in accuracy could be obtained by developing active learning schemes to $N$, $T$, and the location of $\theta_{1:T}$ and $x_{1:N}^t$ for all $t$ in an adaptive manner. 
Additionally, CBQ could be extended for nested expectation problems by using a second level of BQ based on the output of second stage heteroscedastic GP, potentially leading to a further increase in accuracy.

\paragraph{Ackowledgments} The authors would like to thank Motonobu Kanagawa for some helpful discussions and pointers to the literature. ZC, MN and FXB acknowledge support from the Engineering and Physical Sciences Research Council (ESPRC) through grants [EP/S021566/1] and [EP/Y022300/1].  AG was partly supported by the Gatsby Charitable Foundation.

%% file: uai_2024/appendix.tex
\begin{appendices}
\crefalias{section}{appendix}
\crefalias{subsection}{appendix}
\crefalias{subsubsection}{appendix}

\setcounter{equation}{0}
\renewcommand{\theequation}{\thesection.\arabic{equation}}

\onecolumn

{\hrule height 1mm}
\vspace*{-0pt}
\section*{\LARGE\bf \centering Supplementary Material
}
\vspace{8pt}
{\hrule height 0.1mm}
\vspace{24pt}

\section*{Table of Contents}
\vspace*{-10pt}
\startcontents[sections]
\printcontents[sections]{l}{1}{\setcounter{tocdepth}{2}}

\newpage

\section{Theoretical Results}
\label{appendix:convergence_rate}

To validate our methodology, we established a rate at which the CBQ estimator converges to the true value of the conditional expectation $I$ in the $\calL^2(\Theta)$ norm, $\|I_\mathrm{CBQ} - I\|_{\calL^2(\Theta)}=\int_\Theta (I_\mathrm{CBQ}(\theta) - I(\theta))^2 \mathrm d \theta$ in~\Cref{thm:convergence_generalised}. The more specific version of this result was presented in the main text in~\Cref{thm:convergence}. In this section, we prove a more general version of~\Cref{thm:convergence} (as well as several intermediate results), and expand on the technical background required.

For the duration of the appendix, we will denote by $M$ the total number of points in $\Theta$ instead of $T$, to avoid notation clashes with the integral operator $T$. Additionally, we will be explicit on the dependency of the BQ mean $I_\mathrm{BQ}$ and variance $\sigma^2_\mathrm{BQ}$ at the point $\theta$ on the realisations $x_{1:N}^\theta \sim \Pb_\theta$, meaning
\begin{align*}
    I_\mathrm{BQ}(\theta; x_{1:N}^\theta) &= \mu^\top_\theta(x_{1:N}^\theta) \left(k_\calX(x_{1:N}^\theta, x_{1:N}^\theta) + \lambda_\calX \Id_N \right)^{-1} f(x_{1:N}^\theta, \theta) \\
    \sigma^2_\mathrm{BQ}(\theta; x_{1:N}^\theta) &= \E_{X, X' \sim \Pb_\theta} [k_\calX(X, X')] -\mu^\top_\theta(x_{1:N}^\theta) \left(k_\calX(x_{1:N}^\theta, x_{1:N}^\theta) + \lambda_\calX \Id_N \right)^{-1} \mu_\theta(x_{1:N}^\theta)
\end{align*}
Finally, whenever $x_{1:N}^{\theta_t}$, we will shorten it to $x_{1:N}^t$ to avoid bulky notation. The rest of the section in structured as follows. In~\Cref{sec:technical_assumptions} we present technical assumptions, and state in~\Cref{thm:convergence_generalised} the main convergence result the proof of which is deferred until the necessary Stage 1 and 2 results are proven. In~\Cref{sec:stage1}, we provide the necessary Stage 1 bounds that will be used in the proof of the main result. In~\Cref{sec:stage2}, we provide the necessary auxiliary results and the bound for Stage 2 in terms of Stage 1 errors. Finally, in~\Cref{sec:proof_of_main_theorem} we combine the bounds from both stages to prove~\Cref{thm:convergence_generalised}, the more general version of~\Cref{thm:convergence}.

\subsection{Main Result}
\label{sec:technical_assumptions}
Prior to presenting our findings, we present and justify the assumptions we have made. Throughout we use Sobolev spaces to quantify a function's smoothness. A Sobolev space $\calW^{s, 2}(\calX, \mu)$, with $s>d/2$ and a measure $\mu$ on $\calX \subseteq \R^d$, consists of functions that satisfy certain conditions: they are square integrable under the measure $\mu$, and all weak derivatives up to and including order $s$ are also square integrable under $\mu$. Weak derivatives are a generalization of ordinary derivatives, allowing for functions that are not necessarily differentiable everywhere. We write $\theta = \begin{bmatrix} \theta_{(1)} & \dots & \theta_{(p)}\end{bmatrix}$ for any $\theta \in \Theta \subseteq \R^p$.  For a multi-index $\alpha = (\alpha_1, \dots \alpha_p) \in \mathbb{N}^p$, by $D_\theta^\alpha g$ we denote the $|\alpha|=\sum_{i=1}^d \alpha_i$ order weak derivative $D_\theta^\alpha g=D^{\alpha_1}_{\theta_{(1)}} \dots D^{\alpha_p}_{\theta_{(p)}} g$ for a function $g$ on $\Theta \subseteq \R^p$.
Further, we assume the kernels $k_\Theta, k_\calX$ are Sobolev kernels, meaning they induce Hilbert spaces that are norm-equivalent to Sobolev spaces (two normed spaces $(\calX, \|\cdot\|_\calX), (\calY, \|\cdot\|_\calY)$ are said to be norm-equivalent when $\calX=\calY$ as sets, and there are real constants $C'_1, C'_2>0$ such that for any $x \in \calX$ it holds that $ C'_1 \|x\|_\calY \leq \|x\|_\calX \leq C'_2 \|x\|_\calY$).

Mat\'ern kernels are important examples of Sobolev kernels. It is well-known that the RKHS of a Mat\'ern kernel of order $\nu_\Theta$ over an open, convex and bounded $\Theta \subset \R^p$ is norm-equivalent to the Sobolev space $W^{2,\nu_\Theta+p/2}(\Theta)$ when $\nu_\Theta+p/2 \in \mathbb Z$; this is proven in~\citep[Corollary 10.48]{Wendland2005}. For $\Theta=\R^p$, the result can be straightforwardly extended to fractional order Sobolev-Slobodeckij spaces, $\nu_\Theta+p/2 \in \mathbb R$: by~\cite[Corollary 10.13]{Wendland2005} the RKHS of a Mat\'ern kernel on $\R^p$ is norm-equivalent to a Bessel potential space, which in turn is norm-equivalent to the Sobolev-Slobodeckij space by~\cite[Section 7.62]{adams2003sobolev}. Finally, one can use an extension operator in~\cite[Theorems 6.1 and 6.7]{devore1993besov} to restrict the norm-equivalence result to open, convex and bounded $\Theta \subset \R^p$. 

The following is a more general form of the assumptions in~\Cref{thm:convergence}: specifically, we allow for the case when $\theta_{1:T}$ came from a distribution that doesn't necessarily have a density, and do not assume $\lambda_\calX=0$.
\begin{enumerate}[itemsep=0.1pt,topsep=0pt,leftmargin=*]
\item [B0] 
\begin{enumerate}
    \item[(a)] $f(x, \theta)$ lies in the Sobolev space $\calW^{s_f, 2}(\calX)$ for any $\theta \in \Theta$. 
    \customlabel{as:app_true_f_smoothness}{B0.(a)} 
    \item[(b)] $f(x, \theta)$ lies in the Sobolev space $\calW^{s_I, 2}(\Theta)$ for any $x \in \calX$.
    \customlabel{as:app_true_I_smoothness}{B0.(b)} 
    \item[(c)] $M_f = \sup_{\theta \in \Theta} \max_{|\alpha|<s_I} \| D_\theta^\alpha f(\cdot, \theta) \|_{\calW^{s_I, 2}(\calX)}<\infty$.
    \customlabel{as:app_true_I_norm_bounds}{B0.(c)} 
\end{enumerate}
\item [B1] 
\begin{enumerate}
    \item[(a)] $\calX \subset \R^d$ is open, convex, and bounded.  
    \customlabel{as:app_domains_x}{B1.(a)} 
    \item[(b)] $\Theta\subset \R^p$ is open, convex, and bounded. 
    \customlabel{as:app_domains_theta}{B1.(b)} 
\end{enumerate}
\item [B2]
\begin{enumerate}
    \item[(a)] $\theta_t$ were sampled i.i.d. from some $\Qb$, and $\Qb$ is equivalent to the uniform distribution on $\Theta$, meaning $\Qb(A)=0$ for a set $A \subset \Theta$ if and only if $\operatorname{Unif}(A)=0$. 
    \customlabel{as:app_theta_samples}{B2.(a)} 
    \item[(b)] $x_{1:N}^t \sim \Pb_{\theta_t}$ for all $t \in \{1, \cdots, T\}$.  
    \customlabel{as:app_x_samples}{B2.(b)}
\end{enumerate}
\item [B3] $\Pb_\theta$ has a density $p_\theta$ for any $\theta \in \Theta$, and the densities are such that 
    \begin{enumerate}
    \item[(a)] $\inf_{\theta \in \Theta, x \in \calX} p_\theta(x)=\eta>0$ and $\sup_{\theta \in \Theta}\|p_\theta\|_{\calL^2(\calX)}=\eta_0<\infty$.
    \customlabel{as:app_densities}{B3.(a)}
    \item[(b)] $p_\theta(x)$ lies in the Sobolev space $\calW^{s_I, 2}(\Theta)$ for any $x \in \calX$.
    \customlabel{as:app_densities}{B3.(b)}
    \item[(c)] 
    $M_p = \sup_{\substack{\theta \in \Theta\\x \in \calX}}\max_{|\alpha|\leq s_I} |D_\theta^\alpha p_\theta(x)|<\infty$. 
    \customlabel{as:app_densities_linf}{B3.(c)}
    \end{enumerate}
\item [B4] 
\begin{enumerate}
    \item[(a)] $k_\calX$ is a Sobolev kernel of smoothness $s_\calX \in (d/2, s_f]$. 
    \customlabel{as:app_kernel_x}{B4.(a)}
    \item[(b)] $k_\Theta$ is a Sobolev kernel of smoothness $s_\Theta \in (p/2, s_I]$.
    \customlabel{as:app_kernel_theta}{B4.(b)}
    \item[(c)] $\kappa = \sup_{\theta \in \Theta}k_\Theta(\theta, \theta)<\infty$.
    \customlabel{as:app_kernel_theta_bounded}{B4.(c)}
\end{enumerate}
\item [B5]
\begin{enumerate}
    \item[(a)] $\lambda_\Theta = cM^{1/2}$, for $c>(4/C_6) \kappa \log(4/\delta)$ for some $C_6\leq1$. 
    \customlabel{as:app_regulariser_theta}{B5.(a)}
    \item[(b)] $\lambda_\calX \geq 0$. 
    \customlabel{as:app_regulariser_x}{B5.(b)}
\end{enumerate}
\end{enumerate}

Assumption B0 corresponds to conditions specified in the text of~\Cref{thm:convergence} prefacing the list of assumptions. Assumption~\ref{as:app_true_I_smoothness} implies $I(\theta) \in \calW^{s_I, 2}(\Theta)$: $f(x, \theta) p_\theta(x) \in \calW^{s_I, 2}(\Theta)$ by the product rule for weak derivatives (see, for instance,~\citet[Section 4.2.2]{evans2018measure}), and the integral lies in $\calW^{s_I, 2}(\Theta)$ by $\calW^{s_I, 2}(\Theta)$ being a complete space. 
Assumption~\ref{as:app_true_I_norm_bounds} ensures the $\calX$-Sobolev norm of any weak derivative of $\theta \to f(\cdot, \theta)$ is uniformly bounded across all $\theta$; this will be satisfied unless $f$ is so irregular said Sobolev norms can get arbitrarily close to infinity.
Assumption~\ref{as:app_densities_linf}, similarly, ensures that any weak derivative of $\theta \to p_\theta(x)$ is bounded across all $\theta$ and $x$.
It is worth pointing out assumption~\ref{as:app_kernel_theta_bounded}, boundedness of the kernel, follows from assumption~\ref{as:app_kernel_theta}; however, we keep it separate as some results will only require that the kernel is bounded, not necessarily that it is Sobolev.

Crucially, in the proofs in the next section we will see that the assumptions imply that the setting of the model in Stage 1 satisfies the assumptions of~\cite[Theorem 4]{wynne2021convergence}, and the setting of the model in Stage 2 satisfies the assumptions necessary to establish convergence of a noisy importance-weighted kernel ridge regression estimator---the two key results we will use to prove the convergence rate of the estimator.

We now state the main convergence result, which is a version of~\Cref{thm:convergence} for $\lambda_\calX \geq 0$. The proof of both this result and the more specific~\Cref{thm:convergence} are postponed until Section BLAH, as they rely on intermediary results.

\begin{theorem}[Generalised~\Cref{thm:convergence}]
\label{thm:convergence_generalised}
    Suppose all technical assumptions in~\Cref{sec:technical_assumptions} hold. Then for any $\delta \in (0, 1)$ there is an $N_0>0$ such that for any $N \geq N_0$, with probability at least $1-\delta$ it holds that
    \begin{align*}
        \| I_\mathrm{CBQ} - I \|_{\calL^2(\Theta, \Qb)} &\leq \left(1 + c^{-1} M^{-\frac{1}{2}}\left(\lambda_\calX + C_2 N^{-1 + 2\varepsilon} \left( N^{-\frac{s_\calX}{d}+\frac{1}{2} + \varepsilon} + C_3 \lambda_\calX \right)^2\right) \right) \\
        &\qquad \times \left( C_7(\delta) N^{-\frac{1}{2} + \varepsilon} \left( N^{-\frac{s_\calX}{d} + \frac{1}{2} + \varepsilon} + C_5 \lambda_\calX \right) + C_8(\delta) M^{-\frac{1}{4}} \| I\|_{\calH_\Theta} \right)
    \end{align*}
    for any arbitrarily small $\varepsilon>0$, constants $C_2, C_3, C_5$, $C_7(\delta) = O(1/\delta)$ and $C_8(\delta) = O(\log(1/\delta))$ independent of $N, M, \varepsilon$.
\end{theorem}

\subsection{Stage 1 bounds}
\label{sec:stage1}

Recall that we use the shorthand $x_{1:N}^t$ for $x_{1:N}^{\theta_t}$. In this section, we bound the BQ variance $\sigma^2_\mathrm{BQ}(\theta; x_{1:N}^\theta)$ in expectation in~\Cref{res:bound_on_bq_var}, and the difference between $I_\mathrm{BQ}(\theta; x_{1:N}^\theta)$ and $I$ in the norm of the RKHS $\calH_\Theta$ induced by the kernel $k_\Theta$ in~\Cref{res:bound_on_bq_error}. Later in~\Cref{sec:stage2}, the error of the estimator $I_\mathrm{CBQ}$ will be bounded in terms of these quantities.

\begin{theorem}
\label{res:bound_on_bq_var}
    Suppose Assumptions~\ref{as:app_true_f_smoothness},~\ref{as:app_domains_x},~\ref{as:app_densities},~\ref{as:app_kernel_x}, and~\ref{as:app_regulariser_x} hold. Then there is a $N_0>0$ such that for all $N \geq N_0$ it holds that
    \begin{align*}
        \E_{y_{1:N}^\theta \sim \Pb_\theta} \sigma^2_\mathrm{BQ}(\theta; y_{1:N}^\theta) \leq  \lambda_\calX + C_2 N^{-1 + 2\varepsilon} \left( N^{-\frac{s_\calX}{d}+\frac{1}{2} + \varepsilon} + C_3 \lambda_\calX \right)^2
    \end{align*}
    for any $\theta \in \Theta$, any arbitrarily small $\varepsilon>0$, and $C_2, C_3$ independent of $\theta,N,\varepsilon, \lambda_\calX$.
\end{theorem}

The term $N_0$ quantifies how likely the points $y_{1:N}^\theta$ are to ``fill out'' the space $\calX$---for any $\theta$. Intuitively speaking, $N_0$ is smallest when for all $\theta$, the $\Pb_\theta$ is uniform. 

\begin{proof}
Recall
\begin{align*}
    I_\mathrm{BQ}(\theta; y_{1:N}^\theta) & = \mu_\theta(y_{1:N}^\theta)^\top \left(k_{\calX}(y_{1:N}^\theta, y_{1:N}^\theta)+ \lambda_\calX \Id_N\right)^{-1} f(y_{1:N}^\theta, \theta),\\
    \sigma^2_\mathrm{BQ}(\theta; y_{1:N}^\theta) &= \mathbb{E}_{X,X'\sim \mathbb{P}_\theta}[k_{\calX}(X,X')] - \mu_\theta(y_{1:N}^\theta)^\top \left(k_{\calX}(y_{1:N}^\theta, y_{1:N}^\theta)+ \lambda_\calX \Id_N\right)^{-1} \mu_\theta(y_{1:N}^\theta).
\end{align*}
We seek to bound $\sigma^2_\mathrm{BQ}(\theta; y_{1:N}^\theta)$.~\citep[Proposition 3.8]{kanagawa2018gaussian} pointed out that the Gaussian noise posterior is the worst-case error in the $\calH_{\calX}^{\lambda_\calX}$, the RKHS induced by the kernel $k_\calX^{\lambda_\calX}(x, x') = k_\calX(x, x') + \lambda_\calX \delta(x, x')$ (where $\delta(x, x') = 1$ if $x=x'$, and $0$ otherwise). 
Through straightforward algebraic manipulations and using the reproducing property, one can show that for the vector $w_\theta = k(x, y_{1:N}^\theta)^\top \left(k_{\calX}(y_{1:N}^\theta, y_{1:N}^\theta)+ \lambda_\calX \Id_N\right)^{-1} \in \R^N$, 
\begin{align}
\label{eq:variance_bound_proof_1}
    \sigma^2_\mathrm{BQ}(\theta; y_{1:N}^\theta) - \lambda_\calX =\sup_{\|f\|_{\calH_{\calX}^{\lambda_\calX}} \leq 1} \left| w_\theta f(y_{1:N}^\theta) - \int_\calX f(x) \Pb_\theta(\mathrm d x)\right|^2,
\end{align}
Since $\calH_\calX^{\lambda_\calX}$ is induced by the sum of kernels, $k_\calX^{\lambda_\calX}(x, x') = k_\calX(x, x') + \lambda_\calX$, it holds that $\calH_\calX \subseteq \calH_\calX^{\lambda_\calX}$, and $\| f  \|_{\calH_\calX^{\lambda_\calX}} \leq \| f  \|_{\calH_\calX}$~\citep[Theorem I.13.IV]{aronszajn1950theory}. Therefore, the class of functions $f$ for which $\| f  \|_{\calH_\calX} \leq 1$ is larger than that for which $\| f  \|_{\calH_\calX^{\lambda_\calX}} \leq 1$, and
\begin{align}
\label{eq:variance_bound_proof_2}   \sup_{\|f\|_{\calH_{\calX}^{\lambda_\calX}} \leq 1} \left| w_\theta f(y_{1:N}^\theta) - \int_\calX f(x) \Pb_\theta(\mathrm d x)\right| \leq \sup_{\|f\|_{\calH_{\calX}} \leq 1} \left| w_\theta f(y_{1:N}^\theta) - \int_\calX f(x) \Pb_\theta(\mathrm d x)\right|.
\end{align}
Next, note that for $\hat{f}_\theta(x) = k(x, y_{1:N}^\theta)^\top \left(k_{\calX}(y_{1:N}^\theta, y_{1:N}^\theta)+ \lambda_\calX \Id_N\right)^{-1} f(y_{1:N}^\theta)$,
\begin{align}
\begin{split}
\label{eq:variance_bound_proof_3}
   \left| w_\theta f(y_{1:N}^\theta) - \int_\calX f(x) \Pb_\theta(\mathrm d x)\right| = \left| \int_\calX \left(\hat{f}_\theta(x) - f(x) \right) \Pb_\theta(\mathrm d x)\right| &\leq  \int_\calX \left|\hat{f}_\theta(x) - f(x) \right| \Pb_\theta(\mathrm d x)  \\
    & \leq \|\hat{f}_\theta - f\|_{\calL^2(\calX)} \|p_\theta \|_{\calL^2(\calX)},
\end{split}
\end{align}
where the last inequality is an application of H\"older inequality. By Assumption~\ref{as:app_densities}$, \|p_\theta \|_{\calL^2(\calX)}$ is bounded above by $\eta_0$. In order to apply~\citep[Theorem 4]{wynne2021convergence} to bound $\|\hat{f}_\theta - f\|_{\calL^2(\calX)}$, we show the assumptions of that Theorem hold.
    
{Assumption 1 (Assumptions on the Domain):} An open, bounded, and convex $\calX$ satisfies the assumption, as discussed in~\cite{wynne2021convergence}.
    
{Assumption 2 (Assumptions on the Kernel Parameters) and Assumption 3 (Assumptions on the Kernel Smoothness Range):} Our setting is more specific than the one~\citep[Theorem 4]{wynne2021convergence}: the kernel $k_\calX$ is Mat\'ern, and therefore all smoothness constants mentioned in Assumptions 2 and 3 have the same value, $s_\calX$.

{Assumption 4 (Assumptions on the Target Function and Mean Function):} The target function $f$ was assumed to have higher smoothness than $k_\calX$ in~\ref{as:app_true_f_smoothness}, and~\ref{as:app_kernel_x}; the mean function was taken to be zero.

{Assumption 5 (Additional Assumptions on Kernel Parameters):} By \ref{as:app_kernel_x} and~\ref{as:app_true_f_smoothness} the smoothness of the true function $s_f \geq s_\calX >d/2$, which verifies both statements in the Assumption since all smoothness constants of the kernel are equal to $s_\calX$.

Therefore~\citep[Theorem 4]{wynne2021convergence} holds, and for $\calW_{0,2}(\calX)=\calL^2(\calX)$
\begin{align*}
    \|\hat{f}_\theta - f\|_{\calL^2(\calX)} \leq K_3 \| f \|_{\calH_\calX} h_{y_{1:N}^\theta}^{\frac{d}{2}} \left( h_{y_{1:N}^\theta}^{s_\calX-\frac{d}{2}} + \lambda_\calX \right),
\end{align*}
for any $N$ for which the fill distance $h_{y_{1:N}^\theta} \leq h_0$ for some $h_0$, and $K_3$ and $h_0$ that depend on $\calX, s_f, s_\calX$.\footnote{Note that the result in~\citep[Theorem 4]{wynne2021convergence} features $\| f \|_{\calW^{s_\calX,2}(\calX)}$, not $\| f \|_{\calH_\calX}$. The bound in terms $\| f \|_{\calH_\calX}$ holds since $\calH_\calX$ was assumed to be a Sobolev RKHS.}

For $y_{1:N}^\theta \sim \Pb_\theta$, we can guarantee that $h_{y_{1:N}^\theta} \leq h_0$ in expectation using~\citep[Lemma 2]{oates2019convergence}, which says that provided the density $\inf_{x} p_\theta(x)>0$, there is a $C_\theta$ such that $\E h_{y_{1:N}^\theta} \leq C_\theta N^{-1/d + \varepsilon}$ for an arbitrarily small $\varepsilon>0$, for $C_\theta$ that depends on $\theta$ through $\inf_{x} p_\theta(x)$. The smaller $\inf_{x} p_\theta(x)$, the larger $C_\theta$. Since we assumed $\inf_{x, \theta} p_\theta(x)=\eta>0$ there is a $K_4$ such that $C_\theta \leq K_4$ for any $\theta$. Therefore, we may take $N_0$ to be the smallest $N$ for which $\E h_{y_{1:N}^\theta} \leq K_4 N^{-1/d + \varepsilon}$ holds, and have for all $N \geq N_0$
\begin{equation}
\label{eq:wynne_bound_on_exp}
    \E_{y_{1:N}^\theta \sim \Pb_\theta}\|\hat{f}_\theta - f\|_{\calL^2(\calX)} \leq K_3 K_4^{\frac{d}{2}} \| f \|_{\calH_\calX} N^{-\frac{1}{2} + \varepsilon} \left( K_4^{s_\calX-\frac{d}{2}} N^{-\frac{s_\calX}{d} + \frac{1}{2} + \varepsilon} + \lambda_\calX \right)
\end{equation}
Putting together~\Cref{eq:variance_bound_proof_1,eq:variance_bound_proof_2,eq:variance_bound_proof_3,eq:wynne_bound_on_exp} and Assumption~\ref{as:app_densities}, we get the result,
\begin{align*}
    \E_{y_{1:N}^\theta \sim \Pb_\theta} \sigma^2_\mathrm{BQ}(\theta; y_{1:N}^\theta) - \lambda_\calX 
    &=
    \sup_{\|f\|_{\calH_{\calX}^{\lambda_\calX}} \leq 1} \E_{y_{1:N}^\theta \sim \Pb_\theta} \left| w_\theta f(y_{1:N}^\theta) - \int_\calX f(x) \Pb_\theta(\mathrm d x)\right|^2 \\
    &\leq 
    \sup_{\|f\|_{\calH_{\calX}} \leq 1} \E_{y_{1:N}^\theta \sim \Pb_\theta} \left| w_\theta f(y_{1:N}^\theta) - \int_\calX f(x) \Pb_\theta(\mathrm d x)\right|^2 \\
    &\leq 
    \sup_{\|f\|_{\calH_{\calX}} \leq 1} \E_{y_{1:N}^\theta \sim \Pb_\theta} \|\hat{f}_\theta - f\|^2_{\calL^2(\calX)} \|p_\theta \|^2_{\calL^2(\calX)} \\
    &\leq 
    \eta^2_0 K^2_3 K_4^d N^{-1 + 2\varepsilon} \left( K_4^{s_\calX-\frac{d}{2}} N^{-\frac{s_\calX}{d} + \frac{1}{2} + \varepsilon} + \lambda_\calX \right)^2 \\
    &\eqqcolon C_2 N^{-1 + 2\varepsilon} \left( N^{-\frac{s_\calX}{d}+\frac{1}{2} + \varepsilon} + C_3 \lambda_\calX \right)^2.
\end{align*}
\end{proof}
Before bounding the error $\|I_\mathrm{BQ} - I\|_{\calH_\Theta}$, we give the following general auxiliary result for an arbitrary Sobolev space of function over some open $\Omega \subseteq \R^d$.

\begin{prop}
\label{res:sobolev_product_bound}
    Suppose $f, g$ lie in a Sobolev space $\calW^{s,2}(\Omega)$ for some of smoothness $s$, and for all $|\alpha|\leq s$ the weak derivative $D^\alpha g$ is bounded. Take $M = \max_{|\alpha|\leq s} \|D^\alpha g\|_{\calL^\infty(\Omega)}$. Then, there is a constant $K$ such that
    \begin{equation*}
        \|fg\|_{\calW^{s,2}(\Omega)} \leq KM \|f\|_{\calW^{s,2}(\Omega)}.
    \end{equation*}
\end{prop}
\begin{proof}
    Recall that the norm in a Sobolev space is defined as
    \begin{equation}
    \label{eq:sobolev_norm_of_product_defn}
        \|fg\|^2_{\calW^{s,2}(\Omega)} = \sum_{|\alpha|\leq s} \|D^\alpha[fg]\|^2_{\calL^2(\Omega)}.
    \end{equation}
    Fix some $\alpha$ such that $|\alpha|\leq s$. By the product rule to weak derivatives (see, for instance,~\citet[Section 4.2.2]{evans2018measure}), it holds that
    \begin{equation*}
        D^{\alpha} [fg] = \sum_{\substack{|\alpha'|\leq|\alpha|}} \sum_{\substack{|\alpha''|\leq|\alpha|}} C_{\alpha',\alpha'',\alpha} D^{\alpha'} [f] D^{\alpha''} [g],
    \end{equation*}
    for all $\alpha', \alpha''$ being multi-indices of the same dimension as $\alpha$, and some real constants $C_{\alpha',\alpha'',\alpha}>0$ that only depend on $\alpha$ and not $f$ or $g$. Then
    \begin{align*}
        \|D^{\alpha}[fg]\|^2_{\calL^2(\Omega)} 
        &= \left\|\sum_{\substack{|\alpha'|\leq|\alpha|}} \sum_{\substack{|\alpha''|\leq|\alpha|}} C_{\alpha',\alpha'',\alpha} D^{\alpha'} [f] D^{\alpha''} [g]\right\|^2_{\calL^2(\Omega)} \\
        &\stackrel{(A)}{\leq} 
        \left(\sum_{\substack{|\alpha'|\leq|\alpha|}} \sum_{\substack{|\alpha''|\leq|\alpha|}} C_{\alpha',\alpha'',\alpha} \| D^{\alpha'} [f] D^{\alpha''} [g]\|_{\calL^2(\Omega)} \right)^2 \\
        &\stackrel{(B)}{\leq} 
        2 \begin{pmatrix} d\\|\alpha| \end{pmatrix} \sum_{\substack{|\alpha'|\leq|\alpha|}} \sum_{\substack{|\alpha''|\leq|\alpha|}} C_{\alpha',\alpha'',\alpha} 
        \| D^{\alpha'} [f] D^{\alpha''} [g]\|^2_{\calL^2(\Omega)} \\
        &\stackrel{(C)}{\leq} 
        2M^2 \begin{pmatrix} d\\|\alpha| \end{pmatrix} \sum_{\substack{|\alpha'|\leq|\alpha|}} \sum_{\substack{|\alpha''|\leq|\alpha|}} C_{\alpha',\alpha'',\alpha} 
        \| D^{\alpha'} [f] \|^2_{\calL^2(\Omega)} \\
        &\leq 2M^2 \begin{pmatrix} d\\|\alpha| \end{pmatrix} \sum_{\substack{|\alpha'|\leq|\alpha|}} \sum_{\substack{|\alpha''|\leq|\alpha|}} C_{\alpha',\alpha'',\alpha} 
        \| f \|^2_{\calW^{s,2}(\Omega)},
    \end{align*}
    where $(A)$ holds by triangle inequality, $(B)$ holds as, by Cauchy-Schwartz, $(\sum_{i=1}^n a_i)^2 \leq n \sum_{i=1}^n a_i^2$ for any real $a_i$, and as the number of multi-indices in $\mathbb N^d$ of size at most $\alpha$ is ``$d$ choose $|\alpha|$'', and $(C)$ by the definition $M = \max_{|\alpha|\leq s} \|D^\alpha g\|_{\calL^\infty}(\Omega)$. Substituting this into~\Cref{eq:sobolev_norm_of_product_defn}, we get that for $\sqrt{K} = 2\sum_{|\alpha| < s} \begin{pmatrix} d\\|\alpha| \end{pmatrix} \sum_{\substack{|\alpha'|\leq|\alpha|}} \sum_{\substack{|\alpha''|\leq|\alpha|}} C_{\alpha',\alpha'',\alpha}$,
    \begin{equation*}
        \|fg\|^2_{\calW^{s,2}(\Omega)} \leq K^2 M^2 \| f \|^2_{\calW^{s,2}(\Omega)}.
    \end{equation*}
\end{proof}
With the Sobolev norm bound in place, we are ready to give the bound on $\|I_\mathrm{BQ} - I\|_{\calH_\Theta}$.
\begin{theorem}
\label{res:bound_on_bq_error}
    Suppose Assumptions~\ref{as:app_true_f_smoothness},~\ref{as:app_true_I_norm_bounds},~\ref{as:app_domains_x},~\ref{as:app_x_samples},~\ref{as:app_densities},~\ref{as:app_densities_linf},~\ref{as:app_kernel_x},~\ref{as:app_kernel_theta} and~\ref{as:app_regulariser_x} hold. Then there is a $N_0>0$ such that for all $N \geq N_0$ with probability at least $1-\delta/2$ it holds that
    \begin{equation*}
        \|I_\mathrm{BQ} - I\|_{\calH_\Theta} \leq \frac{2}{\delta} C_4 N^{-\frac{1}{2} + \varepsilon} \left( N^{-\frac{s_\calX}{d} + \frac{1}{2} + \varepsilon} + C_5 \lambda_\calX \right).
    \end{equation*}
    for any arbitrarily small $\varepsilon>0$, and $C_4, C_5$ independent of $N,\varepsilon, \lambda_\calX$.
\end{theorem}
\begin{proof}
    Recall that, as $\calH_\Theta$ is a Sobolev RKHS (meaning $k_\Theta$ is a Sobolev kernel) of smoothness $s_\Theta$, it holds that $C'_1 \|g\|_{\calW^{s_\Theta, 2}(\Theta)} \leq \|g\|_{\calH_\Theta} \leq C'_2 \|g\|_{\calW^{s_\Theta, 2}(\Theta)}$ for some constants $C'_1,C'_2>0$ and any $g \in \calH_\Theta$.
    Take $\hat{f}(x, \theta) = k(x, x_{1:N}^\theta)^\top \left(k_{\calX}(x_{1:N}^\theta, x_{1:N}^\theta)+ \lambda_\calX \Id_N\right)^{-1} f(x_{1:N}^\theta, \theta)$. Then,
    \begin{align*}
        \|I_\mathrm{BQ} - I\|^2_{\calH_\Theta} 
        &= 
        \langle I_\mathrm{BQ} - I, I_\mathrm{BQ} - I\rangle_{\calH_\Theta}  \\
        &= 
        \left\langle \int_\calX \left( \hat f(x, \theta) - f(x, \theta) \right) p_\theta(x)  \mathrm{d} x, \int_\calX \left( \hat f(x', \theta) - f(x', \theta) \right) p_\theta(x')  \mathrm{d} x' \right\rangle_{\calH_\Theta}  \\
        &\leq \int_\calX \int_\calX \left\langle   \left( \hat f(x, \theta) - f(x, \theta) \right) p_\theta(x) , \left( \hat f(x', \theta) - f(x', \theta) \right) p_\theta(x')  \right\rangle_{\calH_\Theta} \mathrm{d} x \mathrm{d} x' \\
        &\stackrel{(A)}{\leq} 
        \left(\int_\calX \left\| \left( \hat f(x, \theta) - f(x, \theta) \right) p_\theta(x)  \right\|_{\calH_\Theta} \mathrm{d} x \right)^2 \\
        &\stackrel{(B)}{\leq} 
        {C'_2}^2 K^2 M_p^2 \left(\int_\calX \left\| \hat f(x, \theta) - f(x, \theta) \right\|_{\calW^{s_\Theta, 2}(\Theta)} \mathrm{d} x \right)^2,
    \end{align*}
    where $(A)$ holds by the Cauchy-Schwarz, $(B)$ by~\Cref{res:sobolev_product_bound} and $\calH_\Theta$ being a Sobolev RKHS. 
    As for the remaining term,
    \begin{align*}
        \int_\calX \left\| \hat f(x, \theta) - f(x, \theta) \right\|^2_{\calW^{s_\Theta, 2}(\Theta)} \mathrm d x 
        &= 
        \sum_{|\alpha|\leq s_\Theta} \int_\calX \int_\Theta \left( D_\theta^\alpha \hat f(x, \theta) - D_\theta^\alpha f(x, \theta) \right)^2 \mathrm d \theta \mathrm d x \\
        &= 
        \sum_{|\alpha|\leq s_\Theta} \int_\Theta \int_\calX \left( D_\theta^\alpha \hat f(x, \theta) - D_\theta^\alpha f(x, \theta) \right)^2 \mathrm d x \mathrm d \theta \\
        &= 
        \sum_{|\alpha|\leq s_\Theta} \int_\Theta \left\| D_\theta^\alpha \hat f(x, \theta) - D_\theta^\alpha f(x, \theta) \right\|^2_{\calL^2(\calX)} \mathrm d \theta
    \end{align*}
    Since $D_\theta^\alpha \hat{f}(x, \theta) = k(x, x_{1:N}^\theta)^\top \left(k_{\calX}(x_{1:N}^\theta, x_{1:N}^\theta)+ \lambda_\calX \Id_N\right)^{-1} D_\theta^\alpha f(x_{1:N}^\theta, \theta)$, and the $\calX$-smoothness of $D_\theta^\alpha f$ is the same as that of $f$, we may use~\citet[Theorem 4]{wynne2021convergence} to bound $\|D_\theta^\alpha \hat f(x, \theta) - D_\theta^\alpha f(x, \theta)\|_{\calL^2(\calX)}$ identically to the proof of~\Cref{res:bound_on_bq_var}. Then, we have that
    \begin{align*}
        \E_{x_{1:N}^\theta \sim \Pb_\theta}\|D_\theta^\alpha \hat f(x, \theta) - D_\theta^\alpha f(x, \theta)\|_{\calL^2(\calX)} 
        &\leq 
        K_3 K_4^{\frac{d}{2}} \| D_\theta^\alpha f \|_{\calH_\calX} N^{-\frac{1}{2} + \varepsilon} \left( K_4^{s_\calX-\frac{d}{2}} N^{-\frac{s_\calX}{d} + \frac{1}{2} + \varepsilon} + \lambda_\calX \right) \\
        &\stackrel{(A)}{\leq} 
        K_3 K_4^{\frac{d}{2}} C'_2 M_f N^{-\frac{1}{2} + \varepsilon} \left( K_4^{s_\calX-\frac{d}{2}} N^{-\frac{s_\calX}{d} + \frac{1}{2} + \varepsilon} + \lambda_\calX \right),
    \end{align*}
    where $(A)$ holds by Assumption~\ref{as:app_true_I_norm_bounds}, and $k_\calX$ being a Sobolev kernel and $C'_2$ being a norm equivalence constant. Define 
    By Markov's inequality, for any $\delta/2 \in (0,1)$ it holds with probability at least $1-\delta/2$ that 
    \begin{align*}
        \|D_\theta^\alpha \hat f(x, \theta) - D_\theta^\alpha f(x, \theta)\|_{\calL^2(\calX)} \leq \frac{2}{\delta}K_3 K_4^{\frac{d}{2}} C'_2 M_f N^{-\frac{1}{2} + \varepsilon} \left( K_4^{s_\calX-\frac{d}{2}} N^{-\frac{s_\calX}{d} + \frac{1}{2} + \varepsilon} + \lambda_\calX \right)
    \end{align*}
    Lastly, the number of $\alpha$ such that $|\alpha| < s_\Theta$ is the combination ``$p$ select $s_\Theta$''. Then,
    \begin{align*}
        \|I_\mathrm{BQ} - I\|^2_{\calH_\Theta}  
        &\leq  
        {C'_2}^2 K^2 M_p^2 \begin{pmatrix} p \\ s_\Theta \end{pmatrix} \left( \frac{2}{\delta}K_3 K_4^{\frac{d}{2}} C'_2 M_f N^{-\frac{1}{2} + \varepsilon} \left( K_4^{s_\calX-\frac{d}{2}} N^{-\frac{s_\calX}{d} + \frac{1}{2} + \varepsilon} + \lambda_\calX \right) \right)^2 \\
        &\eqqcolon
        \frac{4}{\delta^2} \sqrt{C_4} N^{-1 + 2\varepsilon} \left( N^{-\frac{s_\calX}{d} + \frac{1}{2} + \varepsilon} + C_5 \lambda_\calX \right) ^2.
    \end{align*}
\end{proof}

\subsection{Stage 2 bounds}
\label{sec:stage2}

In this section, we establish convergence of the estimator $I_\mathrm{CBQ}$ to the true function $I$ in the norm $\calL^2(\Theta, \Qb)$, first in terms of the error $\|I_\mathrm{BQ}(\cdot; x^\theta_{1:N}) - I(\cdot)\|_{\calH_\Theta}$ in~\Cref{res:convergence_of_iwkrr}, and additionally in the variance $\sigma^2_\mathrm{BQ}(\theta; x^\theta_{1:N})$ in~\Cref{res:unweighted_bound_stage_2}. To do so, we represent the CBQ estimator as
\begin{equation}
\label{eq:cbq_noisy_weights}
    I_\mathrm{CBQ}(\theta) = k_\Theta(\theta, \theta_{1:M}) \left( k_\Theta(\theta_{1:M}, \theta_{1:M}) + \diag \left[ \frac{M \lambda}{w(\theta_{1:M}) + \varepsilon(\theta_{1:M}; x^{1:M}_{1:N})} \right] \right)^{-1} I_\mathrm{BQ}(\theta_{1:M}; x^{1:M}_{1:N}).
\end{equation}
for vector notation $\varepsilon(\theta_{1:M}; x^{1:M}_{1:N}) = [\varepsilon(\theta_1; x^1_{1:N}), \ldots,  \varepsilon(\theta_M; x^M_{1:N}) ]^\top \in \R^M$, and $\lambda$, the weight $w: \Theta \to \R$ and the noise term $\varepsilon: \Theta \to \R$ given by
\begin{equation}
\begin{split}
\label{eq:weights}
    &\lambda = \lambda_\Theta M^{-1} \\
    &w(\theta) = \E_{y^\theta_{1:N} \sim \Pb_\theta}\frac{\lambda_\Theta}{\lambda_\Theta + \sigma^2_\mathrm{BQ}(\theta; y^\theta_{1:N})},\\
    &\varepsilon(\theta; x^\theta_{1:N}) = \frac{\lambda_\Theta}{\lambda_\Theta + \sigma^2_\mathrm{BQ}(\theta; x^\theta_{1:N})} - \E_{y^\theta_{1:N} \sim \Pb_\theta} \frac{\lambda_\Theta}{\lambda_\Theta + \sigma^2_\mathrm{BQ}(\theta; y^\theta_{1:N})}.
\end{split}
\end{equation}
The equality to the CBQ estimator given in the main text can be easily seen, as the term under the $\mathrm{diag}$ is
\begin{equation*}
    \frac{M \lambda}{w(\theta_{1:M}) + \varepsilon(\theta_{1:M}; x^{1:M}_{1:N})} = \frac{M \lambda_\Theta M^{-1}}{\frac{\lambda_\Theta}{\lambda_\Theta + \sigma^2_\mathrm{BQ}(\theta; x^\theta_{1:N})}} = \lambda_\Theta + \sigma^2_\mathrm{BQ}(\theta; x^\theta_{1:N}).
\end{equation*}
If the noise term in~\Cref{eq:cbq_noisy_weights} were absent (meaning, equal to zero), the estimator would become the \emph{importance-weighted kernel ridge regression} (IW-KRR) estimator. The convergence of the IW-KRR estimator was studied in~\citet[Theorem 4]{gogolashvili2023importance}. In this section, we extend their results to the case of noisy weights ($\varepsilon \not\equiv 0$), which are additionally correlated with the noise in $I_\mathrm{BQ}(\theta_i; x^i_{1:N})$ (through the shared datapoints $x^i_{1:N}$). 

Note that, while we only provide results specific for $I_\mathrm{CBQ}$, the proof can be extended with minor modifications to the more general case of arbitrary noisy IW-KRR with weights that satisfy conditions in~\citet{gogolashvili2023importance}, and zero-mean weight noise.

The convergence results for the \emph{noisy importance-weighted kernel ridge regression} estimator in~\Cref{sec:convergence_of_niwkrr} will rely on a representation of $I_\mathrm{CBQ}$ in terms of a sample-level version of a certain weighted integral operator. Then, we bound the gap between $I_\mathrm{CBQ}$ and $I$ in terms of (1) the gap between the sample-level version of said operator, and the population-level version, and (2) the gap between $I_\mathrm{BQ}$ and $I$. Next, we define said operator, and additional notation used in the proofs.

\subsubsection{Notation}
\label{sec:notation}
We will be working on positive, bounded, self-adjoint $\calH_\Theta \to \calH_\Theta$ operators 
\begin{equation}
    T[g](\theta) = \int_\Theta k_\Theta(\theta, \theta') g(\theta') w(\theta') \Qb(\mathrm d \theta') \qquad 
    \hat{T}[g](\theta) = \frac{1}{M} k_\Theta(\theta, \theta_{1:M}) \diag\left[w(\theta_{1:M}) + \varepsilon(\theta_{1:M}; x^{1:M}_{1:N}) \right] g(\theta_{1:M}).
\end{equation}
for the weight function $w$ and noise term $\varepsilon$ as defined in~\Cref{eq:weights}.
We will denote $\mathrm{HS}$ to be the Hilbert space of Hilbert-Schmidt operators $\calH_\Theta \to \calH_\Theta$, $\|\cdot\|_\text{HS}$ to be the Hilbert-Schmidt norm, and $\|\cdot\|_\mathrm{op}$ to be the operator norm. As is customary, we will write $T+\lambda$ to mean the operator $T+\lambda \Id_{\calH_\Theta}$, where $\Id_{\calH_\Theta}$ is the identity operator $\calH_\Theta \to \calH_\Theta$.

\subsubsection{Auxiliary results}
\label{sec:auxilliary_results_stage_2}

The results given in this section are key to proving the main Stage 2 result,~\Cref{res:convergence_of_iwkrr}. The following result bounds the Hilbert-Schmidt norm on the ``gap'' between the population-level $T$ and the sample-level $\hat T$, when their difference is ``sandwiched'' between $(T+\lambda)^{-1/2}$. With some manipulation, this term will appear in the proof of~\Cref{res:convergence_of_iwkrr}.
\begin{lemma}[Modified Lemma 18 in~\citet{gogolashvili2023importance}]
\label{res:s1}
Suppose Assumptions~\ref{as:app_x_samples},~\ref{as:app_kernel_theta_bounded} hold, and the operators $T, \hat{T}$ be as defined in \Cref{sec:notation}. Then, with probability greater than $1 - \delta/2$,
    \begin{equation*}
        S_1 \coloneqq \|(T+\lambda)^{-1/2} (T-\hat{T}) (T+\lambda)^{-1/2} \|_{\mathrm{HS}} \leq \frac{4\kappa}{\lambda \sqrt M} \log(4/\delta).
    \end{equation*}
Additionally, if $\lambda \sqrt{M} > (4/C_6) \kappa \log(4/\delta)$ for some $C_6\leq1$, it holds that $S_1<C_6 \leq 1$.
\end{lemma}
The fact that $S_1$ is strictly less than $1$ will be important in the proof of the main Stage 2 result,~\Cref{res:convergence_of_iwkrr}, as it will allow us to apply Neumann series expansion to $\|(\Id - (T+\lambda)^{-1/2} (T-\hat{T}) (T+\lambda)^{-1/2})^{-1}\|_\mathrm{op}$.

\begin{proof}
    Denote a feature function $\varphi_\theta (\cdot) \coloneqq k_\Theta(\theta, \cdot)$. Let $\xi, \xi_1, \dots, \xi_M$ be random variables in $\mathrm{HS}$ defined as
    \begin{align*}
        \xi &= (T+\lambda)^{-1/2} (w(\theta)+\varepsilon(\theta; x^\theta_{1:N})) \varphi_\theta \langle \varphi_\theta, \cdot \rangle_{\calH_\Theta} (T+\lambda)^{-1/2} \\
        \xi_i &= (T+\lambda)^{-1/2} (w(\theta_i) + \varepsilon(\theta_i; x^i_{1:N})) \varphi_{\theta_i} \langle \varphi_{\theta_i}, \cdot \rangle_{\calH_\Theta} (T+\lambda)^{-1/2}
    \end{align*}
    
    First, note that as $(\theta, x_{1:N}), (\theta_1, x^{1}_{1:N}),\dots,(\theta_N, x^{N}_{1:N})$ are i.i.d., it follows that $\xi, \xi_1, \dots, \xi_N$ are i.i.d. random variables in $\mathrm{HS}$. Further, as $\E_{x^\theta_{1:N} \sim \Pb_\theta} \varepsilon(\theta; x^\theta_{1:N}) =0$, it holds that
    \begin{equation*}
        \E_{\theta \sim \Qb} \E_{x^\theta_{1:N} \sim \Pb_\theta} \xi = \E_{\theta \sim \Qb} \left[(T+\lambda)^{-1/2} w(\theta) \varphi_\theta \langle \varphi_\theta, \cdot \rangle_{\calH_\Theta} (T+\lambda)^{-1/2}\right] =(T+\lambda)^{-1/2} T (T+\lambda)^{-1/2},
    \end{equation*}
    where the last equality, $\E_{\theta \sim \Qb} w(\theta) \varphi_\theta \langle \varphi_\theta, \cdot \rangle_{\calH_\Theta} = T$, holds since  $\E_{\theta \sim \Qb} \left[ w(\theta) \varphi_\theta \langle \varphi_\theta, \cdot \rangle_{\calH_\Theta} \right]g(\theta') = \E_{\theta \sim \Qb}  w(\theta) k(\theta, \theta') g(\theta)$ for any $g \in \calH_\Theta$. Therefore 
    \begin{align*}
        S_1 = \left\| \frac{1}{M} \sum_{i=1}^M \xi_i - \E_{\theta \sim \Qb} \E_{x^\theta_{1:N} \sim \Pb_\theta} \xi \right\|_{\mathrm{HS}}.
    \end{align*} 
    Then, by the Bernstein inequality for Hilbert space-valued random variables~\citep[Proposition 2]{caponnetto2007optimal}, the claimed bound on $S_1$ holds if there exist $L>0, \sigma>0$ such that
    \begin{equation*}
        \E_{\theta \sim \Qb} \E_{x^\theta_{1:N} \sim \Pb_\theta} \left[\| \xi - \E_{\theta \sim \Qb} \E_{x^\theta_{1:N} \sim \Pb_\theta} \xi  \|^m_{\mathrm{HS}}\right] \leq \frac{1}{2} m! \sigma^2 L^{m-2}
    \end{equation*}
    holds for all integer $m \geq 2$. We will show the condition holds.
    For convenience, denote $\E_\xi f(\xi) \coloneqq \E_{\theta \sim \Qb} \E_{x^\theta_{1:N} \sim \Pb_\theta} f(\xi)$. 
    First, suppose $\xi'$ is an independent copy of $\xi$. Identically to the proof of~\citet[Lemma 18]{gogolashvili2023importance}, it holds that
    \begin{equation*}
        \E_\xi\left[\| \xi - \E_\xi \xi  \|^m_{\mathrm{HS}}\right] \stackrel{(A)}{\leq} \E_\xi \E_{\xi'} \left[\| \xi - \xi'  \|^m_{\mathrm{HS}}\right] \stackrel{(B)}{\leq} 2^{m-1} \E_\xi \E_{\xi'}\left[\| \xi \|^m_{\mathrm{HS}} +  \|\xi'  \|^m_{\mathrm{HS}}\right] = 2^m \E_\xi \| \xi \|^m_{\mathrm{HS}}
    \end{equation*}
    where $(A)$ holds by Jensen inequality, and $(B)$ uses the fact that $|a+b|^m \leq 2^{m-1}(|a|^m + |b|^m)$. Next, observe that
    \begin{align*}
        \E_\xi \| \xi \|^m_{\mathrm{HS}} 
        &= 
        \E_{\theta \sim \Qb} \E_{x^\theta_{1:N} \sim \Pb_\theta} \|(T+\lambda)^{-1/2} (w(\theta)+\varepsilon(\theta; x^\theta_{1:N})) \varphi_\theta \langle \varphi_\theta, \cdot \rangle_{\calH_\Theta} (T+\lambda)^{-1/2} \|^m_{\mathrm{HS}} \\
        &\stackrel{(A)}{=} 
        \E_{\theta \sim \Qb} \left[ \E_{x^\theta_{1:N} \sim \Pb_\theta} \left[ (w(\theta)+\varepsilon(\theta; x^\theta_{1:N}))^m \right]  \|(T+\lambda)^{-1/2}  \varphi_\theta \langle \varphi_\theta, \cdot \rangle_{\calH_\Theta} (T+\lambda)^{-1/2} \|^m_{\mathrm{HS}}  \right] \\
        & \stackrel{(B))}{\leq}
        \E_{\theta \sim \Qb} \left[  \|(T+\lambda)^{-1/2}  \varphi_\theta \langle \varphi_\theta, \cdot \rangle_{\calH_\Theta} (T+\lambda)^{-1/2} \|^m_{\mathrm{HS}}  \right] \\
        & \stackrel{(C)}{\leq} 
        \kappa^m \lambda^{-m} \\
        &= 
        \frac{1}{2}m! \sigma^2 L^{m-2}
    \end{align*}
    where $L=\sigma=\kappa \lambda^{-1}$, $(A)$ holds by linearity of norms as $w(\theta)+\varepsilon(\theta; x^\theta_{1:N}) \in \R$, $(B)$ holds since $\sigma^2_\mathrm{BQ}(\theta; x^\theta_{1:N}) \geq 0$, so
    \begin{equation*}
        \left(w(\theta) + \varepsilon(\theta; x^\theta_{1:N})\right)^m = \left(\frac{\lambda_\Theta}{\lambda_\Theta + \sigma^2_\mathrm{BQ}(\theta; x^\theta_{1:N})}\right)^m \leq 1.
    \end{equation*}
    To show $(C)$ holds, take $\{e_j\}_{j=1}^\infty$ to be some orthonormal basis of $\calH_\Theta$. Then,
    \begin{align*}
        \|(T+\lambda)^{-1/2}  \varphi_\theta \langle \varphi_\theta, \cdot \rangle_{\calH_\Theta} (T+\lambda)^{-1/2} \|^2_{\mathrm{HS}} 
        &= 
        \sum_{j=1}^\infty \|(T+\lambda)^{-1/2}  \varphi_\theta \langle \varphi_\theta, \cdot \rangle_{\calH_\Theta} (T+\lambda)^{-1/2} e_j \|^2_{\calH_\Theta} \\
        &= 
        \sum_{j=1}^\infty \|(T+\lambda)^{-1/2}  \varphi_\theta \langle \varphi_\theta, (T+\lambda)^{-1/2} e_j \rangle_{\calH_\Theta}  \|^2_{\calH_\Theta} \\
        &\leq 
        \|(T+\lambda)^{-1/2}  \varphi_\theta  \|^2_{\calH_\Theta} \sum_{j=1}^\infty \langle \varphi_\theta, (T+\lambda)^{-1/2} e_j \rangle^2_{\calH_\Theta}  \\
        &= 
        \|(T+\lambda)^{-1/2}  \varphi_\theta  \|^2_{\calH_\Theta} \sum_{j=1}^\infty \langle (T+\lambda)^{-1/2} \varphi_\theta,  e_j \rangle^2_{\calH_\Theta}  \\
        &\stackrel{(A)}{\leq}
        \|(T+\lambda)^{-1/2}  \varphi_\theta  \|^2_{\calH_\Theta} \|(T+\lambda)^{-1/2}  \varphi_\theta  \|^2_{\calH_\Theta}  \\
        &= 
        \langle (T+\lambda)^{-1}  \varphi_\theta  , \varphi_\theta  \rangle^2_{\calH_\Theta}  \\
        &\leq \kappa^2 \lambda^{-2},
    \end{align*}
    where $(A)$ holds by Bessel's inequality. Then by the Bernstein inequality in~\citet[Proposition 2]{caponnetto2007optimal}, it holds that
    \begin{equation*}
        S_1 \leq \frac{2\kappa}{\lambda \sqrt M} \left(\frac{1}{\sqrt M} + 1\right) \log(4/\delta) \leq \frac{4\kappa}{\lambda \sqrt M} \log(4/\delta),
    \end{equation*}
    with probability at least $1 - \delta/2$. Finally as $\lambda \sqrt{M} > (16/3) \kappa \log(4/\delta)$, $S_1  < 3/4$.
\end{proof}

Next, we bound another relevant term that also quantifies the ``gap'' between $T$ and $\hat T$. Unlike $S_1$, we will not require it to be upper bounded by $1$---as it will only appear in~\Cref{res:convergence_of_iwkrr} as a bounding term to the error.

\begin{lemma}
\label{res:s2}
Suppose Assumptions~\ref{as:app_x_samples},~\ref{as:app_kernel_theta_bounded} hold, and the operators $T, \hat{T}$ be as defined in \Cref{sec:notation}. Then, with probability greater than $1 - \delta/2$,
    \begin{equation*}
        S_2 \coloneqq \|(T+\lambda)^{-1/2} (T-\hat{T})\|_{\mathrm{HS}} \leq \frac{4\kappa}{ \sqrt {\lambda M}} \log(4/\delta).
    \end{equation*}
Additionally, if $\lambda \sqrt{M} > (4/C_6) \kappa \log(4/\delta)$, it holds that $S_2<C_6 \sqrt \lambda$.
\end{lemma}
\begin{proof}
    The proof is identical to that of~\Cref{res:s1}.
\end{proof}

The last auxiliary result we need is a simple bound on the following operator norm.
\begin{lemma}
\label{res:t_t_plus_lambda_inv_bounded}
Let $T: \calH_\Theta \to \calH_\Theta$ be a positive operator. Then,
    \begin{equation*}
        \| T  (T+\lambda)^{-1} \|_\mathrm{op} \leq 1.
    \end{equation*}
\end{lemma}
\begin{proof}
    Since $T$ is positive, for any $f \in \calH_\Theta$ it holds that $\|Tf\|_{\calH_\Theta} \leq \|(T+\lambda) f\|_{\calH_\Theta}$. Therefore, by taking $f=(T+\lambda)^{-1} g$, we get that
    \begin{align*}
        \|T (T+\lambda)^{-1} \|_\mathrm{op} = \sup_{\substack{g \in \calH_\Theta \\ \|g\|_{\calH_\Theta}=1}} \|T (T+\lambda)^{-1} g \|_{\calH_\Theta} \leq \sup_{\substack{g \in \calH \\ \|g\|_{\calH_\Theta}=1}} \|(T+\lambda)  (T+\lambda)^{-1} g\|_{\calH_\Theta}  = 1.
    \end{align*}
\end{proof}

\subsubsection{Convergence of the noisy IW-KRR estimator}
\label{sec:convergence_of_niwkrr}

With the auxiliary results in place, we now extend~\citet[Theorem 4]{gogolashvili2023importance} to the case of noisy weights. We start by establishing convergence in $\calL^2(\Theta, \Qb_w)$, where $\Qb_w$ is the measure defined as $\Qb_w(A) = \int_A w(\theta)\Qb(\mathrm{d} \theta)$ that must be finite and positive. By~\cite[ Proposition 232D]{fremlin2000measure}, for $\Qb_w(A)$ to be a finite positive measure, it is sufficient for $w(\theta)$ to be continuous and bounded. By their definition in~\Cref{eq:weights}, 
\begin{equation*}
    w(\theta) = \E_{y^\theta_{1:N} \sim \Pb_\theta}\frac{\lambda_\Theta}{\lambda_\Theta + \sigma^2_\mathrm{BQ}(\theta; y^\theta_{1:N})}
\end{equation*}
the weights are bounded by $1$, and are continuous in $\theta$ if $p_\theta$ is continuous in $\theta$ (as the dependance of $\sigma^2_\mathrm{BQ}(\theta; y^\theta_{1:N})$ on $\theta$ for a fixed $y^\theta_{1:N}$ is, again, only through $p_\theta$ appearing under integrals and in polynomials). The continuity of $p_\theta$ holds as, by~\ref{as:app_true_I_smoothness},~\ref{as:app_kernel_theta}, $p_\theta$ lies in a Sobolev space of smoothness over $p/2$, and therefore by Sobolev embedding theorem~\citep[Theorem 4.12]{adams2003sobolev} $p_\theta$ is continuous in $\theta$.

\begin{theorem}
\label{res:convergence_of_iwkrr}
Suppose Assumptions~\ref{as:app_true_I_smoothness},~\ref{as:app_domains_theta},~\ref{as:app_theta_samples},~\ref{as:app_x_samples}, and~\ref{as:app_kernel_theta_bounded} hold, and $\lambda \sqrt{M} > (4/C_6) \kappa \log(4/\delta)$ for some $C_6\leq1$. Then,
\begin{equation*}
    \| I_\mathrm{CBQ} - I \|_{\calL^2(\Theta, \Qb_w)} \leq (1-C_6)^{-1} \left( C_6 \sqrt \lambda + 1\right) \|I_\mathrm{BQ}-I \|_{\calH_\Theta} + \left(\frac{8 (1-C_6)^{-1} \kappa}{ \sqrt {\lambda M}} \log(4/\delta)  + \sqrt \lambda \right) \| I\|_{\calH_\Theta}.
\end{equation*}
\end{theorem}
\begin{proof}
    First, note that $I_\mathrm{CBQ}(\theta) = (\hat{T} + \lambda)^{-1}\hat{T}[I_\mathrm{BQ}]$, 
    which can be checked easily by seeing that $(\hat{T} + \lambda)I_\mathrm{CBQ}(\theta) = \hat{T}[I_\mathrm{BQ}]$ for the \emph{weighted} operator $\hat T$ as defined in~\Cref{sec:notation} and $I_\mathrm{CBQ}$ as defined in~\Cref{eq:cbq_noisy_weights}.
    Then, for $I_{\lambda} = (T+\lambda)^{-1} T[I]$, by triangle inequality the error is bounded as
    \begin{align}
    \label{eq:i_cbq_minus_i_triangle}
        \| I_\mathrm{CBQ} - I \|_{\calL^2(\Theta, \Qb_w)} \leq \| I_\mathrm{CBQ} - I_\lambda \|_{\calL^2(\Theta, \Qb_w)} + \| I_\lambda - I \|_{\calL^2(\Theta, \Qb_w)}
    \end{align}
    The second term, $\| I_\lambda - I \|^2_{\calL^2(\Theta, \Qb_w)}$, can be bounded in terms of $\lambda$ as
    \begin{align}
        \| I_\lambda - I \|_{\calL^2(\Theta, \Qb_w)} 
        &= 
        \| \lambda (T+\lambda)^{-1}[I] \|_{\calL^2(\Theta, \Qb_w)} \nonumber \\
        &= 
        \| \lambda T^{1/2} (T+\lambda)^{-1}[I] \|_{\calH_\Theta} \nonumber \\
        &\stackrel{(A)}{\leq} 
        \lambda \|  T (T+\lambda)^{-1} \|^{1/2}_\mathrm{op} \|(T+\lambda)^{-1/2}  \|_\mathrm{op} \| I \|_{\calH_\Theta} \nonumber \\
        &\leq 
        \sqrt{\lambda} \| I \|_{\calH_\Theta} , \label{eq:i_lambda_minus_i_bound}
    \end{align}
    where $(A)$ holds by~\Cref{res:t_t_plus_lambda_inv_bounded} and $T$ being a positive operator.
    Next, the $\calL^2(\Theta, \Qb_w)$ norm between $I_\mathrm{CBQ} - I_{\lambda}$ can be bounded as
    \begin{align*}
        \| I_\mathrm{CBQ} - I_{\lambda} \|_{\calL^2(\Theta, \Qb_w)} = \| T^{1/2} (I_\mathrm{CBQ} - I_{\lambda}) \|_{\calH_\Theta}
        &= 
        \| T^{1/2} ((\hat{T}+\lambda)^{-1} \hat{T}[I_\mathrm{BQ}] - (T+\lambda)^{-1} T[I] ) \|_{\calH_\Theta} \\
        &\stackrel{(A)}{\leq}
        \| T  (T+\lambda)^{-1} \|_\mathrm{op}^{1/2} \| ( \Id - (T+\lambda)^{-1/2} (T - \hat{T}) (T+\lambda)^{-1/2} )^{-1} \|_\mathrm{op} \\
        &\qquad \times \Big(\| (T+\lambda)^{-1/2} (\hat{T}[I_\mathrm{BQ}]-T[I]) \|_{\calH_\Theta} \\
        &\qquad\qquad\qquad+ \| (T+\lambda)^{-1/2} (T-\hat{T}) (T+\lambda)^{-1} T[I] \|_{\calH_\Theta} \Big) \\
        &\stackrel{(B)}{\leq} 
        \| T  (T+\lambda)^{-1} \|_\mathrm{op}^{1/2} \| ( \Id - (T+\lambda)^{-1/2} (T - \hat{T}) (T+\lambda)^{-1/2} )^{-1} \|_\mathrm{op} \\
        &\qquad \times \Big( \| (T+\lambda)^{-1/2} \hat{T}[I_\mathrm{BQ}-I] \|_{\calH_\Theta} \\
        &\qquad\qquad\qquad+ \| (T+\lambda)^{-1/2} (T - \hat T)[I]\|_{\calH_\Theta} \\
        &\qquad\qquad\qquad+ \| (T+\lambda)^{-1/2} (T-\hat{T}) (T+\lambda)^{-1} T[I] \|_{\calH_\Theta} \Big) \\
        &\eqqcolon 
        U_0 \times U_1 \times (U_2 + U_3 + U_4),
    \end{align*}
    where $\|\cdot\|_\mathrm{op}$ denotes the operator norm, $(A)$ holds by~\citet[Lemma 17]{gogolashvili2023importance}, and $(B)$ is an application of triangle inequality,
    \begin{equation*}
        \| (T+\lambda)^{-1/2} (\hat{T}[I_\mathrm{BQ}]-T[I]) \|_{\calH_\Theta} \leq \| (T+\lambda)^{-1/2} \hat{T}[I_\mathrm{BQ}-I] \|_{\calH_\Theta} + \| (T+\lambda)^{-1/2} (T - \hat T)[I] \|_{\calH_\Theta}.
    \end{equation*}
    We will bound the terms $U_0,U_1, U_2, U_3, U_4$, and the result will follow. First, we have that $U_0 = \| T  (T+\lambda)^{-1} \|_\mathrm{op}^{1/2} \leq 1$ by \Cref{res:t_t_plus_lambda_inv_bounded}. To upper bound $U_1=\| ( \Id - (T+\lambda)^{-1/2} (T - \hat{T}) (T+\lambda)^{-1/2} )^{-1} \|_\mathrm{op}$ we may expand it as Neumann series, provided $\| (T+\lambda)^{-1/2} (T - \hat{T}) (T+\lambda)^{-1/2} )^{-1} \|_\mathrm{op} < 1$. This condition holds as
    \begin{equation*}
        \| (T+\lambda)^{-1/2} (T - \hat{T}) (T+\lambda)^{-1/2} \|_\mathrm{op} \stackrel{(A)}{\leq} \| (T+\lambda)^{-1/2} (T - \hat{T}) (T+\lambda)^{-1/2} \|_{\mathrm{HS}} \stackrel{(B)}{<} C_6 \leq 1,
    \end{equation*}
    where $(A)$ holds as the operator norm is bounded by the Hilbert-Schmidt norm, and $(B)$ by~\Cref{res:s1}. 
    Therefore,
    \begin{align*}
        \left\| ( \Id - (T+\lambda)^{-1/2} (T - \hat{T}) (T+\lambda)^{-1/2} )^{-1} \right\|_\mathrm{op} 
        &\stackrel{(A)}{=} \left\| \sum_{i=0}^\infty \left((T+\lambda)^{-1/2} (T - \hat{T}) (T+\lambda)^{-1/2} \right)^i \right\|_\mathrm{op} \\
        &\stackrel{(B)}{\leq} \sum_{i=0}^\infty \left\| (T+\lambda)^{-1/2} (T - \hat{T}) (T+\lambda)^{-1/2}  \right\|^i_\mathrm{op}\\
        &\stackrel{(C)}{\leq} \sum_{i=0}^\infty \left\| (T+\lambda)^{-1/2} (T - \hat{T}) (T+\lambda)^{-1/2}  \right\|^i_{\mathrm{HS}}\\
        &\stackrel{(D)}{=} \left(1 - \left\| (T+\lambda)^{-1/2} (T - \hat{T}) (T+\lambda)^{-1/2}  \right\|_{\mathrm{HS}}\right)^{-1}\\
        &\stackrel{(E)}{\leq} (1 - C_6)^{-1},
    \end{align*}
    where $(A)$ holds by the Neumann series expansion, $(B)$ by the triangle inequality, and the fact that operator norm is sub-multiplicative for bounded operators, $(C)$ since the operator norm is bounded by the Hilbert-Schmidt norm, $(D)$ by the geometric series, and $(E)$ by~\Cref{res:s1}.

    To bound $U_2 = \| (T+\lambda)^{-1/2} \hat{T}[I_\mathrm{BQ}-I] \|_{\calH_\Theta}$, observe that
    \begin{align*}
        U_2 = \| (T+\lambda)^{-1/2} \hat{T}[I_\mathrm{BQ}-I] \|_{\calH_\Theta} &\leq \| (T+\lambda)^{-1/2} \hat{T} \|_\mathrm{op} \|I_\mathrm{BQ}-I \|_{\calH_\Theta} \\
        &\leq
        \left( \| (T+\lambda)^{-1/2} (T - \hat T) \|_\mathrm{op} + \| (T+\lambda)^{-1/2} T \|_\mathrm{op} \right) \|I_\mathrm{BQ}-I \|_{\calH_\Theta} \\
        &\stackrel{(A)}{\leq}
        \left( S_2 + 1\right) \|I_\mathrm{BQ}-I \|_{\calH_\Theta},
    \end{align*}
    where $(A)$ holds by~\Cref{res:s2,res:t_t_plus_lambda_inv_bounded}.

    Both $U_3$ and $U_4$ are upper bounded by the $S_2$ term in~\Cref{res:s2}, as 
    \begin{align*}
        U_3 &=\| (T+\lambda)^{-1/2} (T - \hat T)[I]\|_{\calH_\Theta} \leq \| (T+\lambda)^{-1/2} (T - \hat T)\|_\mathrm{op} \| I \|_{\calH_\Theta} = S_2 \| I \|_{\calH_\Theta}, \\
        U_4 &=\| (T+\lambda)^{-1/2} (T-\hat{T}) (T+\lambda)^{-1} T[I] \|_{\calH_\Theta} 
        \leq
        \| (T+\lambda)^{-1/2} (T - \hat T)\|_\mathrm{op} \| (T+\lambda)^{-1} T \|_\mathrm{op} \| I \|_{\calH_\Theta} \stackrel{(A)}{\leq} 
        S_2 \| I \|_{\calH_\Theta},
    \end{align*}
    where $(A)$ holds by~\Cref{res:t_t_plus_lambda_inv_bounded}. 
    Putting the upper bounds on $U_0, U_1, U_2, U_3, U_4$ together, we get
    \begin{align*}
        \| I_\mathrm{CBQ} - I_{\lambda} \|_{\calL^2(\Theta, \Qb_w)} \leq U_0 \times U_1 \times (U_2 + U_3 + U_4) \leq 4 \left(( S_2 + 1) \|I_\mathrm{BQ}-I \|_{\calH_\Theta} + 2S_2 \| I\|_{\calH_\Theta} \right).
    \end{align*}
    By applying the union bound, we get that that with probability at least $1-\delta$,
    \begin{align*}
        \| I_\mathrm{CBQ} - I_{\lambda} \|_{\calL^2(\Theta, \Qb_w)} 
        &\leq 
        U_0 \times U_1 \times (U_2 + U_3 + U_4) \\
        &\leq 
        (1-C_6)^{-1} \left(( S_2 + 1) \|I_\mathrm{BQ}-I \|_{\calH_\Theta} + 2S_2 \| I\|_{\calH_\Theta} \right) \\
        &\stackrel{(A)}{\leq}
        (1-C_6)^{-1} \left(\left( C_6 \sqrt \lambda + 1\right) \|I_\mathrm{BQ}-I \|_{\calH_\Theta} + \frac{8\kappa}{ \sqrt {\lambda M}} \log(4/\delta) \| I\|_{\calH_\Theta} \right),
    \end{align*}
    where $(A)$ holds by~\Cref{res:s2}. Inserting this and the bound in~\Cref{eq:i_lambda_minus_i_bound} into~\Cref{eq:i_cbq_minus_i_triangle} gives
    \begin{equation*}
        \| I_\mathrm{CBQ} - I \|_{\calL^2(\Theta, \Qb_w)} \leq (1-C_6)^{-1} \left( C_6 \sqrt \lambda + 1\right) \|I_\mathrm{BQ}-I \|_{\calH_\Theta} + \left(\frac{8 (1-C_6)^{-1} \kappa}{ \sqrt {\lambda M}} \log(4/\delta)  + \sqrt \lambda \right) \| I\|_{\calH_\Theta}.
    \end{equation*}
\end{proof}
Finally, we use the $\calL^2(\Theta, \Qb_w)$ bound in~\Cref{res:convergence_of_iwkrr} to establish a bound in $\calL^2(\Theta, \Qb)$ in terms of the BQ variance.
\begin{cor}
\label{res:unweighted_bound_stage_2}
    Suppose Assumptions~\ref{as:app_true_I_smoothness},~\ref{as:app_domains_theta},~\ref{as:app_theta_samples},~\ref{as:app_x_samples}, and~\ref{as:app_kernel_theta_bounded}, and $\lambda \sqrt{M} > (4/C_6) \kappa \log(4/\delta)$ for some $C_6\leq1$. Then
    \begin{align*}
        \| I_\mathrm{CBQ} - I \|_{\calL^2(\Theta, \Qb)} &\leq \left(1 + \frac{1}{c \sqrt{M}} \sup_{\theta \in \Theta} \E_{y^\theta_{1:N} \sim \Pb_\theta}\sigma^2_\mathrm{BQ}(\theta; y^\theta_{1:N})\right) \\
        &\qquad \times \left( (1-C_6)^{-1} \left( C_6 \sqrt \lambda + 1\right) \|I_\mathrm{BQ}-I \|_{\calH_\Theta} + \left(\frac{8 (1-C_6)^{-1} \kappa}{ \sqrt {\lambda M}} \log(4/\delta)  + \sqrt \lambda \right) \| I\|_{\calH_\Theta}\right)
    \end{align*}
\end{cor}
\begin{proof}
    Observe that for any $g \in \calL^2(\Theta, \Qb)$, it holds that $\|g\|^2_{\calL^2(\Theta, \Qb_w)} \geq \left(\inf_{\theta \in \Theta} w(\theta) \right) \times \|g\|^2_{\calL^2(\Theta, \Qb)}$. Then, since
    \begin{align*}
        w(\theta) = \E_{y^\theta_{1:N} \sim \Pb_\theta}\frac{\lambda_\Theta}{\lambda_\Theta + \sigma^2_\mathrm{BQ}(\theta; y^\theta_{1:N})} \geq \frac{\lambda_\Theta}{\lambda_\Theta + \E_{y^\theta_{1:N} \sim \Pb_\theta} \sigma^2_\mathrm{BQ}(\theta; y^\theta_{1:N})} = \frac{1}{1 + \lambda_\Theta^{-1} \E_{y^\theta_{1:N} \sim \Pb_\theta} \sigma^2_\mathrm{BQ}(\theta; y^\theta_{1:N})},
    \end{align*}
    the bound in~\Cref{res:convergence_of_iwkrr}, the definition of $\lambda$ in~\Cref{eq:weights}, and Assumption~\ref{as:app_regulariser_theta} give the desired statement.
\end{proof}

\subsection{Proof of Theorem~\ref{thm:convergence}}
\label{sec:proof_of_main_theorem}
We are now ready to prove our main convergence result, which is a version of~\Cref{thm:convergence} for $\lambda_\calX \geq 0$. We start by restating it for the convenience of the reader.

\begin{proof}[Restatement of~\Cref{thm:convergence_generalised}]
    Suppose all technical assumptions in~\Cref{sec:technical_assumptions} hold. Then for any $\delta \in (0, 1)$ there is an $N_0>0$ such that for any $N \geq N_0$, with probability at least $1-\delta$ it holds that
    \begin{align*}
        \| I_\mathrm{CBQ} - I \|_{\calL^2(\Theta, \Qb)} &\leq \left(1 + c^{-1} M^{-\frac{1}{2}}\left(\lambda_\calX + C_2 N^{-1 + 2\varepsilon} \left( N^{-\frac{s_\calX}{d}+\frac{1}{2} + \varepsilon} + C_3 \lambda_\calX \right)^2\right) \right) \\
        &\qquad \times \left( C_7(\delta) N^{-\frac{1}{2} + \varepsilon} \left( N^{-\frac{s_\calX}{d} + \frac{1}{2} + \varepsilon} + C_5 \lambda_\calX \right) + C_8(\delta) M^{-\frac{1}{4}} \| I\|_{\calH_\Theta} \right)
    \end{align*}
    for any arbitrarily small $\varepsilon>0$, constants $C_2, C_3, C_5$, $C_7(\delta) = O(1/\delta)$ and $C_8(\delta) = O(\log(1/\delta))$ independent of $N, M, \varepsilon$.
\end{proof}

\begin{proof}[Proof of~\Cref{thm:convergence_generalised}]
    By inserting~\Cref{res:bound_on_bq_error,res:bound_on_bq_var} into~\Cref{res:unweighted_bound_stage_2} and applying the union bound, we get that the result holds with probability at least $1-\delta$ and
    \begin{align*}
        C_7(\delta) &= (1-C_6)^{-1} \left( C_6 c^{\frac{1}{2}} + 1\right) C_4 (2/\delta), \\
        C_8(\delta) &= \left( 8c^{-\frac{1}{2}} (1-C_6)^{-1} \kappa \log(4/\delta)  + c^{\frac{1}{2}} \right).
    \end{align*}
\end{proof}

As discussed in the main text, convergence is fastest when the regulariser $\lambda_\calX$ is set to $0$; $\lambda_\calX>0$ ensures greater stability at the cost of a lower speed of convergence. For clarity we show how~\Cref{thm:convergence} in the main text follows from the more general~\Cref{thm:convergence_generalised} by setting $\lambda_\calX=0$. 
    
\begin{proof}[Proof of~\Cref{thm:convergence}]         
    In~\Cref{thm:convergence_generalised}, take $\lambda_\calX=0$. Then
    \begin{align*}
        \| I_\mathrm{CBQ} - I \|_{\calL^2(\Theta, \Qb)} &\leq \left(1 + c^{-1} M^{-\frac{1}{2}} C_2 N^{-\frac{2s_\calX}{d} + \varepsilon} \right) \times \left( C_7(\delta) N^{-\frac{s_\calX}{d} + \varepsilon} + C_8(\delta) M^{-\frac{1}{4}} \| I\|_{\calH_\Theta} \right).
    \end{align*}
    As $\Qb$ was assumed equivalent to the uniform distribution in Assumption~\ref{as:app_theta_samples}, the error in uniform measure is bounded by the error in $\Qb$. Therefore, the result holds for
    \begin{align*}
        C_0(\delta)&=\left(1 + c^{-1} C_2 \right) C_7(\delta)=O(1/\delta), \\
        C_1(\delta)&=\left(1 + c^{-1} C_2 \right) \|I\|_{\calH_\Theta} C_8(\delta) =O(\log(1/\delta)).
    \end{align*}
\end{proof}

\section{Practical Considerations for Conditional Bayesian Quadrature}\label{appendix:practical_considerations}

We now discuss important practical considerations which can have significant impact on the performance of CBQ. Firstly, in \Cref{appendix:tractable_kernel_means} we discuss how to ensure a closed-form expression for kernel mean embeddings and initial errors of BQ estimators. Then, we discuss the selection of all kernel hyperparameters in \Cref{appendix:hyperparameter_selection}.

\subsection{Tractable Kernel Means}\label{appendix:tractable_kernel_means}

In the main text, we discussed the requirement for both BQ and  CBQ that the kernel mean embedding $\mu$ and its integral (called initial error) are known in closed-form. A list of well-known pair can be found in Table 1 in \citep{fx_quadrature} or the \texttt{ProbNum} package \citep{Wenger2021}. 
However, even when none of these pairs are appropriate for the problem at hand, there are still multiple solutions:

\begin{itemize}
    \item First, for a fixed $k$, when the embedding of $\mathbb{P}$ is intractable but the embedding of some other distribution $\mathbb{Q}$ is known, we can use the `importance sampling trick' which consists of writing the integral as $I=\mathbb{E}_{X \sim \mathbb{P}} [f(X)] = \mathbb{E}_{X \sim \mathbb{Q}} [g(X)]$ where $g(x)=f(x)p(x)/q(x)$ and $p,q$ are the densities of $\mathbb{P},\mathbb{Q}$. This allows us to use BQ on the integral of $g$, which is tractable by construction.

\vspace{2mm}

\item Secondly, again for a fixed $k$ and assuming that we know the quantile function $\Phi^{-1}$ of the distribution $\mathbb{P}$ and that the embedding of the uniform distribution is available, we can use the `inverse transform trick' which consists of writing $I=\mathbb{E}_{X \sim \mathbb{P}} [f(X)] = \mathbb{E}_{U \sim \mathbb{U}} [g(U)]$ where $g(u) = f(\Phi^{-1}(u))$ and $\mathbb{U}$ is a uniform distribution on some hypercube. Once again, BQ can now be applied to the transformed problem.

\vspace{2mm}

\item Finally, for any distribution $\mathbb{P}$ whose density is known up to the normalisation constant (for example most posterior distributions), then specialised kernels with closed-form embeddings can be constructed. This is true of Stein reproducing kernels~\cite{anastasiou2023stein}. Suppose we have a distribution $\mathbb{P}$ with density $p:\calX \rightarrow \R^+$ and a function $f:\calX \rightarrow \R$ with the property that $\lim_{x \to \infty} p(x)f(x) = 0$. The Langevin Stein kernel $k: \calX \times \calX \to \R$ \citep{anastasiou2023stein} is given by:
\begin{align*}
    k_p(x,x') & := \nabla_x \log p(x)^\top k(x, x' ) \nabla_{x'} \log p(x') + \nabla_{x} \log p(x) ^\top \nabla_{x'} k(x, x') 
    \\ & \qquad + \nabla_{x'} \log p(x') ^\top \nabla_x k(x, x') + \nabla_x \cdot \nabla_{x'} k(x, x'),
\end{align*}
where $\nabla_x = (\partial/ \partial x_1, \cdots, \partial/ \partial x_d)^\top$ and $\nabla_x \cdot \nabla_{x'} k(x, x') = \sum_{i=1}^d \frac{\partial k(x, x')}{\partial x_i \partial x_i'}$.

The main advantage of using Stein kernel is that the mean embedding $\mu(x') = \int_{\calX} k_p(x, x')p(x)dx = 0$ by construction. 
However, this means our GP prior on $f$ encodes beliefs that the function has mean zero.
To weaken this, we can add a constant $c \in \R$; i.e $\tilde{k}_p(x, x') = k_p(x,x') + c$, so that the kernel mean embedding becomes $\mu(x') = c$. 
The constant $c$ can then be treated as a kernel hyperparameter and estimated alongside all other parameters. 

\end{itemize}


\subsection{Model and Hyperparameter Selection}\label{appendix:hyperparameter_selection}

We now discuss our approach for model and hyperparameter selection for CBQ and baseline methods. 

\paragraph{Conditional Bayesian quadrature} The hyperparameter selection for CBQ boils down to the choice of GP interpolation hyperparameters at stage 1 and the choice of GP regression hyperparameters at stage 2. To simplify this choice, we renormalise all our function values before performing GP regression and interpolation. This is done by first subtracting the empirical mean and then dividing by the empirical standard deviation. 
All of our experiments then use prior mean functions $m_\Theta$ and $m_\calX$ which are zero functions, a reasonable choice given the function was renormalised using the empirical mean. This choice is made for simplicity, and we might expect further improvements in accuracy if more information is available.

The choice of covariance functions $k_\calX$ and $k_\Theta$ is made on a case-by-case basis in order to both encode properties we expect the target functions to have, but also to ensure that the corresponding kernel mean is available in closed-form (as per the previous section). Once this is done, we typically still need to make a choice of hyperparameters for both kernel: lengthscales $l_\calX$, $\l_\Theta$ and amplitudes $A_\calX, A_\Theta$. 
We also need to select the regularizer $\lambda_\calX, \lambda_\Theta$. 
$\lambda_\calX$ is fixed to be $0$ as suggested by \Cref{thm:convergence}.
The rest of the hyperparameters are selected through empirical Bayes, which consists of maximising the log-marginal likelihood.
For stage 1, the log-marginal likelihood can be written as~\citep{GPML}:
\begin{align*}
& L(l_\calX, A_\calX) =  -\frac{1}{2} \log \left| k_{\calX}(x_{1:N},x_{1:N}; l_\calX, A_\calX) \right| - \frac{N}{2} \log(2 \pi) \\
& \quad -\frac{1}{2}(f(x_{1:N})-m_{\calX}(x_{1:N}))^\top \left(k_{\calX}(x_{1:N},x_{1:N};l_\calX, A_\calX) + \lambda_\calX \Id_N \right)^{-1} (f(x_{1:N})-m_{\calX}(x_{1:N})),
\end{align*}
where $|\cdot|$ denotes the determinant of the matrix, and we write $k_{\calX}(x_{1:N},x_{1:N}; l_\calX, A_\calX)$ to emphasise the hyperparameters used to compute the Gram matrix.
The optimisation is implemented through a grid search over $\left[1.0, 10.0, 100.0, 1000.0 \right]$ for the amplitude $A_\calX$ and a grid search over $\left[0.1, 0.3, 1.0, 3.0, 10.0 \right]$ for the lengthscale $l_\calX$. 

If $k_\calX$ is a Stein reproducing kernel, we have an extra hyperparameter $c_\calX$. 
In this case, we use stochastic gradient descent on the log-marginal likelihood to find the optimal value for $c_\calX, l_\calX, A_\calX$, which is implemented with \texttt{JAX} autodiff library~\citep{jax2018github}. 
The reason we are using gradient based optimization instead of grid search for Stein kernel is that Stein kernel requires an accurate estimate of $c_\calX$ to work well. 
In order to return accurate results, grid search would require finer grid which is very expensive, while gradient based methods would require good initialization to avoid getting stuck in local minima. Fortunately, since $c_\calX$ indicates the mean of functions in the RKHS, we know that $c_\calX = 0$ is a good initialisation point since we have subtracted the empirical mean when normalising.

Additionally, it is important to note that we could technically use $T$ different kernels $k_\calX^1, \cdots, k_\calX^T$ for each integral in stage 1. However, the hyperparameters of each kernel $k_\calX^t$ would need to be selected using empirical Bayes under the observations $x_{1:N}^t$, which means we would need to repeat the above optimization $T$ times. In practice, when performing initial experiments, we observed that the estimated hyperparameters were very similar. Our strategy is therefore to select the hyperparameters of $k_\calX^1$ and subsequently reuse them across all $T$ integrals in stage 1. This is done for computational reasons, and we expect CBQ to show better performances if hyperparameters are optimised separately.

For the kernel $k_\Theta$, we also select the hyperparameters by maximising the log-marginal likelihood: 
\begin{align*}
   & L(l_\Theta, A_\Theta) =  -\frac{1}{2} \log | k_{\Theta}(\theta_{1:T},\theta_{1:T}; l_\Theta, A_\Theta)| - \frac{T}{2} \log(2 \pi)
     \\
& \quad  -\frac{1}{2} (I_\mathrm{BQ} (\theta_{1:T})- m_{\Theta}(\theta_{1:T}))^\top \left(k_{\Theta}(\theta_{1:T}, \theta_{1:T};l_\Theta, A_\Theta) + \left( \lambda_\Theta + \sigma_{\mathrm{BQ}}^2(\theta_{1:T}) \right) \Id_T \right)^{-1} (I_\mathrm{BQ} (\theta_{1:T})- m_{\Theta}(\theta_{1:T})).
\end{align*}
Similar to above, we also do a grid search over $\left[1.0, 10.0, 100.0, 1000.0 \right]$ for amplitude $A_\Theta$, a grid search over $\left[0.1, 0.3, 1.0, 3.0, 10.0 \right]$ for lengthscale $l_\Theta$ and a grid search over $\left[0.01, 0.1, 1.0 \right]$ for $\lambda_\Theta$, so we select the value that gives the largest log-marginal likelihood. 

\paragraph{Least-squares Monte Carlo} LSMC implements Monte Carlo in the first stage and polynomial regression in the second stage. In the second stage, the hyperparameters include the regularisation coefficient $\lambda_\Theta$ and the order of the polynomial $p \in \{1,2,3,4\}$.
These hyperaparameters are also selected with grid search to give the lowest RMSE on a separate held out validation set.

\paragraph{Kernel least-squares Monte Carlo}
KLSMC implements Monte Carlo in the first stage and kernel ridge regression in the second stage. In the second stage, the hyperparameters are analogous to the hyperparameters in the second stage of CBQ, namely $A_\Theta, l_\Theta, \lambda_\Theta$.
These hyperaparameters are selected with grid search to give the lowest RMSE on a separate held out validation set.

\paragraph{Importance sampling} For IS, there are no hyperparameters to select.


\section{Additional Experiments}\label{appendix:experiments}
We now provide detailed description of all experiments in the main text, as well as further results and ablation studies. All figures reported in the paper are created using the median values obtained from $20$ separate runs with different random seeds. Standard error is shown as shaded area around the median.

\subsection{Synthetic Experiment: Bayesian Sensitivity Analysis for Linear Models}\label{appendix:bayes_sensitivity}

\subsubsection{Experimental Setting}

\begin{figure}[t]
    \centering
    \begin{minipage}{1.0\textwidth}
    \centering
    \includegraphics[width=250pt]{figures/legend.pdf}
    \vspace{-5pt}
    \end{minipage}
    
    \centering
    \begin{subfigure}{0.33\textwidth}
        \centering
        \hspace{-10pt}
        \includegraphics[width=\textwidth]{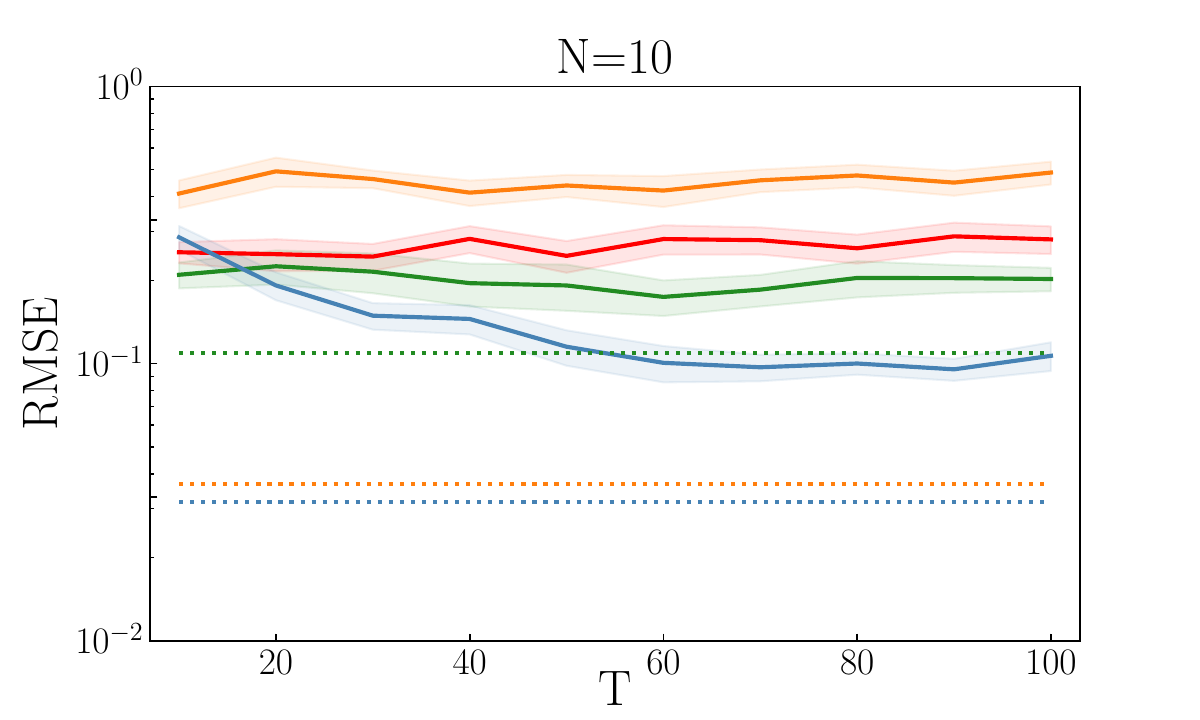}
        \caption{RMSE with fixed $N=10$}
    \end{subfigure}%
    \hfill 
    \begin{subfigure}{0.33\textwidth}
        \centering
        \hspace{-10pt}
        \includegraphics[width=\textwidth]{figures/bayes_sensitivity_N_50.pdf}
        \caption{RMSE with fixed $N=50$.}
    \end{subfigure}%
    \hfill 
    \begin{subfigure}{0.33\textwidth}
        \centering
        \hspace{-10pt}
        \includegraphics[width=\textwidth]{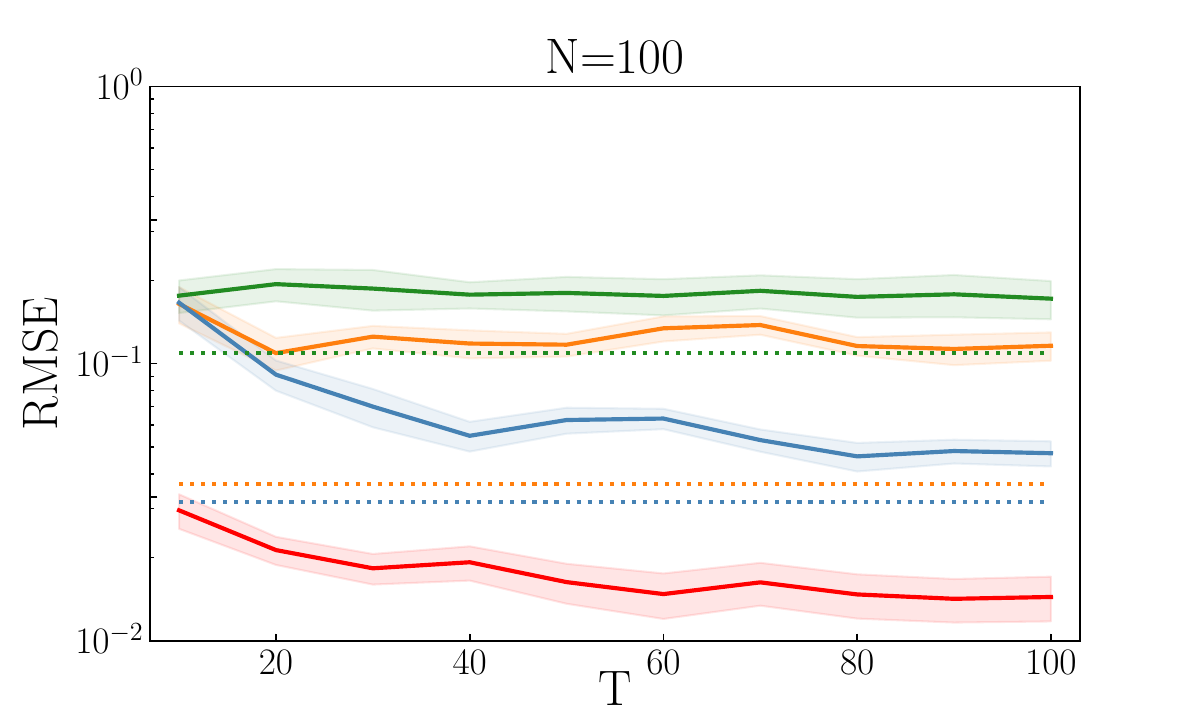}
        \caption{RMSE with fixed $N=100$.}
    \end{subfigure}
        \centering
    \begin{subfigure}{0.33\textwidth}
        \centering
        \hspace{-10pt}
        \includegraphics[width=\textwidth]{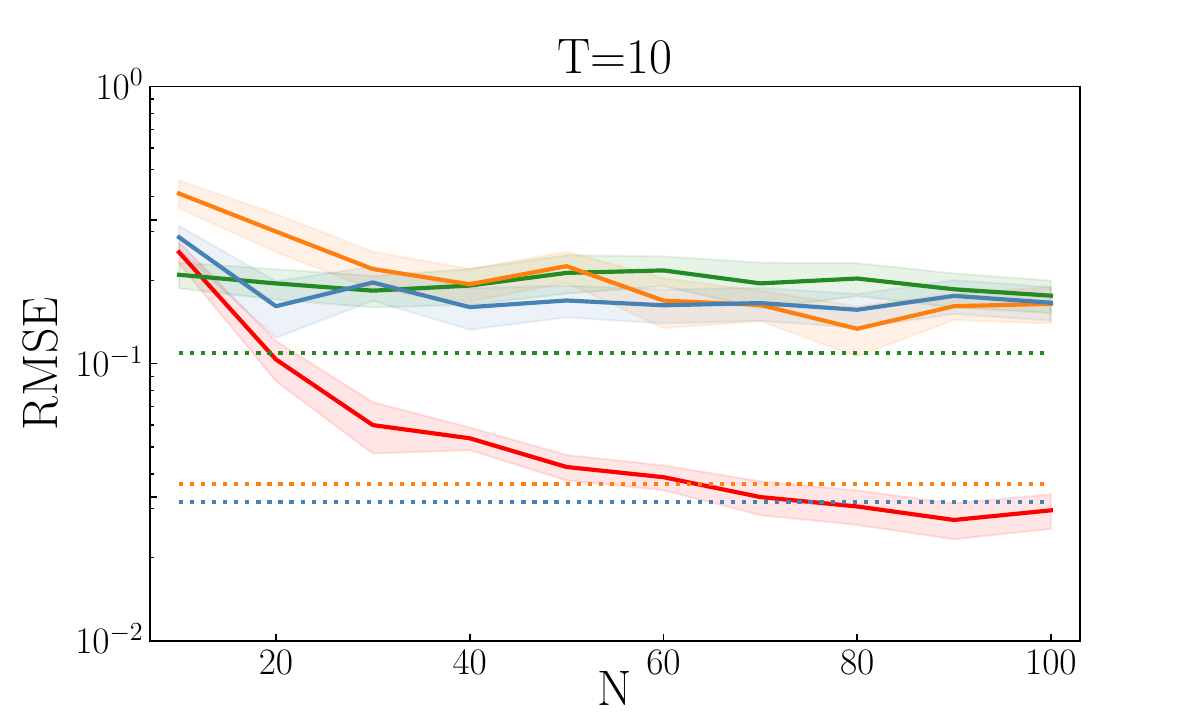}
        \caption{RMSE with fixed $T=10$}
    \end{subfigure}%
    \hfill 
    \begin{subfigure}{0.33\textwidth}
        \centering
        \hspace{-10pt}
        \includegraphics[width=\textwidth]{figures/bayes_sensitivity_T_50.pdf}
        \caption{RMSE with fixed $T=50$.}
    \end{subfigure}%
    \hfill 
    \begin{subfigure}{0.33\textwidth}
        \centering
        \hspace{-10pt}
        \includegraphics[width=\textwidth]{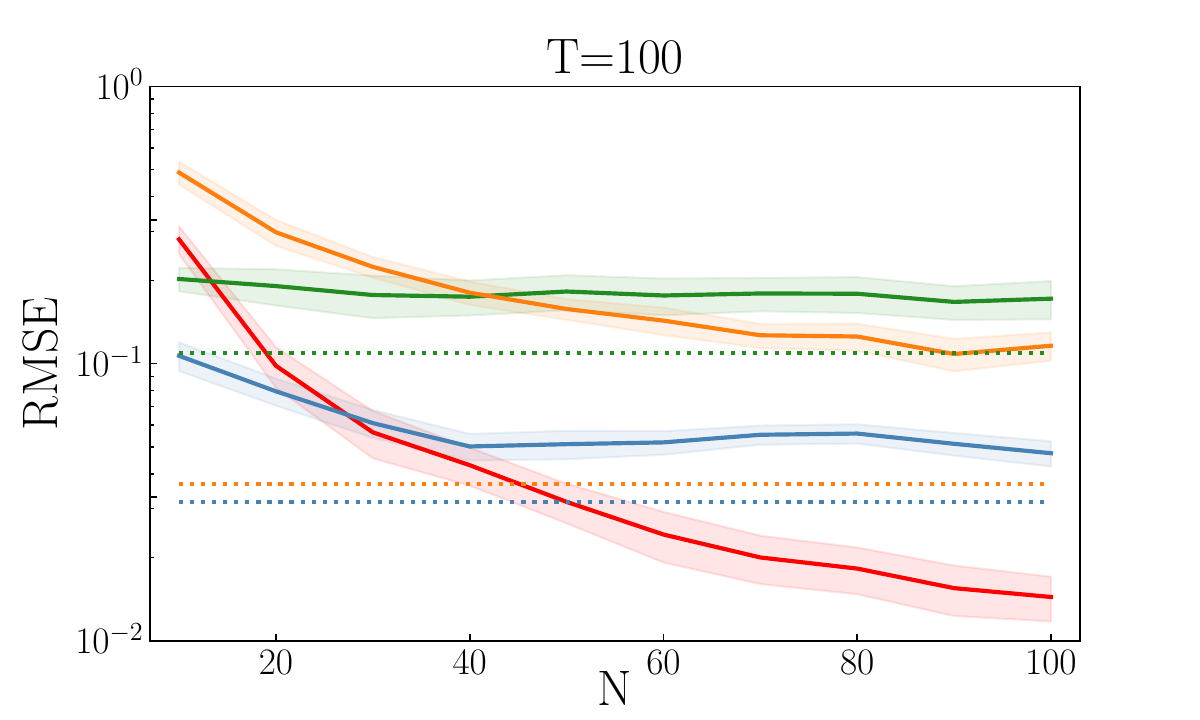}
        \caption{RMSE with fixed $T=100$.}
    \end{subfigure}
    \caption{\emph{Bayesian sensitivity analysis for linear models}. \textbf{First Row:} Dimension $d=2$ with fixed $N=10, 50, 100$ and increasing $T$. 
    \textbf{Second Row:} Dimension $d=2$ with fixed $T=10, 50, 100$ and increasing $N$. The intergral is $f(x) = x^\top x$.} 
    \label{appfig:bayes_sensitivity_1}
\end{figure}

In this synthetic experiment, we do sensitivity analysis on the hyperparameters in Bayesian linear regression. 
The observational data for the linear regression are $Y \in \R^{m \times d}, Z \in \R^{m}$ with $m$ being the number of observations and $d$ being the dimension.
We use $x$ to denote the regression weight; this is unusual but is done so as to keep the notation consistent with the main text.
By placing a $\calN(x ; 0, \theta \Id_d)$ prior 
on the regression weights $x \in \R^{d}$ with $\theta \in \left( 1, 3\right)^d$, and assuming independent $\calN(0, \eta)$ observation noise for some known $\eta>0$, we can obtain (via conjugacy) a multivariate Gaussian posterior $\Pb_\theta$ whose mean and variance have a closed form expression~\citep{bishop:2006:PRML}.
\begin{align*}
    \Pb_\theta = \calN(\tilde{m}, \tilde{\Sigma}), \quad \tilde{\Sigma}^{-1} = {\frac{1}\theta \Id_d} + \eta Y^\top Y, \quad \tilde{m} = \eta \tilde{\Sigma} Y^\top Z.
\end{align*}
We can then analyse sensitivity by computing the conditional expectation $I(\theta)=\int_\calX f(x)\Pb_\theta(dx)$ of some quantity of interest $f$.
For example, if  $f(x)=x^\top x$, then $I(\theta)$ is the second moment of the posterior and the results are already reported in the main text.
If $f(x) = x^\top y^\ast$ for some new observation $y^\ast$, then $I(\theta)$ is the predictive mean. 
In these simple settings, $I(\theta)$ can be
computed analytically, making this a good synthetic example for benchmarking.
We sample parameter values $\theta_{1:T}$ from a uniform distribution $ \Qb = \operatorname{Unif}(\Theta)$ where $\Theta = (1, 3)^d$, and for each such parameter $\theta_t$, we obtain $N$ observations $x_{1:N}^t$ from $\Pb_{\theta_t}$.  
In total, we have $N \times T$ samples.

For conditional Bayesian quadrature (CBQ), we need to carefully choose two kernels $k_\Theta$ and $k_\calX$. Firstly, we choose the kernel $k_\calX$ to be an isotropic Gaussian kernel: $k(x, x') = A_\calX \exp(-\frac{1}{2 l_\calX^2} (x - x')^\top(x - x'))$ for the purpose that the Gaussian kernel mean embedding has a closed form under the Gaussian posterior $\Pb_\theta$:
\begin{align}\label{appeq:E14}
    \mu_\theta(x) = A_\calX {\left| Id_d + l_\calX^{-2} \tilde{\Sigma} \right|}^{-1/2} \exp \left(-\frac{1}{2} (x - \tilde{m})^\top (\tilde{\Sigma} + l_\calX^2 \Id_d)^{-1} (x - \tilde{m})\right)
\end{align}
In addition, the integral of the kernel mean embedding $\mu_\theta$ (known as the initial error) also has a closed form
$\int_{\calX} \mu_\theta(x) \Pb_\theta(dx) = A_\calX l_\calX / \sqrt{| l_\calX^2 \Id_d + 2 \tilde{\Sigma}|}$.

This leaves us with a choice for $k_\Theta$. 
In this synthetic setting, we know that $I(\theta)$ is infinitely times differentiable, but we opt for Mat\'ern-3/2 kernel $k_\Theta(\theta, \theta') = A_\Theta (1+\sqrt{3} |\theta - \theta'|/l_\Theta) \exp (-\sqrt{3} |\theta - \theta'|/l_\Theta)$ to encode a more conservative prior information on the smoothness of $I(\theta)$.

\begin{figure}[t]
    \centering
    \begin{minipage}{1.0\textwidth}
    \centering
    \includegraphics[width=250pt]{figures/legend.pdf}
    \vspace{-5pt}
    \end{minipage}
    
    \centering
    \begin{subfigure}{0.33\textwidth}
        \centering
        \hspace{-10pt}
        \includegraphics[width=\textwidth]{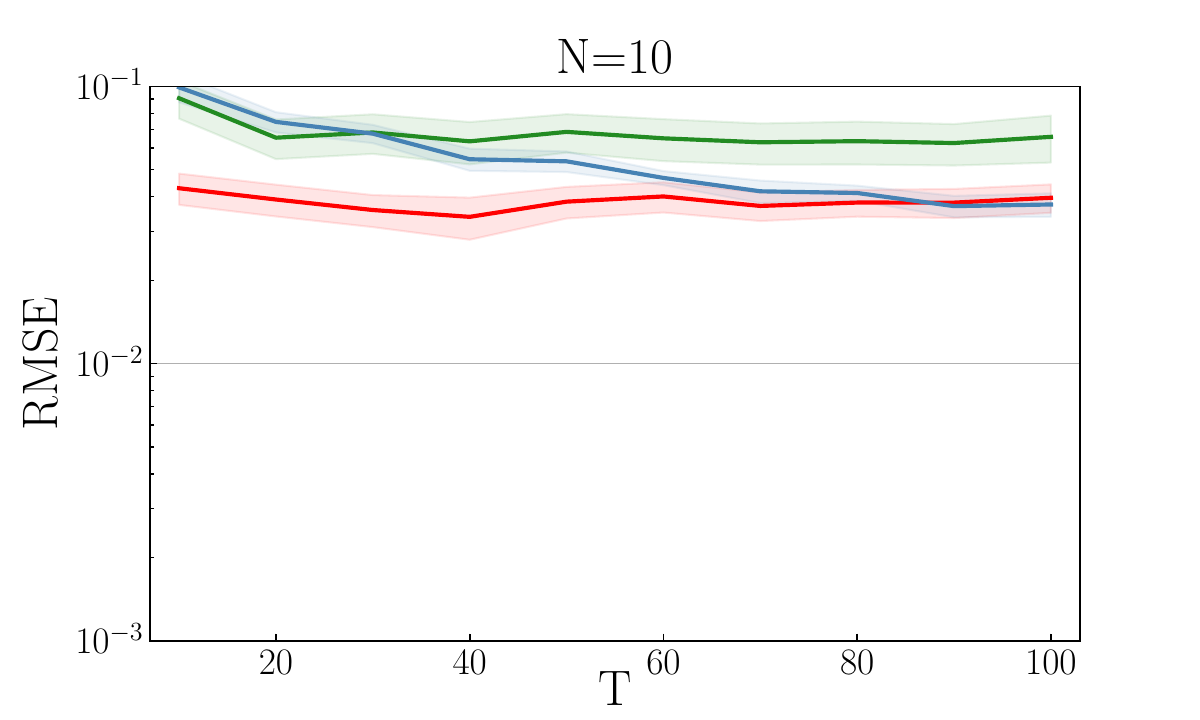}
        \caption{RMSE with fixed $N=10$}
    \end{subfigure}%
    \hfill 
    \begin{subfigure}{0.33\textwidth}
        \centering
        \hspace{-10pt}
        \includegraphics[width=\textwidth]{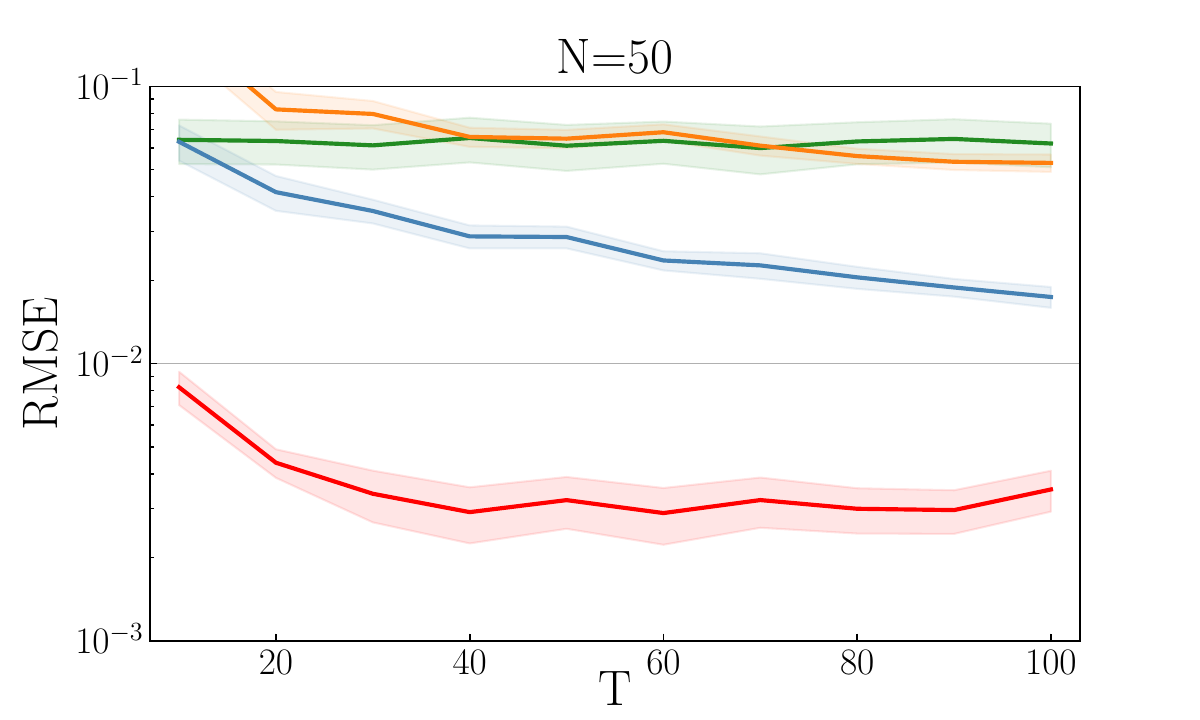}
        \caption{RMSE with fixed $N=50$.}
    \end{subfigure}%
    \hfill 
    \begin{subfigure}{0.33\textwidth}
        \centering
        \hspace{-10pt}
        \includegraphics[width=\textwidth]{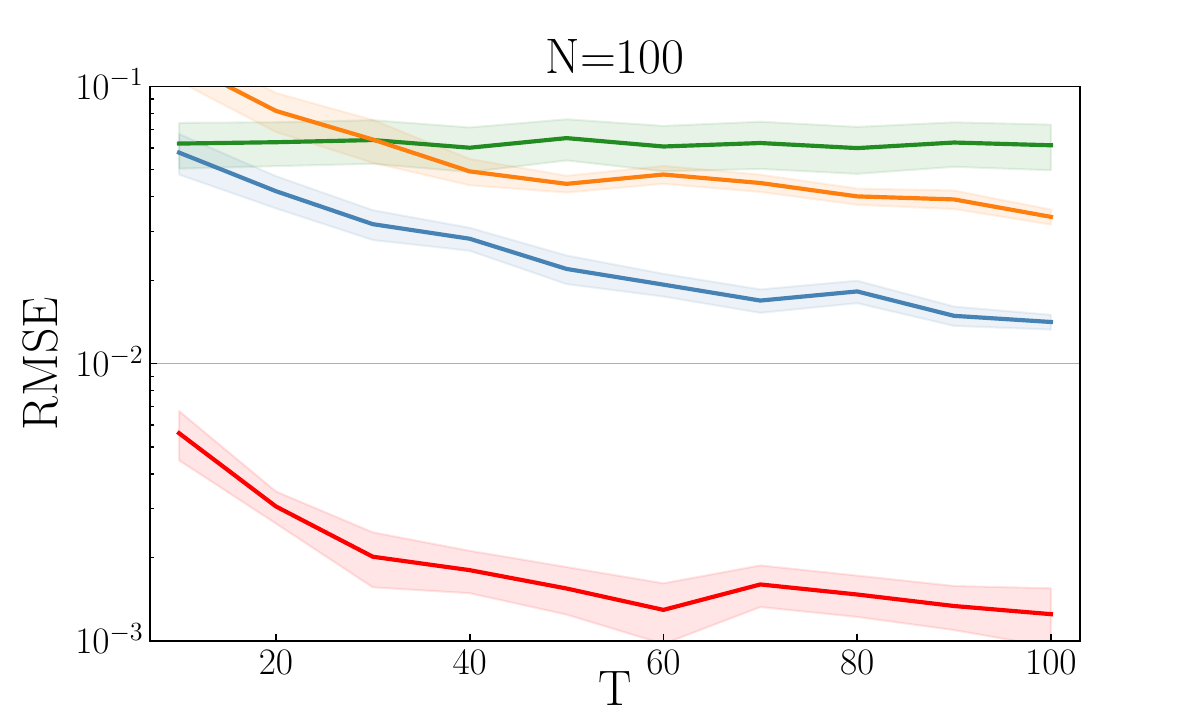}
        \caption{RMSE with fixed $N=100$.}
    \end{subfigure}
        \centering
    \begin{subfigure}{0.33\textwidth}
        \centering
        \hspace{-10pt}
        \includegraphics[width=\textwidth]{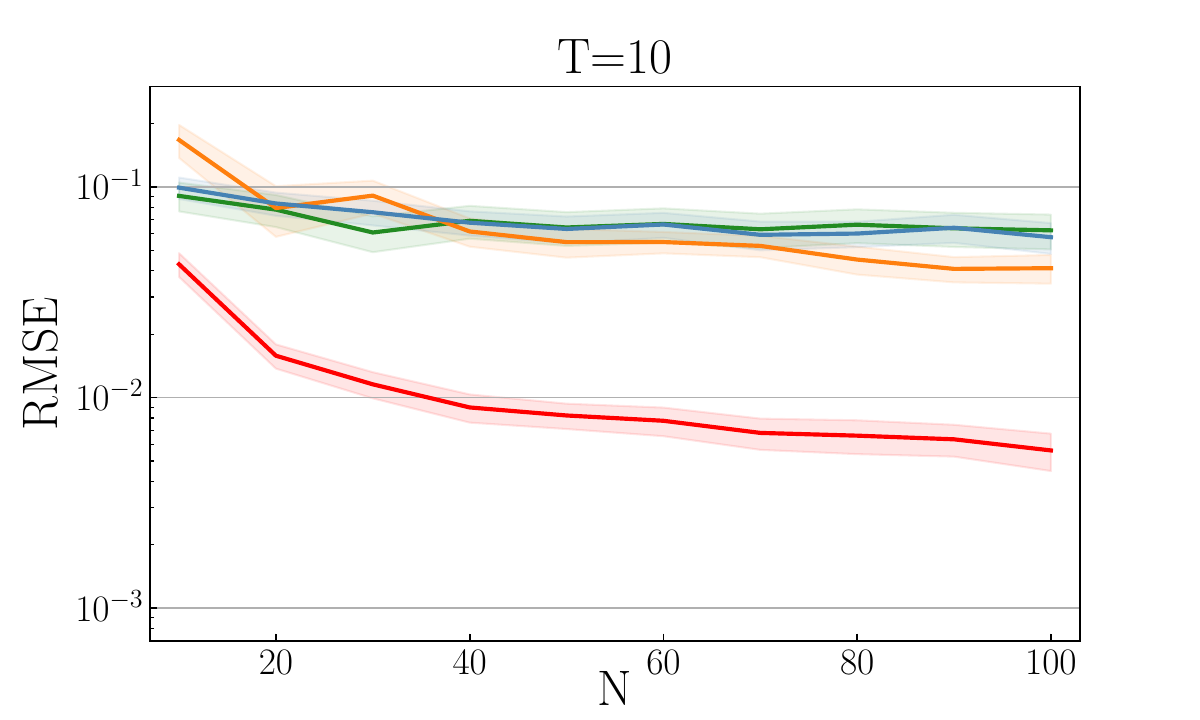}
        \caption{RMSE with fixed $T=10$}
    \end{subfigure}%
    \hfill 
    \begin{subfigure}{0.33\textwidth}
        \centering
        \hspace{-10pt}
        \includegraphics[width=\textwidth]{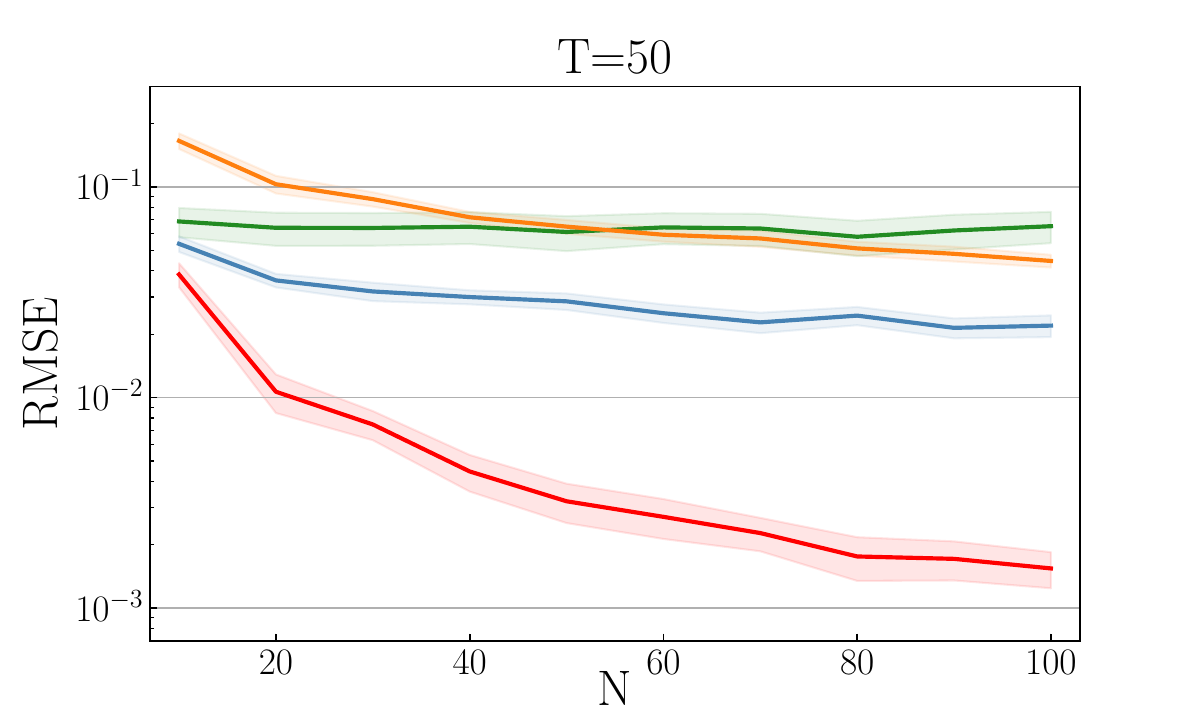}
        \caption{RMSE with fixed $T=50$.}
    \end{subfigure}%
    \hfill 
    \begin{subfigure}{0.33\textwidth}
        \centering
        \hspace{-10pt}
        \includegraphics[width=\textwidth]{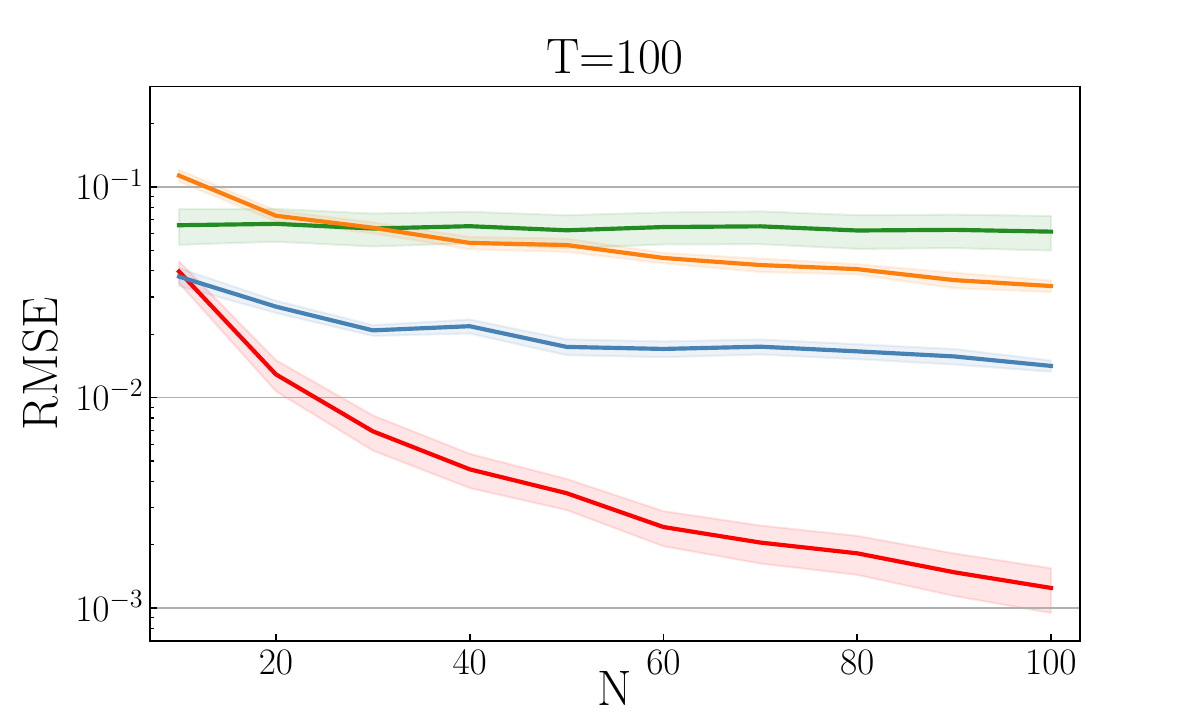}
        \caption{RMSE with fixed $T=100$.}
    \end{subfigure}
    \caption{\emph{Bayesian sensitivity analysis for linear models}. \textbf{First Row:} Dimension $d=2$ with fixed $N=10, 50, 100$ and increasing $T$. 
    \textbf{Second Row:} Dimension $d=2$ with fixed $T=10, 50, 100$ and increasing $N$. The intergral is $f(x) = x^\top y^\ast$. } \label{appfig:bayes_sensitivity_2}
\end{figure}

\subsubsection{Assumptions from \Cref{thm:convergence}} 
We would like to check whether the assumptions made in \Cref{thm:convergence} hold in this experiment.
\begin{itemize}
    \item A1: 
    Although $\calX=\R$ is not a compact domain, $\Pb_\theta$ is a Gaussian distribution so the probability mass outside a large compact subset of $\calX$ decays exponentially. $\Theta = \left( 1, 3 \right)^d$ is a compact domain. A1 is therefore approximately satisfied.
    \item A2: A2 is satisfied due to the sampling mechanism of $\theta_{1:T}$ and $\{x_{1:N}^t\}_{t=1}^T$.
    \item A3: $\Qb$ is a uniform distribution so its density $q$ is constant and hence upper bounded and strictly positive. $\Pb_\theta$ is a Gaussian distribution so its density $p_\theta$ is strictly positive on a compact and large domain with finite second moment. A3 is approximately satisfied.
    \item A4: Both $f(x)$ and $I(\theta)$ are infinitely times differentiable, so $s_I=s_f = \infty$. 
    Although $k_\calX$ is Gaussian kernel which does not satisfy the assumption of \Cref{thm:convergence}, we have ablation study in \Cref{appendix:ablation} showing similar performance when $k_\calX$ is Mat\'ern-3/2 kernel so $s_\calX = \frac{3}{2} + \frac{d}{2}$, and $k_\Theta$ is Mat\'ern-3/2 kernel so $s_\Theta = \frac{3}{2} + \frac{d}{2}$, where $d$ is the dimension. A4 is satisfied.
    \item A5: $\lambda_\calX$ is picked to be $0$ and $\lambda_\Theta$ is found via grid search among $\{0.01, 0.1, 1.0\}$. A5 is satisfied.
\end{itemize}
 
\subsubsection{Additional Experimental Results}
We now provide additional experimental results for Bayesian sensitivity analysis in linear models. 
\Cref{appfig:bayes_sensitivity_1} provides the result when the integrand is chosen to be $f(x)=x^\top x$ so $I(\theta)$ represents the posterior second moment, and \Cref{appfig:bayes_sensitivity_2} provides the result when the integrand is chosen to be $f(x)=x^\top y^\ast$ so $I(\theta)$ represents the predictive mean.
We can see that CBQ has demonstrated consistent smaller RMSE for both tasks under the same number of samples and faster convergence rate compared to all other baseline methods. The conclusions that we draw from the main text also hold for different values of $N$ and $T$.
By comparing the performance of CBQ and KLSMC, where the second stage of both methods are identical, and the main difference lies in the first stage, we believe that CBQ shows better performances mainly due to using Bayesian quadrature instead of Monte Carlo in the first stage.
Also by comparing the first and second row in both \Cref{appfig:bayes_sensitivity_1} and \Cref{appfig:bayes_sensitivity_2}, we can confirm the theory we proved in \Cref{appendix:convergence_rate} that CBQ has a faster convergence rate in $N$ than in $T$. 

In general, CBQ is more computationally expensive than baselines (KLSMC, LSMC and IS), so in this simple setting it is more efficient to spend more budget on obtaining more samples. 
Nonetheless, in scenarios where the expense of sample collection constitutes a significant fraction of the computational budget, or when the evaluation of the integrand proves to be highly costly, it becomes more cost-effective to spend a larger share of the budget towards CBQ. For example, sampling can become expensive easily when the prior and likelihood are not conjugate, so Markov chain Monte Carlo methods are needed to sample from unnormalized posterior. 
Also, we show in the next section \Cref{appendix:sir} a real world example when sampling is particularly costly and hence using CBQ is overall more efficient.

\subsection{Bayesian Sensitivity Analysis for Susceptible-Infectious-Recovered (SIR) Model }\label{appendix:sir}
\begin{figure}[t]
    \begin{minipage}{\textwidth}
    \centering
    \includegraphics[width=380pt]{figures/legend_finance.pdf}
    \end{minipage}
    
    \centering
    \begin{subfigure}{0.33\textwidth}
        \centering
        \hspace{-10pt}
        \includegraphics[width=\textwidth]{figures/SIR_15.pdf}
        \caption{RMSE with fixed $T=15$}
    \end{subfigure}%
    \hfill 
    \begin{subfigure}{0.33\textwidth}
        \centering
        \hspace{-10pt}
        \includegraphics[width=\textwidth]{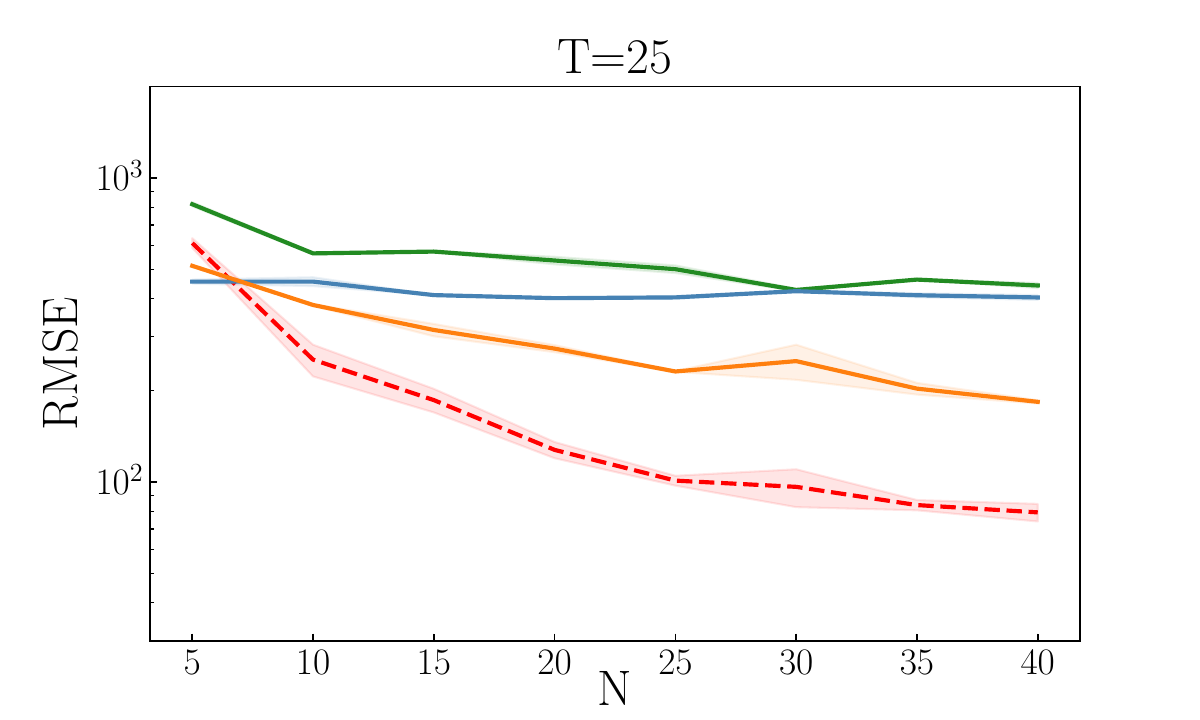}
        \caption{RMSE with fixed $T=25$.}
    \end{subfigure}%
    \hfill 
    \begin{subfigure}{0.33\textwidth}
        \centering
        \hspace{-10pt}
        \includegraphics[width=\textwidth]{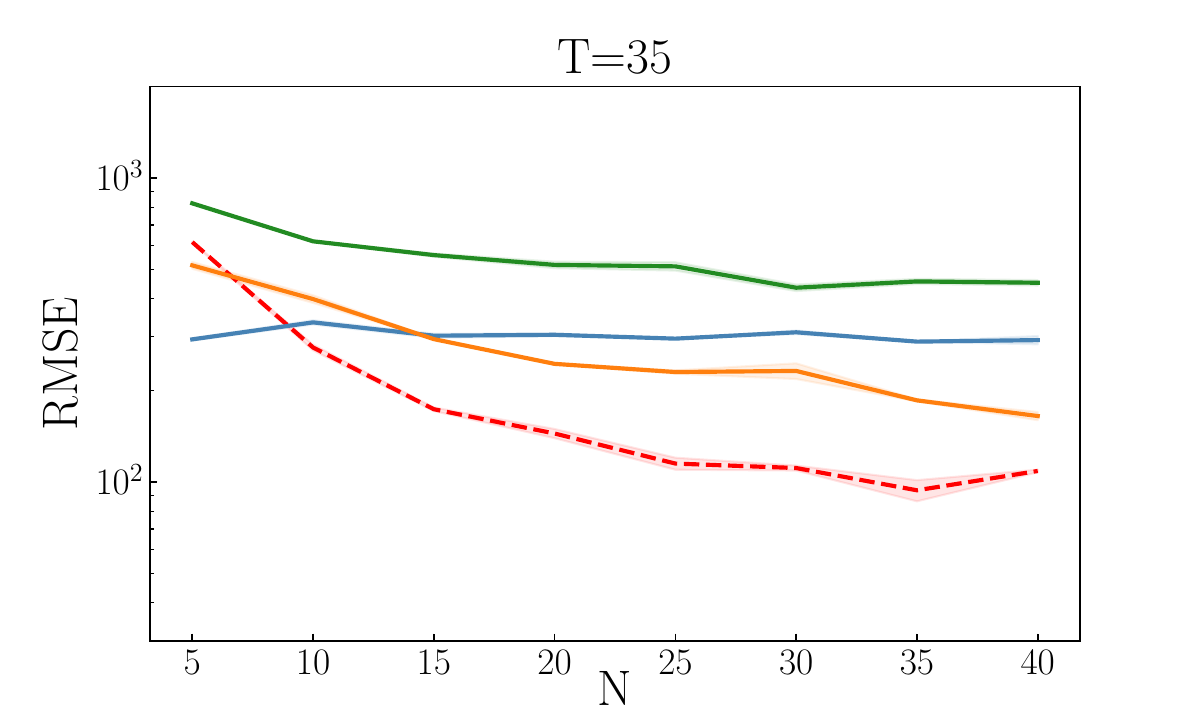}
        \caption{RMSE with fixed $T=35$.}
    \end{subfigure}
    \caption{\emph{Bayesian sensitivity analysis for SIR model.} $T=15, 25, 35$ and increasing $N$.}\label{appfig:sir}
\end{figure}
\subsubsection{Experimental Setting}
The SIR model is commonly used to simulate the dynamics of infectious diseases through a population~\cite{kermack1927sir}. 
It divides the population into three sections.
Susceptibles (S) represents people who are not infected but can be infected after getting contact with an infectious individual.
Infectious (I) represents people who are currently infected and can infect susceptible individuals.
Recovered (R) represents individuals who have been infected and then removed from the disease, either by recovering or dying. The dynamics are governed by a system of ordinary differential equations (ODE) as below.
\begin{align*}
    \begin{aligned}
\frac{\mathrm{d} S}{\mathrm{~d} r} &= -x S I, \quad
\frac{\mathrm{d} I}{\mathrm{~d} r} &= x S I-\gamma I, \quad
\frac{\mathrm{d} R}{\mathrm{~d} r} &= \gamma I
\end{aligned}
\end{align*}
with $x$ being the infection rate, $\gamma$ being the recovery rate and $r$ is the time. The solution to the SIR model would be a vector of $\left(N_I^r, N_S^r, N_R^r \right)$ representing the number of infectious, susceptibles and recovered at day $r$.

In this experiment, we assume that the recovery rate $\gamma$ is fixed and we place a Gamma prior distribution on $x$; i.e. $\Pb_\theta = \operatorname{Gamma}(\theta, \xi)$ where $\theta$ represents the initial belief of the infection rate deduced from the study of the virus in the laboratory at the beginning of the outbreak, and $\xi$ represents the amount of uncertainty on the initial belief. 
We fix the parameter $\xi=10$, the total population is set to be $10^6$ and the recovery rate $\gamma = 0.05$. 
The target of interest is the expected peak number of infected individuals under the prior distribution on $x$: 
\begin{align*}
    I(\theta) = \E_{x}\left[\max_r N_I^r(x) \mid \theta \right] = \int_{\calX} \max_r N^r_I(x) \Pb_\theta(dx)
\end{align*}
with the integrand $f(x) = \max_r N_I^r(x)$. We are interested in the sensitivity analysis of the shape parameter $\theta$ to the final estimate of the expected peak number of infected individuals.
The initial belief of the infection rate $\theta_{1:T}$ are sampled from the uniform distribution $\Qb = \operatorname{Unif}\left(2,9\right)$ and then $N$ number of $x^t_{1:N}$ are sampled from $\Pb_{\theta_t} = \operatorname{Gamma}(\theta_t, \xi)$. 
In this setting, sampling $x$ is very expensive as it necessarily involves solving the system of SIR ODEs, which can be very slow as the discretization step gets finer.
In the middle panel of \Cref{fig:finance_sir}, we have shown that obtaining one sample from SIR ODEs under discretization time step $\tau = 0.1$ takes around $3.0$s, whereas running the whole CBQ algorithm takes $1.0$s, not to mention that sampling from SIR ODEs need to be repeated $N \times T$ times. Therefore, using CBQ is ultimately more efficient overall within the same period of time.

For CBQ, we need to carefully choose two kernels $k_\Theta$ and $k_\calX$.
First we choose $k_\calX$, we use Mat\'ern-3/2 as the base kernel and then apply a Langevin Stein operator to both arguments of the base kernel to obtain $k_\calX$. 
The reason we use a Langevin Stein kernel is that Stein kernel gives an RKHS which is a subset on the Sobolev space with one order less smoothness than the base kernel, and since the smoothness of the integrand $f(x) = \max_r N_I^r(x)$ is unknown, using a Stein kernel enforces weaker prior information than Mat\'ern-3/2.
Furthermore, the kernel mean embedding of a Stein kernel $\mu(x)$ is a constant $c$ by construction as per the discussion in \Cref{appendix:practical_considerations}. 
The initial error is also a constant $c$ by construction.
Then we choose $k_\Theta$. Since $I(\theta)$ represents the peak number of infections so $I(\theta)$ is expected to be smooth and continuous, and hence we choose $k_\Theta$ as Mat\'ern-3/2 kernel. 
All hyperparameters in $k_\calX$ and $k_\Theta$ are selected according to \Cref{appendix:hyperparameter_selection}.
We use a MC estimator with $5000$ samples as the pseudo ground truth and evaluate the RMSE across all methods.

\subsubsection{Assumptions from \Cref{thm:convergence}} 
We would like to check whether the assumptions made in \Cref{thm:convergence} hold in this experiment.
\begin{itemize}
    \item A1: Although $\calX=\R^+$ is not a compact domain,$\Pb_\theta$ is a Gamma distribution so the probability mass outside a large compact subset of $\calX$ around the origin decays exponentially. $\Theta = \left(2, 9 \right)^d$ is a compact domain. A1 is approximately satisfied.
    \item A2: A2 is satisfied due to the sampling mechanism of $\theta_{1:T}$ and $\{x_{1:N}^t\}_{t=1}^T$.
    \item A3: $\Qb$ is a uniform distribution so its density $q$ is constant and hence upper bounded and strictly positive. $\Pb_\theta$ is a Gamma distribution so its density $p_\theta$ is strictly positive within a large compact subset of $\calX$ and has finite second moment. A3 is approximately satisfied.
    \item A4: $f(x) = \max_r N_I^r(x)$ is the maximum number of infections so $f(x)$ is not necessarily smooth. $I(\theta)$ represents the peak number of infections with varying initial estimate of the infection rate, so $I(\theta)$ is smooth and continuous with $s_I \leq 1$. 
    $k_\calX$ is Stein kernel with Matern-3/2 kernel as the base, so the corresponding RKHS will have functions which are rough (i.e. of smoothness $1/2$) but is only a subset of a Sobolev space. In addition, $k_\Theta$ is Matern-3/2 kernel so $s_\Theta = \frac{3}{2} + \frac{1}{2} = 2$. It is therefore unclear if A4 is satisfied.
    \item A5: $\lambda_\calX$ is picked to be $0$ and $\lambda_\Theta$ is found via grid search among $\{0.01, 0.1, 1.0\}$. A5 is satisfied.
\end{itemize}

\subsubsection{Additional Experimental Results}
We report more results in \Cref{appfig:sir} with fixed $T=15, 25, 35$ and increasing $N$, to showcase that CBQ consistently exhibits smaller RMSE than baseline methods. The conclusions that we draw from the main text also hold for different values of N and T for this experiment.


\subsection{Option Pricing in Mathematical Finance}\label{appendix:black_scholes}

\begin{figure}[t]
    \begin{minipage}{\textwidth}
    \centering
    \includegraphics[width=380pt]{figures/legend_finance.pdf}
    \end{minipage}
    
    \centering
    \begin{subfigure}{0.33\textwidth}
        \centering
        \hspace{-10pt}
        \includegraphics[width=\textwidth]{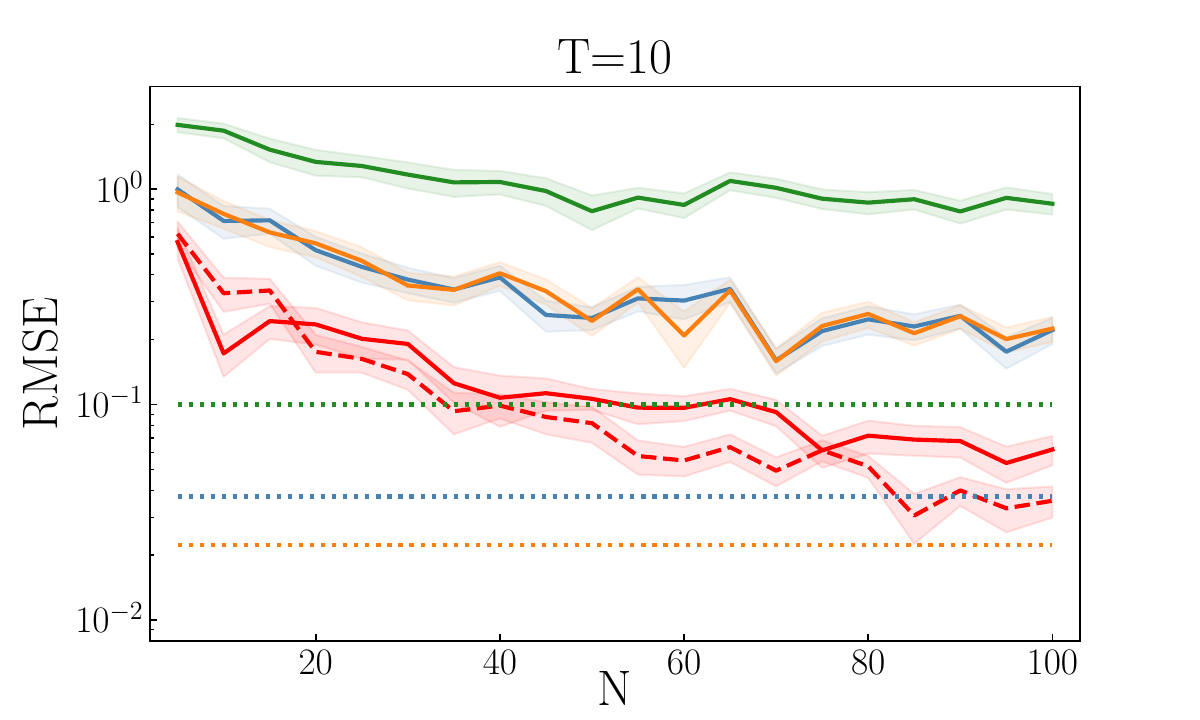}
        \caption{RMSE with fixed $T=10$}
    \end{subfigure}%
    \hfill 
    \begin{subfigure}{0.33\textwidth}
        \centering
        \hspace{-10pt}
        \includegraphics[width=\textwidth]{figures/finance_T_20.pdf}
        \caption{RMSE with fixed $T=20$.}
    \end{subfigure}%
    \hfill 
    \begin{subfigure}{0.33\textwidth}
        \centering
        \hspace{-10pt}
        \includegraphics[width=\textwidth]{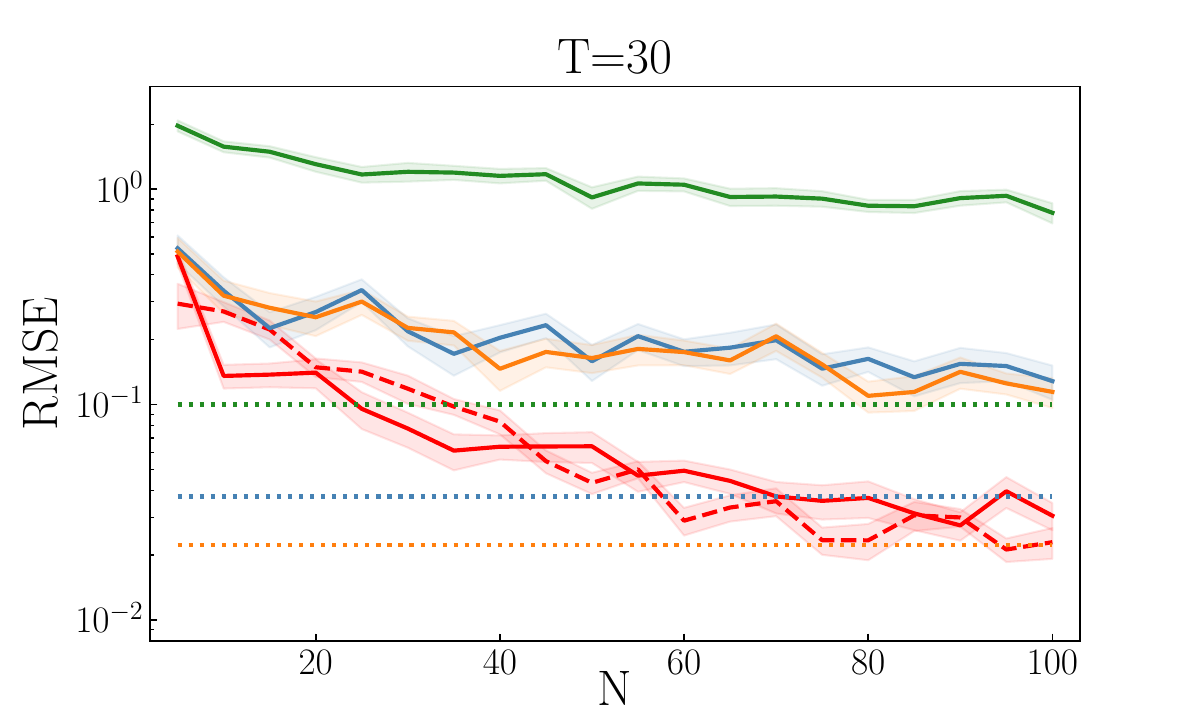}
        \caption{RMSE with fixed $T=30$.}
    \end{subfigure}
    \caption{\emph{Option pricing in mathematical finance.} $T=10, 20, 30$ and increasing $N$.}\label{appfig:finance}
\end{figure}

\subsubsection{Experimental Setting}
In this experiment, we consider specifically an asset whose price $S({\tau})$ at time $\tau$ follows the Black-Scholes formula $S(\tau) = S_0 \exp \left(\sigma W(\tau) - \sigma^2 \tau/2 \right)$ for $\tau \geq 0$, where $\sigma$ is the underlying volatility, $S_0$ is the initial price and $W$ is the standard Brownian motion.
The financial derivative we are interested in is a butterfly call option whose payoff at time $\tau$ can be expressed as $\psi(S({\tau}))=\max (S(\tau)-K_1, 0) + \max (S(\tau)-K_2, 0) - 2\max (S(\tau) - (K_1+K_2)/2, 0)$.

In addition to the expected payoff, insurance companies are interested in computing the expected loss of their portfolios if a shock would occur in the economy.
We follow the setting in \cite{alfonsi2021multilevel, alfonsi2022many} assuming that a shock occur at time $\eta$, at which time the option price is $S(\eta)=\theta$, and this shock multiplies the option price by $1 + s$. The option price at maturity time $\zeta$ is denoted as $S(\zeta) = x$. The expected loss caused by the shock can be expressed as 
\begin{align*}
    \mathcal{L} = \E [\max (I(\theta), 0 )], \text{   } I(\theta) = \int_0^\infty \psi(x)-\psi \left((1 + s) x \right) \Pb_\theta(dx)
\end{align*}
So the integrand is $f(x) = \psi(x)-\psi((1+s)x)$.

Following the setting in \cite{alfonsi2021multilevel, alfonsi2022many}, we consider the initial price $S_0 = 100$, the volatility $\sigma = 0.3$, the strikes $K_1 = 50, K_2 = 150$, the option maturity $\zeta=2$ and the shock happens at $\eta=1$ with strength $s = 0.2$. 
The option price at which the shock occurs are $\theta_{1:T}$ sampled from the log normal distribution deduced from the Black-Scholes formula $\theta_{1:T} \sim \Qb = \operatorname{Lognormal}( \log S_0 - \frac{\sigma^2}{2} \eta, \sigma^2 \eta)$. 
Then $x^t_{1:N}$ are sampled from another log normal distribution also deduced from the Black-Scholes formula $x^t_{1:N} \sim \Pb_{\theta_t} = \operatorname{Lognormal}( \log \theta_t - \frac{\sigma^2}{2} (\zeta - \eta), \sigma^2 (\zeta - \eta))$. 

For CBQ, we need to carefully choose two kernels $k_\calX$ and $k_\Theta$. First we choose the kernel $k_\calX$ to be a log-Gaussian kernel for the purpose that the log-Gaussian kernel mean embedding has a closed form under log-normal distribution $\Pb_\theta = \operatorname{Lognormal}(\bar{m}, \bar{\sigma}^2)$ with $\bar{m} = \log \theta - \frac{\sigma^2}{2}(\zeta - \eta)$ and  $\bar{\sigma}^2 = \sigma^2 (\zeta - \eta)$. 
The log Gaussian kernel is defined as $k_\calX(x, x') = A_\calX \exp(-\frac{1}{2 l_\calX^2} (\log x - \log x')^2)$
and the kernel mean embedding has the form
\begin{align*}
    \mu_\theta(x) = \frac{A_\calX}{\sqrt{1 + \frac{\bar{\sigma}^2}{l_\calX^2}}} \left. \exp \left(-\frac{\bar{m}^2 + (\log x)^2 }{2(\bar{\sigma}^2 + l_\calX^2)}\right) x^{\frac{\bar{m}}{\bar{\sigma}^2 + l_\calX^2}}  \right.
\end{align*}
The initial error, which is the integral of kernel mean $\mu_\theta(x)$ does not have a closed form expression, so we use the empirical average as an approximation. Then, we choose the kernel $k_\Theta$ to be a Mat\'ern-3/2 kernel.

For this experiment, we also implement CBQ with Langevin Stein reproducing kernel. We use Mat\'ern-3/2 as the base kernel and then apply the Langevin Stein operator to both arguments of the base kernel to obtain $k_\calX$. 
The reason we use a Stein kernel is that Stein kernels have an RKHS whose functions have one order less smoothness than the base kernel, and since the integrand has very low smoothness (due to the maximum function), we do not want to use an overly smooth kernel. 
The kernel mean embedding of a Stein kernel is a
constant $c$ by construction as per the discussion in \Cref{appendix:practical_considerations}.
The kernel $k_\Theta$ is selected as Mat\'ern-3/2 kernel.
All hyperparameters in $k_\calX$ and $k_\Theta$ for CBQ and hyperparameters for baseline methods are selected according to \Cref{appendix:hyperparameter_selection}.

\subsubsection{Assumptions from \Cref{thm:convergence}} 
We would like to check whether the assumptions made in \Cref{thm:convergence} hold in this experiment.
\begin{itemize}
    \item A1: Although $\calX=\R^+$ is not a compact domain, $\Pb_\theta$ is a lognormal distribution so the probability mass outside a large compact subset of $\calX$ decays super exponentially. A similar argument can be made for $\Theta$ as well. A1 is therefore approximately satisfied.
    \item A2: A2 is satisfied due to the sampling mechanism of $\theta_{1:T}$ and $\{x_{1:N}^t\}_{t=1}^T$.
    \item A3: $\Qb$ is a lognormal distribution so its density $q$ is upper bounded and strictly positive within a large compact subset of $\Theta$. $\Pb_\theta$ is also a lognormal distribution so its density $p_\theta$ is strictly positive within a large compact subset of $\calX$ and has finite second moment. A3 is approximately satisfied.
    \item A4: $f(x)$ is a combination of piecewise linear functions so $s_f = 1$ and $I(\theta)$ is infinitely times differentiable so $s_f = \infty$. 
    When $k_\calX$ is Stein kernel with Matern-3/2 kernel as the base, the functions in the corresponding RKHS have smoothness $1/2$, whereas when $k_\calX$ is the log Gaussian kernel, the functions are infinitely differentiable. Neither of these choices satisfy the assumption, although Stein kernel contain many (but not necessarily all) function of smoothness $1/2$. $k_\Theta$ is Matern-3/2 kernel so $s_\Theta = \frac{3}{2} + \frac{1}{2} = 2$. It is therefore unclear if A4 is satisfied.
    \item A5: $\lambda_\calX$ is picked to be $0$ and $\lambda_\Theta$ is found via grid search among $\{0.01, 0.1, 1.0\}$. A5 is satisfied.
\end{itemize}

\subsubsection{More Experimental Results}
We report more results in \Cref{appfig:finance} with fixed $T=10, 20, 30$ and increasing $N$, to showcase that CBQ consistently exhibits smaller RMSE than baseline methods. The conclusions that we draw from the main text also
hold for different values of $N$ and $T$ for this experiment.
The performance of CBQ is similar between $k_\calX$ being Stein kernel and $k_\calX$ being log Gaussian kernel. It would be interesting to further investigate the performance of CBQ in estimating the future price of other financial derivatives, and we leave it for future work.


\subsection{Uncertainty Decision Making in Health Economics}\label{appendix:decision}
\subsubsection{Experimental Settings}
In the medical world, it is important to compare the cost and the relative advantages of conducting extra medical experiments. 
The expected value of partial perfect information (EVPPI) quantifies the expected gain from conducting extra experiments to obtain precise knowledge of some unknown variables \citep{brennan2007calculating}:
\begin{align*}
    \text{EVPPI} = \E \Bigl[\max_c I_c(\theta) \Bigr] - \max_c \E \Bigl[I_c(\theta) \Bigr], \text{   } I_c(\theta) = \int_{\calX} f_c(x, \theta) \Pb_\theta(dx)
\end{align*}
where $c \in \mathcal{C}$ is a set of potential treatments and $f_c$ measures the potential outcome of treatment $c$. Our method is applicable for estimating the conditional expectation $I_c(\theta)$ of the first term. 

We adopt the same experimental setup as delineated in \cite{Giles2019}, wherein $x$ and $\theta$ have a joint 19-dimensional Gaussian distribution, meaning that $\Pb_\theta$ is a Gaussian distribution. 
The specific meanings of all $x$ and $\theta$ are outlined in \Cref{tab:mytable}.
All these variables are independent except that $\theta_1, \theta_2, x_6, x_{14}$ are pairwise correlated with a correlation coefficient $0.6$.
The observations $\theta_{1:T}$ are sampled from the marginal Gaussian distribution $\Qb$ and then $N$ observations of $x^t_{1:N}$ are sampled from $\Pb_{\theta_t}$.

We are interested in a binary decision-making problem ($\calC = \{1, 2\}$) with $f_1(x, \theta)=10^4 (\theta_1 x_5 x_6 + x_7 x_8 x_{9})-(x_1 + x_2 x_3 x_4)$ and $f_2(x, \theta) = 10^4 (\theta_2 x_{13} x_{14} + x_{15} x_{16} x_{17})-(x_{10} + x_{11} x_{12} x_4)$. 
In computing EVPPI, we estimate $I_c(\theta)$ with CBQ and baselines, and then use standard MC for the rest of the expectations.
We draw $10^6$ samples from the joint distribution to generate a pseudo ground truth, and evaluate the RMSE across different methods. 
Note that IS is no longer applicable here because $f_c$ now depends on both $x$ and $\theta$, so we only comparing CBQ against KLSMC and LSMC.

For CBQ, we need to carefully choose two kernels. First, we take $k_\calX$ to be a Mat\'ern-3/2 to ensure that the kernel mean embedding under a Gaussian distribution $\Pb_\theta = \calN(\tilde{\mu}, \tilde{\Sigma})$ has a closed form if we use the 'inverse transform trick' as outlined in \Cref{appendix:practical_considerations}. 
Specifically speaking, we initially sample $u$ from $\calN(0, \Id_d)$, then calculate $x = \tilde{m} + L^\top u$ where $L$ is the lower triangular matrix derived from the Cholesky decomposition of the covariance matrix $\tilde{\Sigma}$. 
The integral now becomes
\begin{align}\label{appeq:transform}
    I_c(\theta) = \int_{\R^d} f(x)\calN(x; \tilde{m},\tilde{\Sigma}) dx = \int_{\R^d} f(\tilde{m} + L^\top u) \calN(u; 0, \Id_d) du
\end{align}
The closed form expression of kernel mean embedding for a Mat\'ern-3/2 kernel and isotropic Gaussian can be found in the Appendix S.3 of \cite{ming2021linked}.
Then we pick $k_\Theta$. 
We know there is a high chance that $I_c(\theta)$ is infinitely times differentiable, but we opt for Mat\'ern-3/2 kernel to encode a more conservative prior information on the smoothness of $I_c(\theta)$ because we do not have a closed form of it.
All hyperparameters in $k_\calX$ and $k_\Theta$ are selected according to \Cref{appendix:hyperparameter_selection}.

\begin{figure}[t]
    \begin{minipage}{\textwidth}
    \centering
    \includegraphics[width=250pt]{figures/legend.pdf}
    \vspace{-10pt}
    \end{minipage}
    
    \centering
    \begin{subfigure}{0.33\textwidth}
        \centering
        \hspace{-10pt}
        \includegraphics[width=\textwidth]{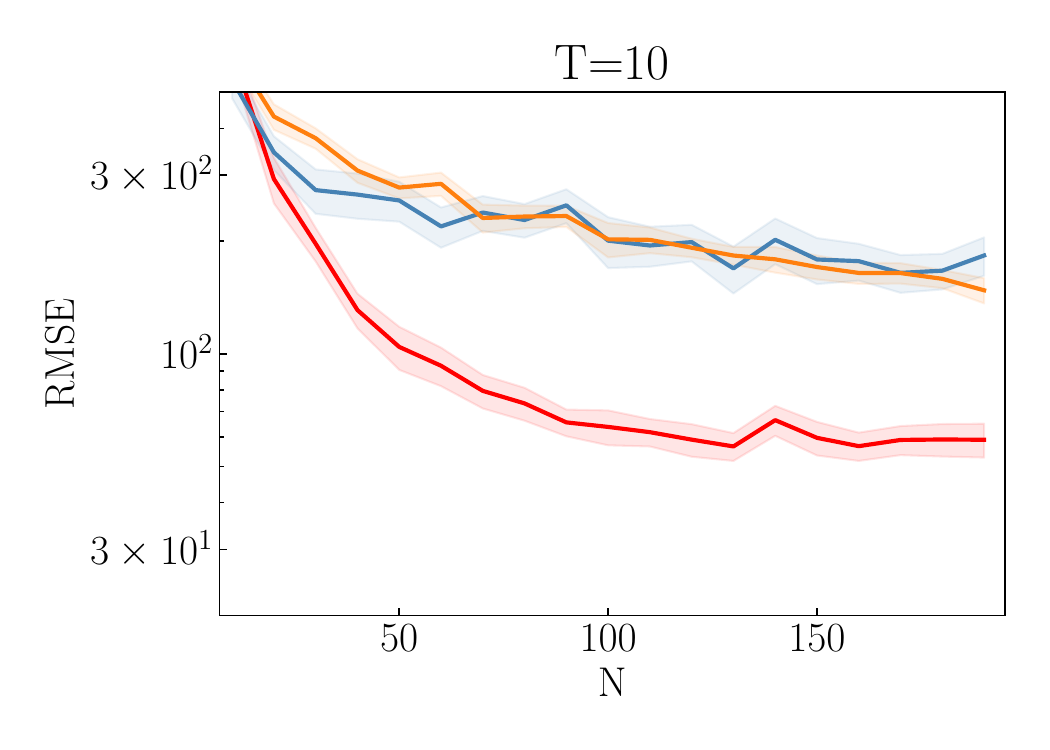}
        \caption{RMSE with fixed $T=10$}
    \end{subfigure}%
    \hfill 
    \begin{subfigure}{0.33\textwidth}
        \centering
        \hspace{-10pt}
        \includegraphics[width=\textwidth]{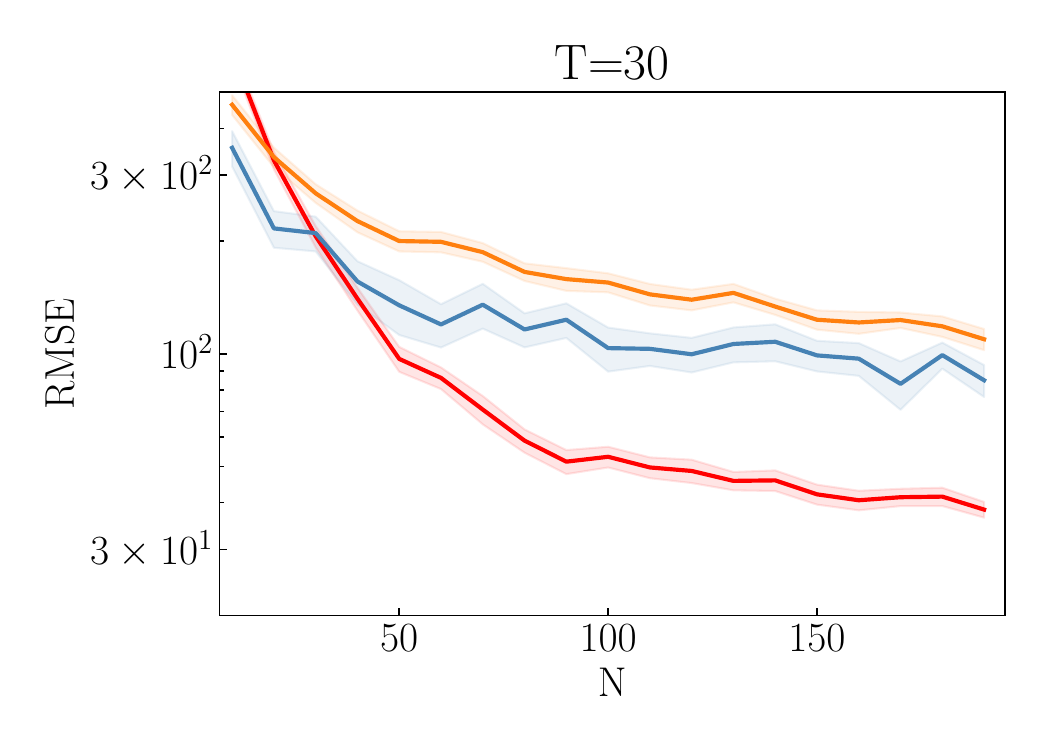}
        \caption{RMSE with fixed $T=30$.}
    \end{subfigure}%
    \hfill 
    \begin{subfigure}{0.33\textwidth}
        \centering
        \hspace{-10pt}
        \includegraphics[width=\textwidth]{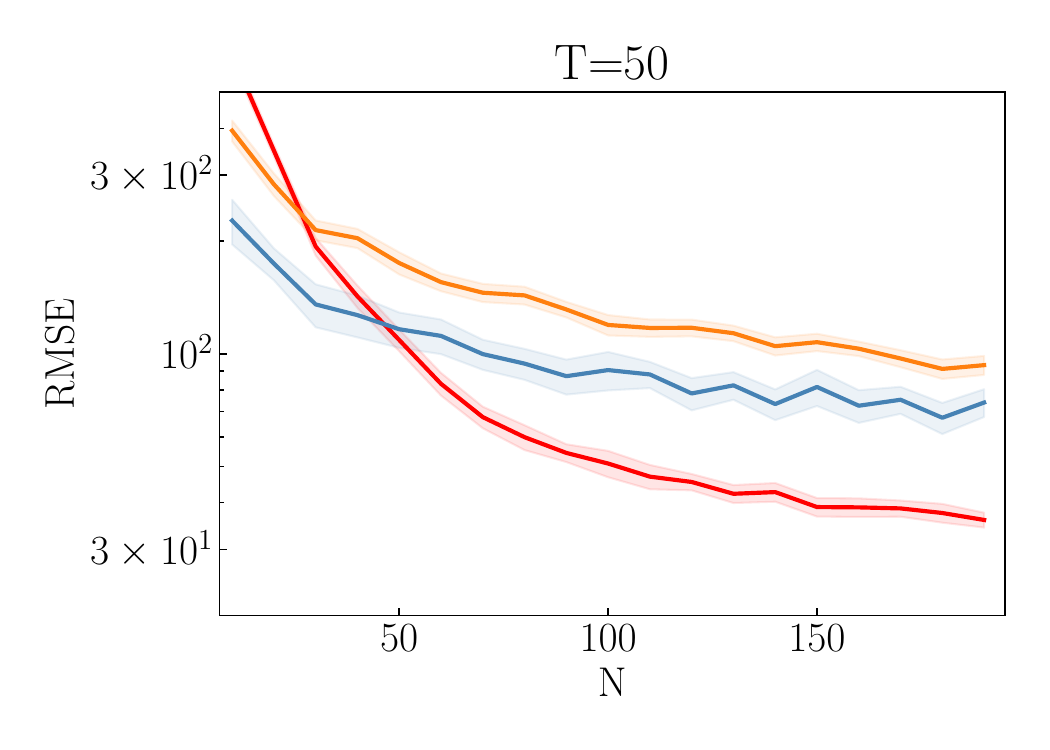}
        \caption{RMSE with fixed $T=50$.}
    \end{subfigure}
    \caption{\emph{Uncertainty decision making in health economics.} $T=10, 30, 50$ and increasing $N$. }\label{appfig:decision}
\end{figure}

\begin{table}[t]
\centering
\begin{tabular}{
>{\centering\arraybackslash}p{1.5cm}
>{\centering\arraybackslash}p{1cm}
>{\centering\arraybackslash}p{1cm}
>{\centering\arraybackslash}p{5cm}}
\toprule
Variables & Mean & Std & Meaning \\
\midrule
$X_1$ & 1000 & 1.0 & Cost of treatment \\
$X_2$ & 0.1 & 0.02 & Probability of admissions \\
$X_3$ & 5.2 & 1.0 & Days of hospital \\
$X_4$ & 400 & 200 & Cost per day \\
$X_5$ & 0.3 & 0.1 & Utility change if response \\
$X_6$ & 3.0 & 0.5 & Duration of response \\
$X_7$ & 0.25 & 0.1 & Probability of side effects \\
$X_8$ & -0.1 & 0.02 & Change in utility if side effect \\
$X_{9}$ & 0.5 & 0.2 & Duration of side effects \\
$X_{10}$ & 1500 & 1.0 & Cost of treatment \\
$X_{11}$ & 0.08 & 0.02 & Probability of admissions \\
$X_{12}$ & 6.1 & 1.0 & Days of hospital \\
$X_{13}$ & 0.3 & 0.05 & Utility change if response \\
$X_{14}$ & 3.0 & 1.0 & Duration of response \\
$X_{15}$ & 0.2 & 0.05 & Probability of side effects \\
$X_{16}$ & -0.1 & 0.02 & Change in utility if side effect \\
$X_{17}$ & 0.5 & 0.2 & Duration of side effects \\
$\theta_1$ & 0.7 & 0.1 & Probability of responding \\
$\theta_2$ & 0.8 & 0.1 & Probability of responding \\
\bottomrule

\end{tabular}
\vspace{5pt}
\caption{Variables in the health economics experiment.}
\label{tab:mytable}
\end{table}

\subsubsection{Assumptions from \Cref{thm:convergence}} 
We would like to check whether the assumptions made in \Cref{thm:convergence} hold in this experiment.
\begin{itemize}
    \item A1: 
    Although $\calX=\R$ is not a compact domain, but $\Pb_\theta$ is a Gaussian distribution so the probability mass outside a large compact subset of $\calX$ decays exponentially. Similarly, $\Theta = \R$ is not a compact domain, but $\Qb$ is a Gaussian distribution so the probability mass outside a large compact subset of $\Theta$ decays exponentially. A1 is approximately satisfied.
    \item A2: A2 is satisfied due to the sampling mechanism of $\theta_{1:T}$ and $\{x_{1:N}^t\}_{t=1}^T$.
    \item A3: $\Qb$ is also a Gaussian distribution so its density $q$ is upper bounded and strictly positive on a compact and large domain. $\Pb_\theta$ is a Gaussian distribution so its density $p_\theta$ is strictly positive on a compact and large domain with finite second moment. A3 is approximately satisfied.
    \item A4: Both the integrand $f$ and the conditional expectation $I_c(\theta)$ are infinitely times differentiable, so $s_f = s_I = \infty$. On the other hand, due to the choice of Mat\'ern-3/2 kernels, $s_{\Theta}=3/2+1/2=2$ and $s_{\calX}=3/2+9/2=6$. A4 is therefore satisfied.
    \item A5: $\lambda_\calX$ is picked to be $0$ and $\lambda_\Theta$ is found via grid search among $\{0.01, 0.1, 1.0\}$. A5 is satisfied.
\end{itemize}

\subsubsection{Additional Experimental Results}
We report more results in \Cref{appfig:decision} with fixed $T = 10, 30, 50$ and increasing N, to showcase that CBQ consistently exhibits smaller RMSE than baseline methods.
The conclusions that we draw from the main text also hold for different values of $N$ and $T$ for this experiment.

\begin{figure}[t]\label{appfig:mobq}
    \centering
    \begin{minipage}{1.0\textwidth}
    \centering
    \includegraphics[width=350pt]{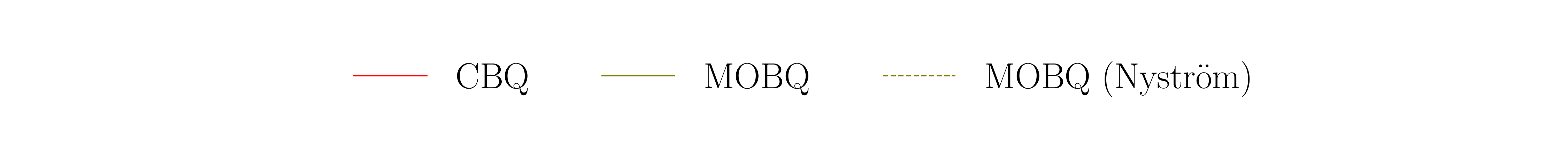}
    \end{minipage}
    \vspace{-10pt}
    
    \begin{subfigure}{0.32\textwidth}
        \centering
        \includegraphics[width=\linewidth]{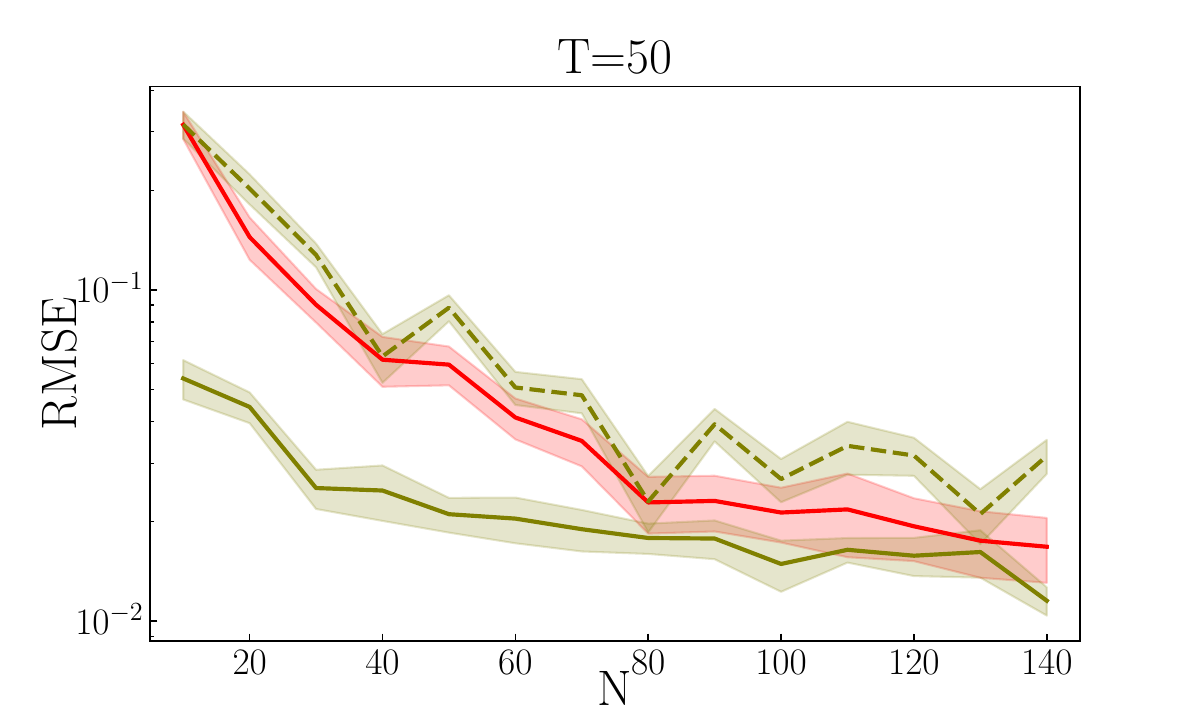}
    \end{subfigure}
    \hfill
    \begin{subfigure}{0.32\textwidth}
        \centering
        \includegraphics[width=\linewidth]{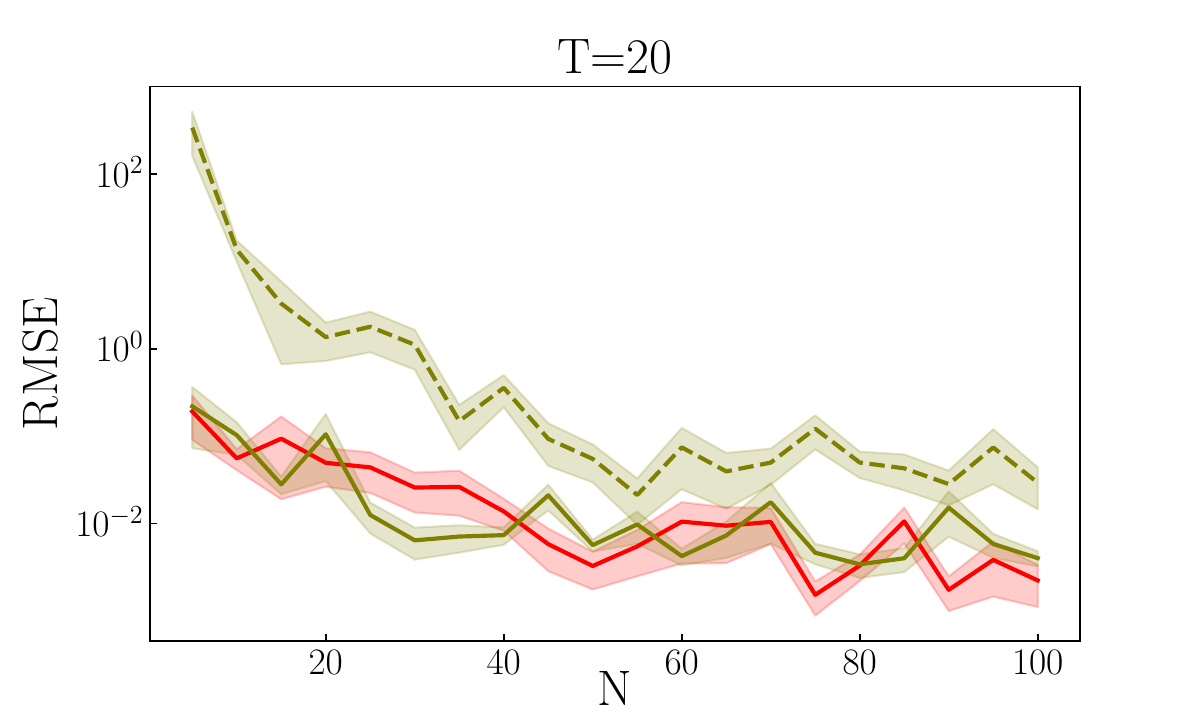}
    \end{subfigure}
    \hfill
    \begin{subfigure}{0.32\textwidth}
        \centering
        \includegraphics[width=\linewidth]{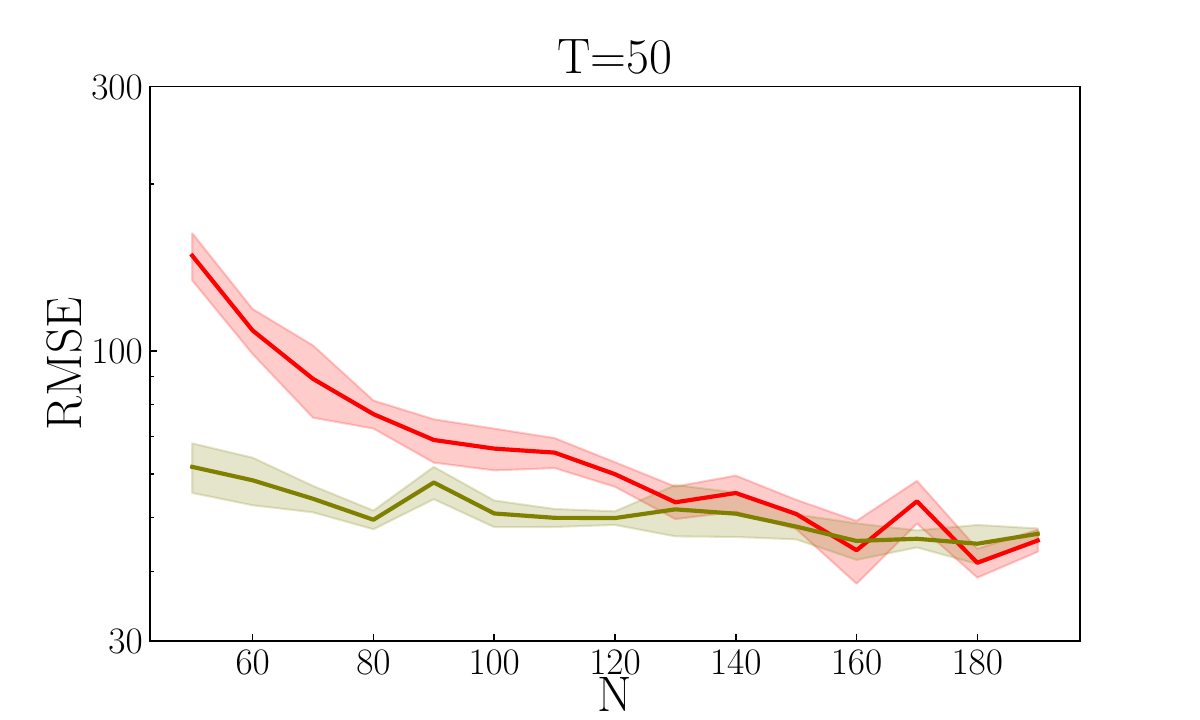}
    \end{subfigure}
    \\
    \begin{subfigure}{0.32\textwidth}
        \centering
        \includegraphics[width=\linewidth]{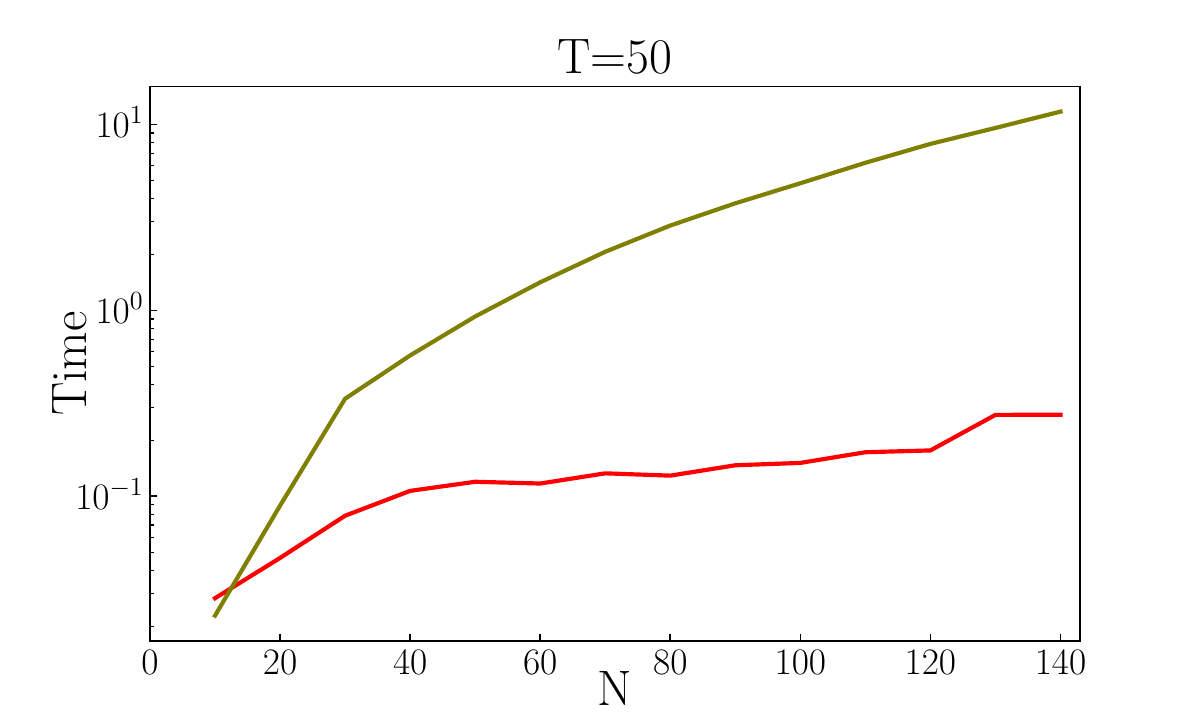}
        \caption{Bayesian sensitivity analysis for linear models.}
        \label{appfig:mobq_bayes_sensitivity}
    \end{subfigure}
    \hfill
    \begin{subfigure}{0.32\textwidth}
        \centering
        \includegraphics[width=\linewidth]{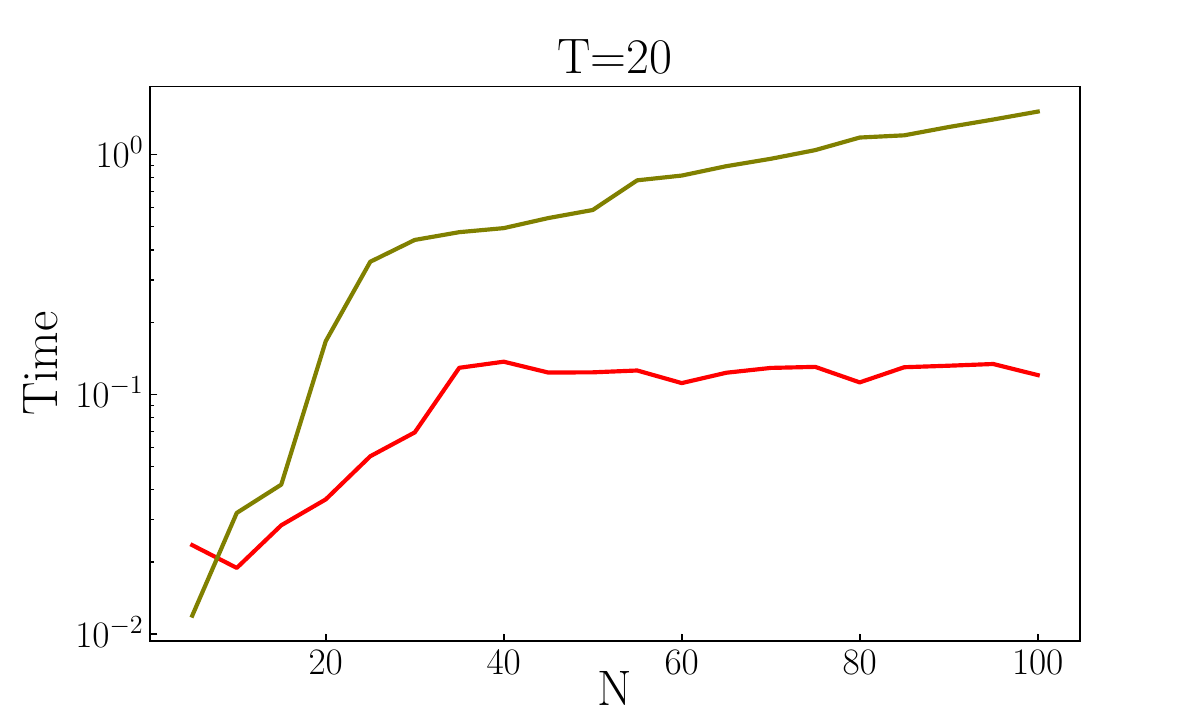}
        \caption{Option pricing in mathematical finance.}
        \label{appfig:mobq_finance}
    \end{subfigure}
    \hfill
        \begin{subfigure}{0.32\textwidth}
        \centering
        \includegraphics[width=\linewidth]{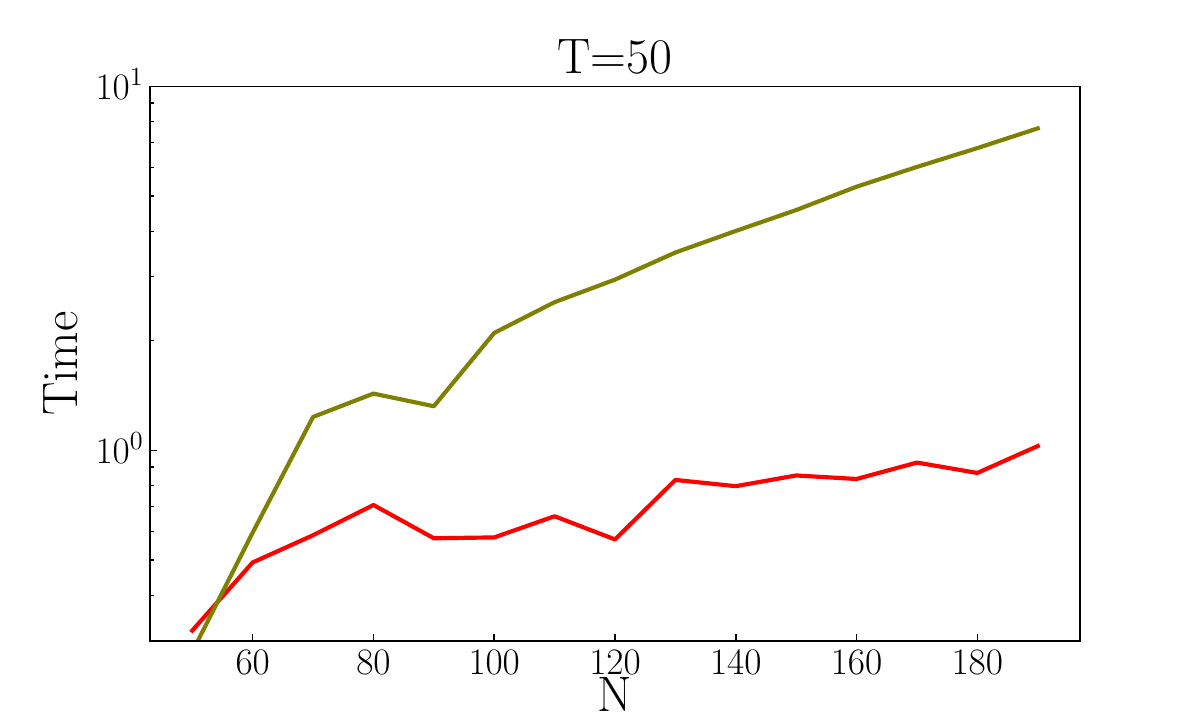}
        \caption{Uncertainty decision making in health economics.}
        \label{appfig:mobq_decision}
    \end{subfigure}
    \hfill
    \caption{Comparison of CBQ and MOBQ in terms of RMSE (first row) and computational time (second row). \textbf{Left (a):} Bayesian sensitivity analysis for linear models. \textbf{Middle (b):} Option pricing in mathematical finance. \textbf{Right (c):} Uncertainty decision making in health economics.}
\end{figure}

\subsection{Comparison of Conditional Bayesian Quadrature and Multi-Output Bayesian Quadrature}\label{appendix:cbq_mobq}

In Section \ref{sec:cbq} in the main text, we mentioned a comparison of CBQ and multi-output Bayesian quadrature~\cite{xi2018bayesian} (MOBQ) in terms of their computational complexity. 
For $T$ parameter values $\theta_1, \cdots, \theta_T$ and $N$ samples from each probability distribution $\mathbb{P}_{\theta_1}, \ldots, \mathbb{P}_{\theta_T}$, the computational cost is $\calO(TN^3 + T^3)$ for CBQ and $\calO(N^3T^3)$ for MOBQ. 
We now give a more thorough comparison of CBQ and MOBQ in this section. 

When the integrand $f$ only depends on $x$ (Bayesian sensitivity analysis for linear models, option pricing in mathematical finance), MOBQ only requires one kernel $k_\calX$. 
\begin{align*}
    I_{\mathrm{MOBQ}}(\theta^\ast) = \left(\int_\calX k_\calX(x, x_{1:NT}) \Pb_{\theta^\ast}(dx) \right) \Big(k_\calX(x_{1:NT}, x_{1:NT}) + \lambda_\calX \Id_{NT} \Big)^{-1} f(x_{1:NT})
\end{align*}
where $x_{1:NT} \in \R^{NT}$ is a concatenation of $x_{1:N}^1, \cdots, x_{1:N}^T$.
When the integrand $f$ depends on both $x$ and $\theta$ (uncertainty decision making in health economics), MOBQ requires two kernels $k_\calX$ and $k_\Theta$.
\begin{align*}
\begin{aligned}
    I_{\mathrm{MOBQ}}(\theta^\ast) &= \Big( \int_\calX k_\calX(x, x_{1:NT})    \odot k_\Theta(\theta^\ast, \theta_{1:NT}) \Pb_{\theta^\ast}(dx)  \Big) \\ &
    \Big(k_\calX(x_{1:NT}, x_{1:NT}) \odot k_\Theta(\theta_{1:NT}, \theta_{1:NT})  + \lambda_\calX \Id_{NT} \Big)^{-1} f(x_{1:NT})
\end{aligned}
\end{align*}
where $\odot$ denotes element-wise product, and $\theta_{1:NT} = \left[\theta_1, \cdots, \theta_1, \cdots, \theta_T, \cdots, \theta_T \right] \in \R^{NT}$.
From the above two equations, we can see that the computation cost of $\calO(N^3T^3)$ mainly comes from the inversion of a $NT \times NT$ kernel matrix.
All the MOBQ hyperparameters in $k_\calX$ and $k_\Theta$ are selected by empirical Bayes in the same way as CBQ outlined in \Cref{appendix:hyperparameter_selection}.
It's crucial to note that the MOBQ computational cost is significantly higher for Stein reproducing kernel during hyperparameter selection (an approach analogous to the ``vector-valued control variates'' of \cite{Sun2021}), as evaluating the log marginal likelihood at every iteration would require the inversion of a $NT \times NT$ matrix.
Therefore, we do not include the experiment of Bayesian sensitivity analysis for the SIR model in this section.
All the hyperparameters for CBQ are reused as in \Cref{appendix:experiments}.

For Bayesian sensitivity analysis in linear models, the integrand is $f(x) = x^\top x$, the dimension is fixed $d=2$ and $T=50$.
In \Cref{appfig:mobq_bayes_sensitivity}, we can see that MOBQ indeed achieves lower RMSE at the beginning, but CBQ catches up when $N$ grows higher.
For option pricing in mathematical finance, we only compare MOBQ and CBQ when $k_\calX$ is the log Gaussian kernel and $T=20$.
For uncertainty decision making in health economics, we compare MOBQ and CBQ when $T=50$.
In \Cref{appfig:mobq_finance} and \Cref{appfig:mobq_decision}, we can see that CBQ and MOBQ achieves similar performances in terms of RMSE.
Additionally, in the second row of \Cref{appfig:mobq}, we compare the computational cost of MOBQ and CBQ, where we can see that the computational time of MOBQ is much larger than CBQ as $N$ grows across all settings, due to the complexity of $\calO(N^3T^3)$ for MOBQ.

Additionally, as the main computational bottleneck of MOBQ is the inversion of the kernel matrix, so it would be interesting to see if MOBQ combined with scalable GP methods can reduce the computational time while still preserving the same level of accuracy. 
The scalable approximation method used here is Nyström approximation~\cite{williams2000using}.
We report the performance of MOBQ (Nyström) in both \Cref{appfig:mobq_bayes_sensitivity} and \Cref{appfig:mobq_finance}, and we can see that MOBQ (Nyström) performs worse than CBQ in terms of RMSE.
The reason of worse performance of MOBQ (Nyström) is that the use of scalable GP methods would introduce an extra layer of approximation that slows down the convergence rate.
Additionally, most scalable GP methods are used in the “regression” setting, while quadrature methods like BQ or CBQ belong to the “interpolation" setting~\cite{kanagawa2018gaussian}, so the quadrature problem will be more sensitive to the approximation error introduced.

\subsection{Quasi Monte Carlo}\label{appendix:QMC}

Quasi Monte Carlo (QMC) is another line of research on improving the precision of approximating intractable integrals. 
While quadrature methods like BQ and CBQ aim at finding a smart way to combine the function values, QMC aims to find samples that can more uniformly cover the integration domain than random sampling~\citep{niu2023discrepancy, hickernell1998generalized, gerber2015sequential}. 
In the development of CBQ, we don't make any assumptions about the sampling of observations; specifically, we don't mandate i.i.d sampling. 
Therefore, it would be interesting to see whether combining quadrature algorithms with QMC could further improve the accuracy for estimating conditional expectation.

For a fair comparison in the experiment of Bayesian sensitivity analysis for linear models, we implement QMC sampling for all methods including CBQ and baseline methods. 
The samples $x_{1:N}^t$ are generated from a Sobol sequence which is a low-discrepancy sequence commonly used in QMC to cover the multidimensional space more uniformly than random sequences.
We follow the technique introduced in randomized QMC~\cite{lemieux2004randomized} to shift the Sobol sequence by a random amount.

It can be observed in \Cref{appfig:qmc} that replacing random sampling with QMC significantly enhances the performance of baseline methods, such as LSMC and KLSMC, while subtly improves the performance of CBQ. The limited degree of improvement seen in CBQ with QMC sampling can be attributed to the fact that CBQ already yields a remarkably low RMSE. Consequently, the margin of improvement offered by QMC sampling is not as evident in CBQ as in the baseline methods. We have only studied the effect of combining QMC and CBQ in the experiment of Bayesian sensitivity analysis in linear models. It would be interesting to see if combining QMC and CBQ would result in higher accuracy in other settings, and we leave it for future work. 

\begin{figure}[t]
    \begin{minipage}{1.0\textwidth}
    \includegraphics[width=250pt]{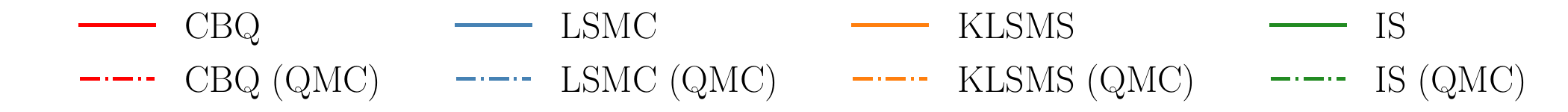}
    \end{minipage}
    
    \begin{subfigure}{0.30\textwidth}
        \centering
        \includegraphics[width=\linewidth]{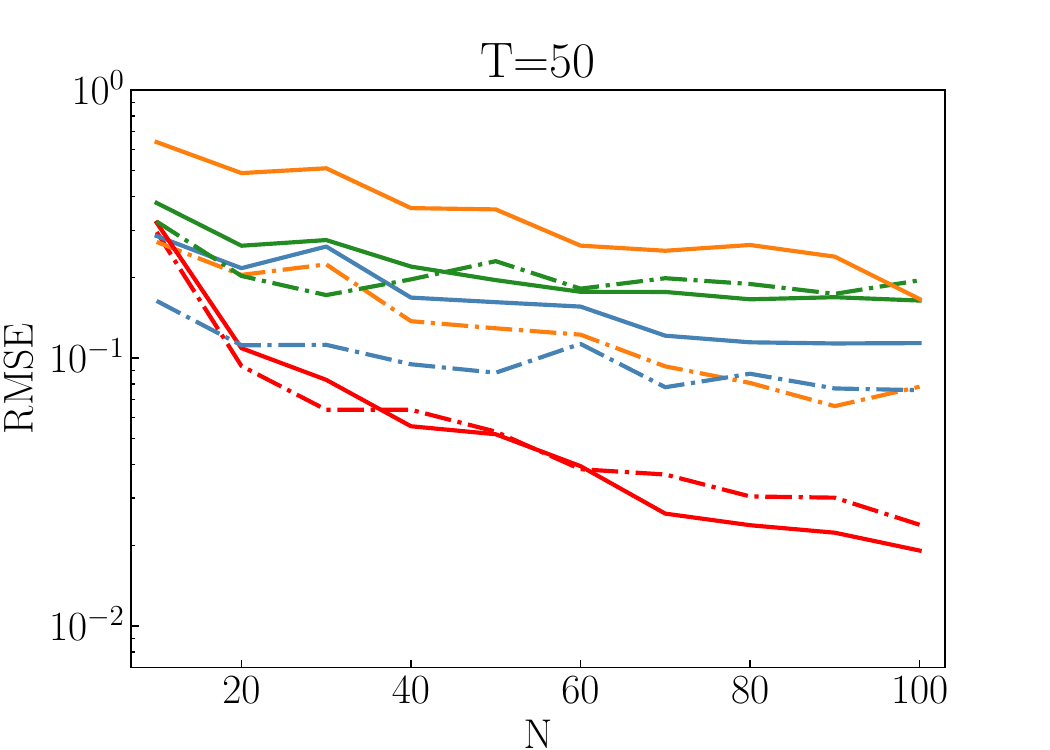}
        \caption{Quasi Monte Carlo}
        \label{appfig:qmc}
    \end{subfigure}
t    \begin{subfigure}{0.30\textwidth}
        \centering
        \includegraphics[width=\linewidth]{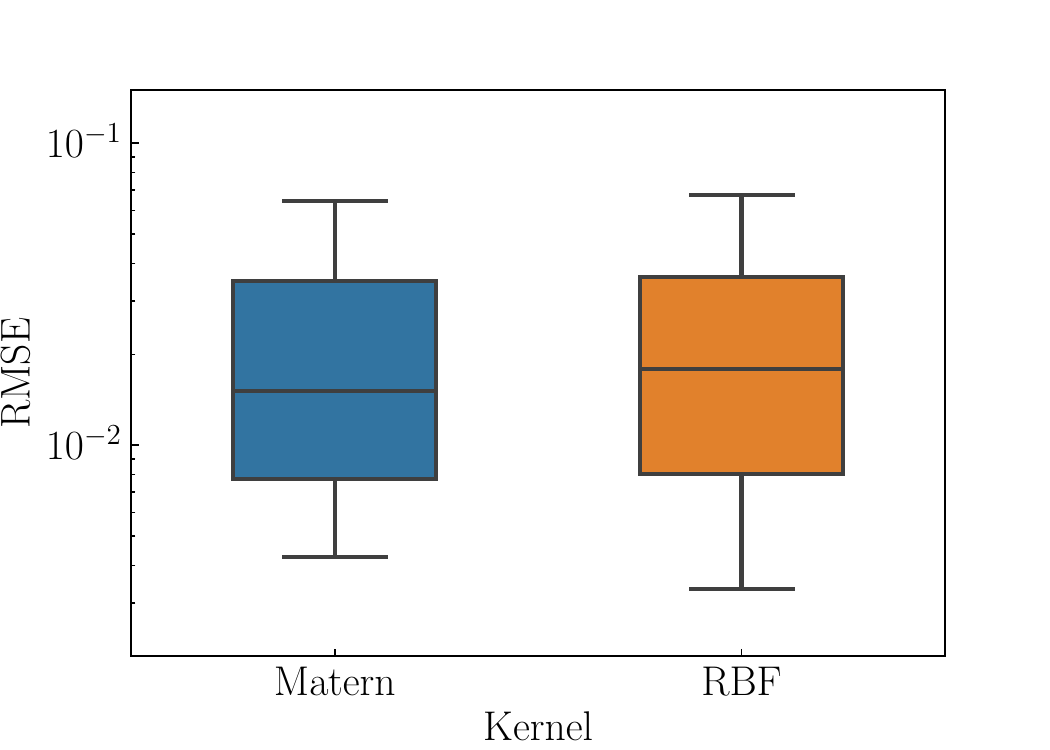}
        \caption{Ablation on kernel $k_\Theta$}
        \label{appfig:ablation_theta}
    \end{subfigure}
    \begin{subfigure}{0.30\textwidth}
        \centering
        \includegraphics[width=\linewidth]{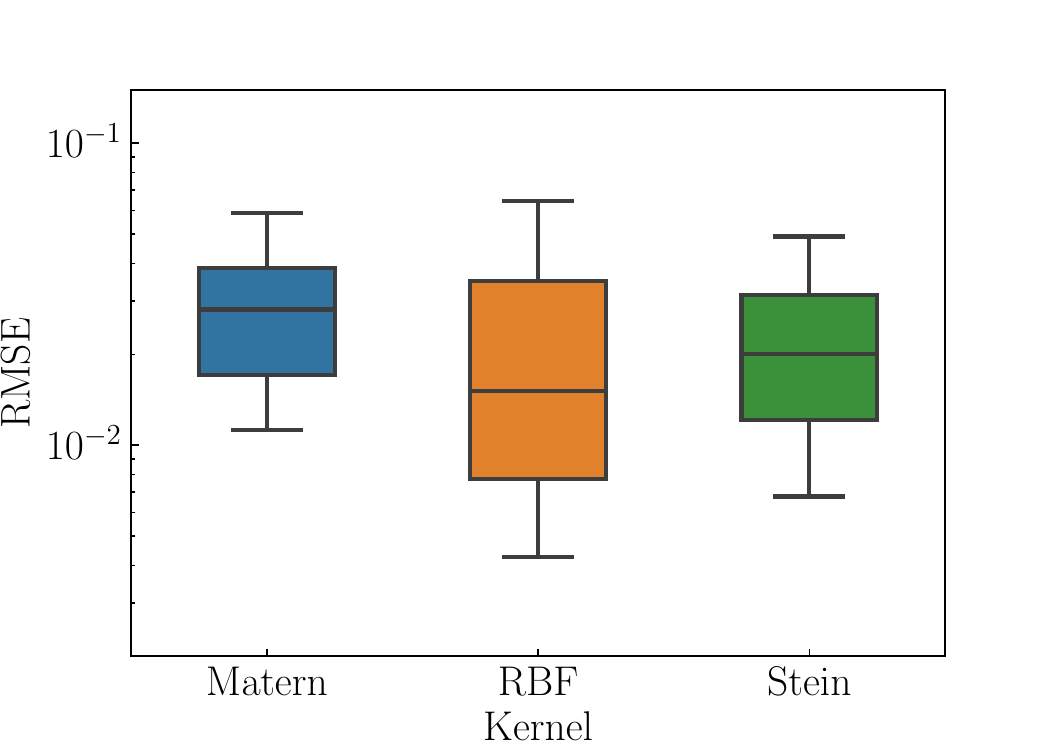}
        \caption{Ablation on kernel $k_\calX$}
        \label{appfig:ablation_x}
    \end{subfigure}
    \caption{\textbf{Left:} Comparison of all methods with standard i.i.d. sampling and Quasi-Monte Carlo samples. \textbf{Middle and Right:} Ablation study for CBQ with different $k_\Theta$ and $k_\calX$ kernels in Bayesian sensitivity analysis for linear models.}
\end{figure}

\subsection{Ablations on kernels}\label{appendix:ablation}

We present an ablation study evaluating the impact of distinct kernel choices $k_\calX$ and $k_\Theta$ within the framework of Bayesian sensitivity analysis in linear models. The integrand is $f(x)=x^\top x$, the dimension $d=2$ and $N=T=50$. 
First, we choose $k_\Theta$ to be Mat\'ern-3/2 kernel and Gaussian kernel. \Cref{appfig:ablation_theta} shows that the performance of CBQ remains consistent across different $k_\Theta$ kernels. 

Subsequently, we opt for Mat\'ern-3/2 kernel, Gaussian kernel and Stein kernel (with Mat\'ern-3/2 as the base kernel) as choices for $k_\calX$. When $k_\calX$ is Gaussian kernel, the formula for kernel mean embedding $\mu_\theta(x)$ is presented in \Cref{appeq:E14}. When $k_\calX$ is Mat\'ern-3/2 kernel, a closed form expression for the kernel mean embedding does not exist for the non-isotropic Gaussian distribution $\calN(\tilde{m}, \tilde{\Sigma})$, but the 'inverse transform trick' can be employed as in \Cref{appeq:transform}. 
When $k_\calX$ is Stein kernel, we choose Mat\'ern-3/2 as the base kernel and then apply Stein operator on both arguments of kernel $k_0$. All hyperparameters are selected according to \Cref{appendix:hyperparameter_selection}.
From \Cref{appfig:ablation_x}, we can see that CBQ performs best when $k_\calX$ is Mat\'ern-3/2 kernel, and we know that $k_\calX$ being Mat\'ern-3/2 kernel satisfies the assumptions of \Cref{thm:convergence}. When $k_\calX$ is Gaussian RBF kernel or Stein kernel, whether the assumptions of \Cref{thm:convergence} still hold is unknown, but in this ablation study, CBQ under both kernels have shown good performances in terms of RMSE.
The ablation study is only implemented in this very simple setting, so we encourage practitioners to be careful in the selection of kernels in real world applications.

\subsection{Calibration}\label{appendix:calibration}
CBQ falls in the area of probabilistic numeric algorithms that can provide finite-sample Bayesian quantification of uncertainty, where the uncertainty arises from having access to only a finite number of function values of the integrand.
Since CBQ is a two-stage hierarchical Gaussian process method in nature, and the final estimate $I_{\textrm{CBQ}}$ is treated as Gaussian distributed, so the standard deviation $\sigma^2_{\textrm{CBQ}}$ is a measure of uncertainty~\cite{kendall2017uncertainties}.
The calibration plots in \Cref{appfig:calibration} are obtained by altering the width of the credible interval and then computes the percentage of times a credible interval contains the true value $I(\theta)$ under repetitions of the experiment.
The black diagonal line represents the ideal case, with any curve lying above the black line indicating underconfidence and any curve lying below indicating overconfidence.
It is generally regarded more preferable to be underconfident than overconfident. 

In \Cref{appfig:calibration_bayes_sensitivity}, we show the calibration of the CBQ posterior for the integrand $f(x)=x^\top x$ when dimension $d=2$. 
We observe that when the number of samples is as small as $10$, CBQ is overconfident, which can be explained by the poor performance of using empirical Bayes to select hyperparameters in the small sample regime. 
On the other hand, when $N$ and $T$ increase, CBQ becomes underconfident, meaning that our posterior variance is more inflated than needed from a frequentist viewpoint.
The calibration plots for other experiments are all demonstrated in \Cref{appfig:calibration}, and the conclusions are consistent across different experiments.

\begin{figure}[t]
\centering
    \begin{subfigure}{0.23\textwidth}
        \centering
        \includegraphics[width=\linewidth]{figures/calibration_bayes.pdf}
        \caption{Calibration}
    \label{appfig:calibration_bayes_sensitivity}
    \end{subfigure}
    \begin{subfigure}{0.23\textwidth}
        \centering
        \includegraphics[width=\linewidth]{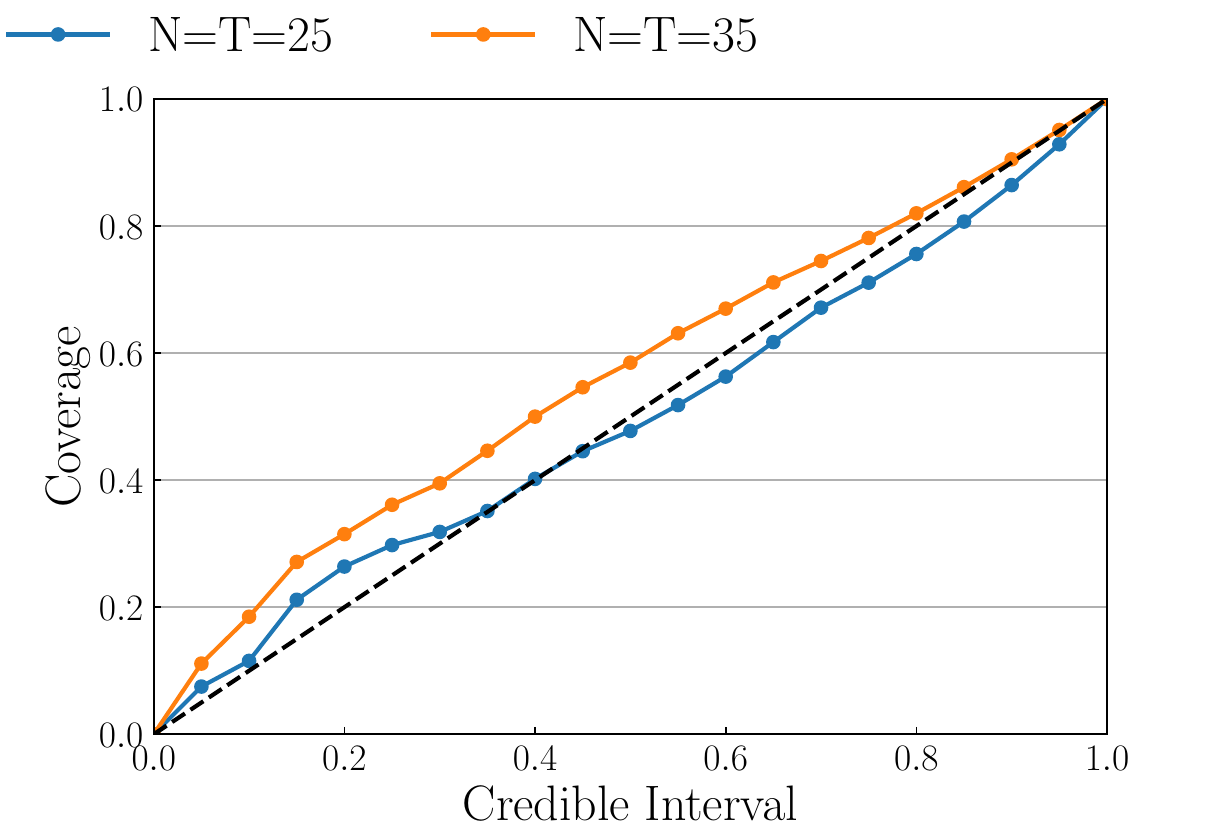}
        \caption{Calibration}
        \label{appfig:calibration_finance}
    \end{subfigure}
    \begin{subfigure}{0.23\textwidth}
        \centering
        \includegraphics[width=\linewidth]{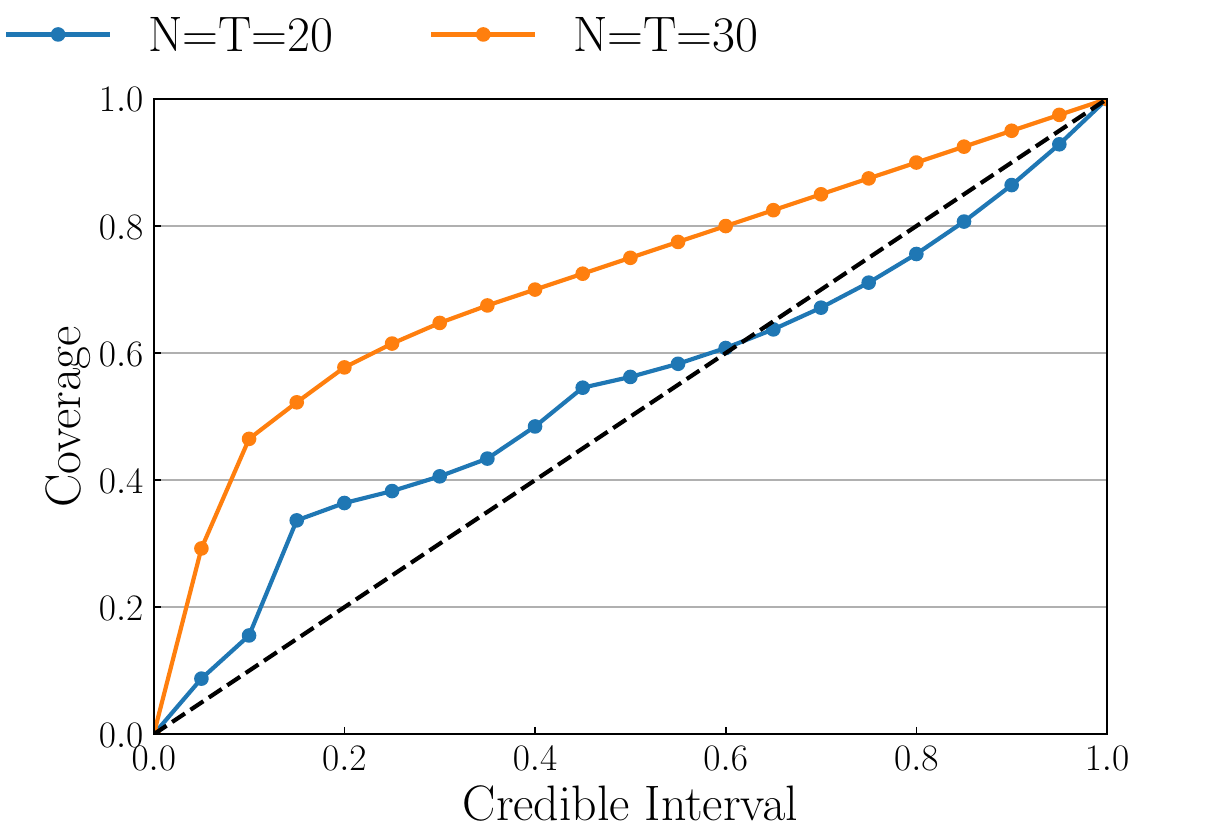}
        \caption{Calibration}
        \label{appfig:calibration_sir}
    \end{subfigure}
    \begin{subfigure}{0.23\textwidth}
        \centering
        \includegraphics[width=\linewidth]{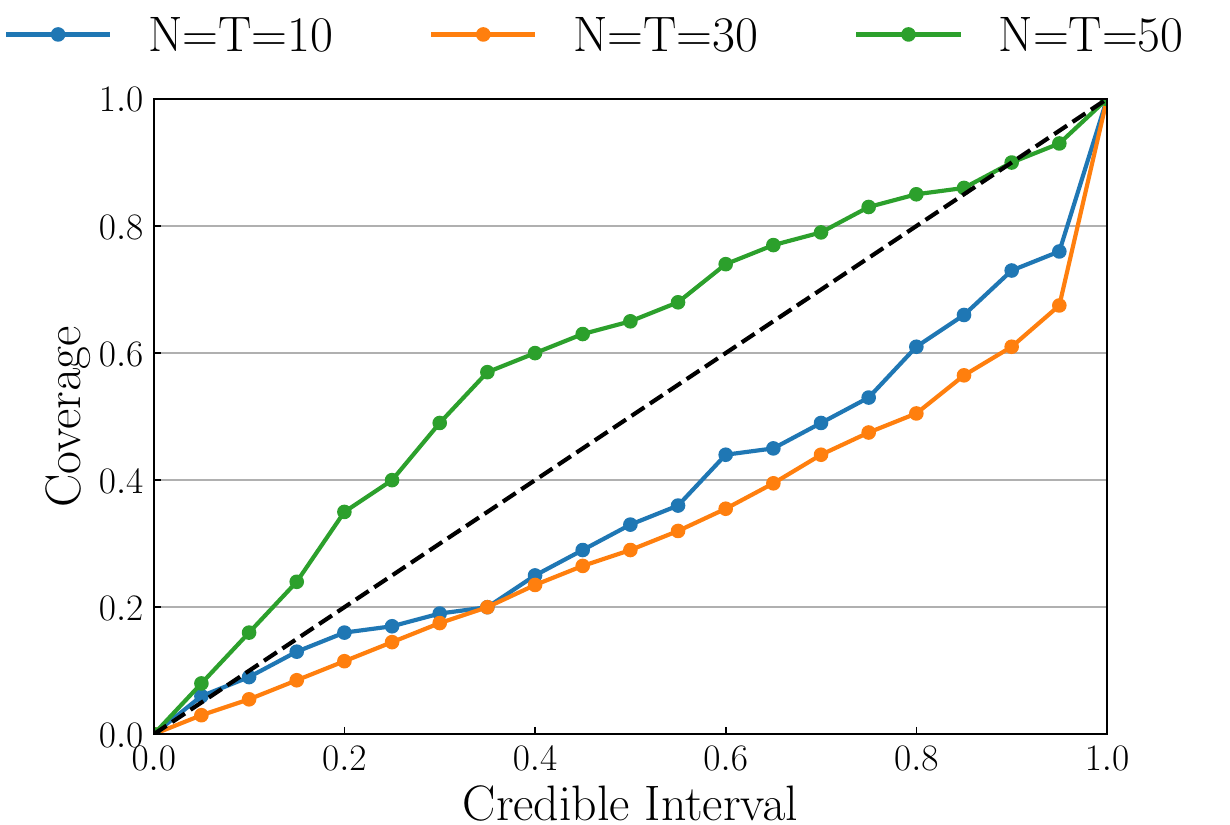}
        \caption{Calibration}
        \label{appfig:calibration_decision}
    \end{subfigure}
    \caption{Calibration plots. \textbf{Top Left:}  Bayesian sensitivity analysis in linear models. \textbf{Top Right:} Bayesian sensitivity analysis for SIR model. \textbf{Bottom Left:} Option pricing in mathematical finance. \textbf{Bottom Right:} Uncertainty decision making in health economics.}
    \label{appfig:calibration}
\end{figure}


\end{appendices}